\documentclass[lettersize,onecolumn]{IEEEtran}
\usepackage{amsmath,amsfonts}
\usepackage{algorithmic}
\usepackage{booktabs}
\usepackage{algorithm}
\usepackage{array}

\usepackage{subcaption}
\usepackage{graphicx}

\usepackage{textcomp}
\usepackage{stfloats}
\usepackage{url}
\usepackage{verbatim}
\usepackage{graphicx}
\usepackage{cite}
\usepackage{hyperref}

\usepackage{mathrsfs} 
\usepackage{amssymb}
\usepackage{mathtools}
\usepackage{amsthm}
\usepackage{multirow}
\usepackage{arydshln}
\usepackage{makecell}
\usepackage{bm}

\theoremstyle{plain}
\newtheorem{theorem}{Theorem}[section]
\newtheorem{proposition}[theorem]{Proposition}
\newtheorem{lemma}[theorem]{Lemma}
\newtheorem{corollary}[theorem]{Corollary}
\theoremstyle{definition}
\newtheorem{definition}[theorem]{Definition}

\theoremstyle{remark}
\newtheorem{remark}[theorem]{Remark}

\newtheorem{restatetheoreminner}{\bf Theorem}
\newenvironment{restatetheorem}[1]{%
  \renewcommand\therestatetheoreminner{#1}%
  \restatetheoreminner
}{\endrestatetheoreminner}

\newtheorem{restatepropositioninner}{\bf Proposition}
\newenvironment{restateproposition}[1]{%
  \renewcommand\therestatepropositioninner{#1}%
  \restatepropositioninner
}{\endrestatepropositioninner}

\newtheorem{restatecorollaryinner}{\bf Corollary}
\newenvironment{restatecorollary}[1]{%
  \renewcommand\therestatecorollaryinner{#1}%
  \restatecorollaryinner
}{\endrestatecorollaryinner}

\begin{document}

\title{Drift and Dependence: Layer-wise Information-Theoretic Bounds for Replay-Based Continual Learning}

\author{Tieliang Gong, Zhongbo Zhang, Wen Wen and Yong-Jin Liu, ~\IEEEmembership{Senior Member,~IEEE}
\thanks{This work was supported by National Natural Science Foundation of China under Grant 62576268, the Fundamental Research Funds for the Central Universities and the Key Research and Development Project in Shaanxi Province No. 2024PT-ZCK-89. (Corresponding author: Tieliang Gong.)}
\thanks{Tieliang Gong, Zhongbo Zhang and Wen Wen are with the School of Computer Science and Technology, Xi'an Jiaotong University, China (e-mail: adidasgtl@gmail.com; shuaishuai737@gmail.com; wen190329@gmail.com).}
\thanks{Yong-Jin Liu is with  Department of Computer Science and Technology, Tsinghua University, China (e-mail: liuyongjin@tsinghua.edu.cn)}
}

\maketitle

\begin{abstract}

Continual learning must absorb new tasks without erasing old ones, and replay---mixing a small buffer of past examples into current training---is among the most effective remedies for catastrophic forgetting. Yet its generalization behavior is shaped by two coupled effects that existing analyses fold into a single hypothesis-level quantity: finite memory replaces each past distribution with an empirical proxy, and repeated reuse couples the buffer, the current data, and the final hypothesis through a shared optimization trajectory. We develop a layer-wise information-theoretic framework that separates these effects at every depth. Our main result decomposes the expected generalization gap into a replay-induced representation drift and an optimization-dependence term, the latter further resolved into stability, plasticity, interaction, and residual-coupling components. Two refinements make the framework operational. A Wasserstein relaxation of the drift term, valid under support mismatch, yields a depth-dependent drift--sensitivity trade-off whose minimizer identifies which interior layer to stabilize. An SGLD instantiation of the optimization term reduces it to a trajectory-level log-determinant budget, exposing a curvature-aware gradient-alignment statistic that serves as an online diagnostic of task-wise forgetting. Controlled and benchmark experiments confirm the predicted memory scaling, the interior funnel, and the alignment signal's link to forgetting.

\end{abstract}


\section{Introduction}
Continual learning (CL) studies how a model can incrementally learn a sequence of tasks while retaining knowledge acquired from previous ones \cite{kirkpatrick2017overcoming,parisi2019continual,de2021continual}. A central obstacle in CL is catastrophic forgetting: when the model adapts to a new task, its performance on earlier tasks may deteriorate significantly \cite{mccloskey1989catastrophic,goodfellow2013empirical}. In recent years, extensive efforts have been made to mitigate catastrophic forgetting \cite{kao2021natural,liang2023adaptive}. Among them, replay-based methods \cite{rolnick2019experience,rebuffi2017icarl,guo2022online} have emerged as one of the most successful practical mechanisms, where the learner stores a small buffer of past examples and mixes them with the current task during training. 

Despite this success, the generalization behavior of replay-based CL remains insufficiently understood since replay changes the statistical object being optimized. The learner never minimizes risk against the past-task populations. It minimizes risk against finite proxies of them, and does so repeatedly as the model continues to evolve. This substitution acts in two channels. First, replacing a population by a finite stored subset produces a proxy bias, say, the empirical objective on the buffer is a biased stand-in for the population objective, and the bias does not disappear merely because the data are i.i.d. Second, reusing the same stored samples across many updates produces a dependence coupling among the old samples, the current samples, and the final hypothesis, since all three are tied together through a shared optimization trajectory. Note that neither channel presupposes that the model learns a representation, and both are already present when replay is performed directly on raw inputs.

What a deep network adds is not a third channel but a change in how these two channels behave with depth, and it does so asymmetrically. Because the buffer influences later tasks only through a learned feature map, and because that map is fitted to the particular stored exemplars, the proxy bias is no longer fixed: it is generated and reshaped at every layer, and a discrepancy that is negligible in input space can be amplified by the sub-network above it. The dependence coupling, by contrast, is set by the algorithm's reuse of data and survives even when no representation is learned. The learned representation therefore turns the proxy bias into a depth-dependent object while leaving the coupling channel essentially intact. As the boundary cases of our bound make precise, the proxy-bias channel vanishes at the input layer and grows with depth, whereas the coupling channel survives even there.

Existing theory does not separate these channels. Analyses based on replay losses, trade-off parameters, or domain-adaptation arguments \cite{riemerlearning,shi2023unified} blend the intrinsic difficulty of learning from the full stream with the approximation cost of compressing it into a buffer, so the finite capacity loss and the reuse-induced dependence never surface as distinct quantities. Classic complexity bounds fare no better: their capacity terms grow far faster than the sample size, leaving them numerically vacuous in precisely the overparameterized regime of interest. Stability and convergence analyses under smoothness or Lipschitz assumptions \cite{li2024stability,knoblauch2020optimal} take the opposite tack---they track optimization dynamics closely but say little about how buffer size, sequence length, depth, and algorithmic dependence jointly shape the gap. The information-theoretic framework comes closest: it relates the learned parameters to the data through mutual information and its refinements  \cite{xu2017information,steinke2020reasoning,haghifam2020sharpened} and recent work carries this analysis across the layers of a deep network \cite{he2024hierarchical}. This is our starting point, but replay changes the objects the framework must measure, so it cannot be applied off the shelf. The relevant quantities are no longer parameter--data dependencies but buffer-induced representation distributions, and the coupling created by reuse and by sampling without replacement is intrinsic to the problem rather than a nuisance term: it must be built into the analysis from the outset. What makes a layer-wise treatment unavoidable is the asymmetry between the two channels. Because the finite-memory channel is reshaped at every layer while the coupling channel is not, a black-box, hypothesis-level analysis averages over exactly the depth structure the buffer induces, collapsing the very quantity we set out to resolve---the depth at which the buffer's finite-sample bias is least amplified by the network above it. We therefore introduce two layer-wise objects: per-task replay centroids, which make the finite-memory discrepancy well-defined at every depth, and a decoupled reference, which isolates the dependence due to reuse. Together these let finite-memory bias and optimization coupling sit in a single decomposition---something hypothesis-level information-theoretic bounds cannot express.

There are two difficulties in carrying this out rigorously. The first is circularity:  the feature map at a given depth is learned online and shifts with whichever examples enter the buffer, so the map we use to measure replay discrepancy is shaped by the very buffer we are measuring. The second is accumulation: the dependence entering a new task is not generated fresh but inherited from earlier tasks and built up along the training trajectory. Both must be handled jointly and at an arbitrary depth, not just at the input or output. This is what the paper's three constructions provide: a replay centroid that makes the finite-memory discrepancy well-defined at every layer and separates it from the dependence term; a Wasserstein relaxation of that discrepancy for the regime where its information-theoretic form is ill-posed under support mismatch; and a trajectory-level decomposition that, for a concrete noisy optimizer, splits the accumulated dependence into per-step contributions.

At a high level, our analysis separates replay generalization into two coupled effects: a replay-induced drift term, arising from the finite-memory approximation of past task distributions, and an optimization-dependence term, arising from the reuse of replayed samples through shared downstream parameters. The former quantifies the mismatch between the replay buffer and the corresponding past-task population in representation space, while the latter captures how strongly the learned suffix parameters encode information about both old and current task representations. This decomposition cleanly distinguishes the statistical cost of memory compression from the dependence created by repeated optimization. We then further refine the drift branch with a Wasserstein relaxation when the KL-based form becomes uninformative under support mismatch, and instantiate the dependence branch for stochastic gradient Langevin dynamics (SGLD). The resulting quantities are intended as analytic diagnostics and control signals for replay-based training, rather than as a replacement for existing continual learning algorithms.

In summary, our contributions are as follows:
\begin{itemize}
  \item We develop a layer-wise information-theoretic bound for replay-based continual learning that explicitly separates replay-induced drift from optimization dependence. The bound further clarifies how old-task dependence, current-task dependence, their interaction, and residual coupling jointly shape the expected generalization gap.

  \item When the KL-based drift term becomes vacuous under empirical-population support mismatch, we derive a Wasserstein relaxation that yields a depth-dependent drift--sensitivity trade-off and identifies an interior generalization-funnel layer, indicating which layer to stabilize.

  \item We instantiate the optimization-dependence branch with SGLD and obtain a trajectory-level log-determinant budget. This leads to sensitivity-aware gradient-alignment diagnostics and a local replay--current mixing rule.

  \item We validate the theory through controlled and benchmark experiments, confirming the predicted finite-memory scaling, funnel behavior, and alignment patterns in replay dynamics.
\end{itemize}

\section{Related Work}
\subsection{Replay-based continual learning}
Catastrophic forgetting under sequential training has long been recognized as a central obstacle for neural networks \cite{mccloskey1989catastrophic,french1999catastrophic}, and modern CL methods are commonly grouped into replay-based, regularization-based, and architectural/parameter-isolation approaches \cite{van2024continual,parisi2019continual,wang2024comprehensive}. Regularization methods constrain updates through parameter-importance penalties such as Fisher information and synapse-importance surrogates \cite{kirkpatrick2017overcoming,zenke2017continual,aljundi2018memory}, while architectural strategies expand or partition parameters across tasks \cite{rusu2016progressive,mallya2018packnet}. Replay remains a strong and competitive option because it directly counteracts forgetting by revisiting stored examples, through exemplar rehearsal \cite{rebuffi2017icarl}, generic experience replay \cite{rolnick2019experience}, and generative replay \cite{shin2017continual}. Modern replay systems differ mainly in how they allocate memory and select exemplars: a typical choice is reservoir sampling \cite{chrysakis2020online}, while more informed strategies prioritize examples that maximize future utility or reduce interference \cite{sun2022information}, and further work corrects the representation drift and class imbalance that erode replay in online CL \cite{caccianew}. Replay also connects to constrained-optimization views that explicitly control gradient interference \cite{aljundi2019gradient}, including GEM \cite{lopez2017gradient}, its averaged variant A-GEM \cite{chaudhry2018efficient}, and orthogonal projection \cite{saha2021gradient}. These methods align with the view taken here, that replay replaces past populations with finite proxy distributions while evolving representations cause the buffer to shift within feature space. Our goal is different from any single update rule: we begin from a layer-wise decomposition of replay generalization that separates distributional drift from optimization coupling, and that explains when replay acts as a stable rehearsal mechanism and when it becomes a biased or strongly coupled proxy for the past.

\subsection{Information-theoretic analysis}
Recent theoretical work studies forgetting, task order, and generalization in sequential learning through statistical or optimization lenses. Closest to us are information-theoretic and statistical analyses of forgetting \cite{lin2023theory,ding2024understanding}, which operate in linear or overparameterized regimes and characterize how task similarity and ordering affect generalization. For the replay setting in particular, information-theoretic bounds have been derived from the dependence between the hypothesis and the memory buffer \cite{wen2026information}, separating an ideal full-data risk from a memory-compression cost. Because these bounds are read at the level of the whole hypothesis, their finite-memory term stays a raw-sample compression cost that does not resolve the representation-level drift or the depth at which the gap arises. Our contribution is different in kind: we retain learned hierarchical representations and expose layer-dependent quantities ($\mathcal K^{(l)}$, $\mathcal S^{(l)},\mathcal P^{(l)},\mathcal R^{(l)}$) that such hypothesis- or input-layer analysis cannot see. We do not claim a tighter constant; we claim a finer-grained decomposition. More broadly, the information-theoretic framework relates learned parameters to data via mutual information \cite{xu2017information} and its conditional refinements \cite{steinke2020reasoning,haghifam2020sharpened}, and has been specialized to the trajectory of noisy iterative algorithms such as SGLD \cite{negrea2019information,dong2023understanding}; the dependence of generalization on network depth has also been studied directly \cite{he2024hierarchical,sonoda2025deepbettershallowimplementationagnostic}. Our main theorem combines these perspectives for replay: it keeps the information-theoretic dependence structure, but introduces replay centroids and layer-wise old/new task representations so that finite replay bias and optimization coupling appear in the same decomposition.

\subsection{Geometry and gradient interaction}
Wasserstein and optimal-transport arguments provide finite notions of mismatch when KL-type divergences are unstable under support shift, both as a general geometric tool \cite{peyre2019computational} and in domain-adaptation and generalization bounds that control risk through a transport discrepancy \cite{redko2017theoretical,lopez2018generalization,wang2019information}; these works, however, remain confined to single-task learning. We use this geometry only to relax the drift branch of our main decomposition. Conversely, our SGLD analysis targets the optimization branch and yields gradient-moment diagnostics. This distinguishes our alignment quantities from Euclidean projection rules in GEM/OGD-style methods: here alignment is derived as a bound-driven, sensitivity-aware diagnostic of optimization coupling rather than introduced as a standalone projection heuristic.

\section{Problem Setup and Layer-wise Replay Distributions}

\textbf{Notations.}
Throughout this paper, we denote random variables by uppercase letters (e.g., $X$), their realizations by lowercase letters (e.g., $x$), and their domains by calligraphic letters (e.g., $\mathcal{X}$). The distribution of $X$ is denoted $P_X$. Expectation, variance, and covariance are written $\mathbb{E}[X]$, $\mathrm{Var}(X)$, and $\mathrm{Cov}(X)$, respectively. For conditional distributions, we write $P_{X|Y=y}$. We use the standard information-theoretic quantities: entropy $H(X)$, mutual information $I(X;Y)$, Kullback-Leibler (KL) divergence $D_{\mathrm{KL}}(P \| Q)$, and the $p$-Wasserstein distance $\mathcal{W}_p$ between probability measures. For vectors, $\|v\|$ denotes the Euclidean ($\ell_2$) norm, and for matrices, $\|A\|_{\mathrm{op}}$ denotes the operator norm. We write $[T]\triangleq\{1,\dots,T\}$.

\subsection{Replay-Based Continual Learning}
\label{sec:replay_cl}

We consider supervised classification with input space $\mathcal{X}$, label space $\mathcal{Y}$ and joint space $\mathcal{Z}= \mathcal{X} \times \mathcal{Y}$. A sequence of $T$ tasks $D^{1:T}=\{D^1,\ldots,D^T\}$ arrives incrementally. For each $t \in [T]$, task $t$ provides a dataset $D^t=\{Z_j^t\}_{j=1}^n \in \mathcal{Z}^n$, consisting of $n$ i.i.d.\ samples drawn from an unknown task distribution $\mathcal{D}_t$ on $\mathcal{Z}$. We do not impose restrictions on the relationships among $\{\mathcal{D}_t\}_{t=1}^T$, permitting arbitrary distributional heterogeneity across tasks. Moreover, the datasets $\{D^t\}_{t=1}^T$ are assumed independent across tasks. At the beginning of task $t$, the learner receives a fresh dataset $D^t$ and a replay buffer $\mathcal{M}^{1:t-1}=\{M^i\}_{i=1}^{t-1}$ holding exemplars from previous tasks. For each past task $i\in[t-1]$, the buffer stores a subset $M^i=\{Z_j^i\}_{j\in\mathcal I_i}\subset D^i$, where the index set $\mathcal I_i\subset[n]$ has fixed size $|\mathcal I_i|=m$ and is drawn uniformly without replacement from $[n]$. We focus on the memory-limited regime $m \ll n$. Importantly, a past task never recurs as a full dataset: once task $i$ is completed, it can be revisited only through the $m$ exemplars retained in $M^i$.

Let $\mathcal{W}$ denote the hypothesis parameter space, and let $\ell : \mathcal{W} \times \mathcal{Z} \rightarrow \mathbb{R}_+$ be a fixed nonnegative loss function. After receiving task $t$, training uses the union of the current dataset and all replayed exemplars. We define the empirical risk for $W \in \mathcal{W}$ as
\begin{equation}
\label{eq:emp_risk}
    \hat{\mathcal{L}}^{(t)}(W)
    \;=\;
    \sum_{i=1}^{t-1} 
    \frac{1}{m} \sum_{j\in\mathcal{I}_i} \ell(W, Z_j^i)
    \;+\;
    \frac{1}{n} \sum_{j=1}^{n} \ell(W, Z_j^t),
\end{equation}
with the population risk
\begin{equation}
\label{eq:pop_risk}
    \mathcal{L}^{(t)}(W)
    \;=\;
    \sum_{i=1}^{t}
    \mathbb{E}_{Z \sim \mathcal{D}_i}
    \bigl[ \ell(W, Z) \bigr].
\end{equation}
Since performance is evaluated after training over all tasks, we write $\mathcal{L}(W) \equiv \mathcal{L}^{(T)}(W)$ and $\hat{\mathcal{L}}(W) \equiv \hat{\mathcal{L}}^{(T)}(W)$, and define the expected generalization error of replay-based continual learning as
\begin{equation}
\label{eq:gen_error}
    \mathrm{gen}_W
    \;\triangleq\;
    \mathbb{E}_{W,D^{1:T},\{\mathcal I_i\}_{i<T}}
    \bigl[
        \mathcal{L}(W) - \hat{\mathcal{L}}(W)
    \bigr],
\end{equation}
where the expectation is over the task data, the algorithm randomness, and the replay-sampling procedure. Here $W$ is not a fixed parameter but the random output of the learning algorithm: its law is induced jointly by the stream $D^{1:T}$, the random buffer construction, and any stochasticity of the optimizer.

\subsection{Layer-wise Replay Distributions}
\label{sec:layerwise_dist}
\noindent\textbf{Layer-wise representations.}
We analyze a depth-$L$ network  $f_W:\mathcal{X}\to\mathbb{R}^K$ with weights $W=(W_1,\ldots,W_L)$, where $W_l\in\mathbb{R}^{d_l\times d_{l-1}}$. Setting $a_0(x)=x$, each layer defines the intermediate representation 
\begin{equation}
\label{eq:net_recursion}
    a_l(x)\;=\; \phi_l\!\left(W_l\, a_{l-1}(x)\right), \qquad l\in[L],
\end{equation} 
for an element-wise activation $\phi_l$, and the output is $\hat{y}=f_W(x)=a_L(x)$. We write $A_l = a_l(X)$ for the layer-$l$ representation of a random input $X$, and $A_{j,l}^t = a_l(X_j^t)$ for the representation of the $j$-th example of task $t$. We use $W_{1:l}=(W_1,\dots,W_l)$ for the bottom (feature) sub-network up to layer $l$ and $W_{l+1:L}$ for the remaining suffix.

\noindent\textbf{Split-layer distributions.}
To analyze generalization at an arbitrary depth, we fix a split layer $l \in \{0, \dots, L\}$ and condition on the learned bottom parameters $W_{1:l}$, with the conventions that $W_{1:0}$ and $W_{L+1:L}$ are empty. We unify the data source of each task as $S^i$, where $S^i = M^i$ for an old task $i < T$ and $S^i = D^T$ for the current task $i = T$, and write $S\triangleq\bigcup_{i=1}^T S^i$ for the full training set at time $T$. Replay generalization is governed by how three laws of the representation--label pair $(A_l,Y)$ relate to one another, which we introduce in turn.

\emph{(i) Population law.}
$P_{A_l,Y|i,W_{1:l}}$ is the population distribution of $(A_l,Y)$ obtained by pushing $(X,Y)\sim\mathcal{D}_i$ through the learned features. Here, ``pushing through $W_{1:l}$'' acts on the input only: $X\mapsto A_l=a_l(X)$, while the label $Y$ is carried over unchanged. This is the target the learner would like to control but never observes directly.

\emph{(ii) Empirical proxy.}
$\hat{P}_{A_l,Y|S^i,W_{1:l}}$ is the uniform distribution over the representation-label pairs induced by the finite sample $S^i$ under the learned features $W_{1:l}$. This is the actual objective used in empirical risk minimization. 

\emph{(iii) Replay centroid.}
To quantify the statistical bias introduced by the finite buffer, we average over the randomness of buffer construction and define, for $i<T$, 
\begin{equation}
\label{eq:centroid}
    Q_{A_l,Y|i,W_{1:l}}\;\triangleq\; \mathbb{E}_{S^i} \left[ \hat{P}_{A_l,Y|S^i,W_{1:l}}\;\middle|\; W_{1:l} \right].
\end{equation}
The subscript $W_{1:l}$ and the outer expectation are not redundant: the inner object $\hat{P}_{A_l,Y|S^i,W_{1:l}}$ is the proxy for a single realized buffer $S^i$, whereas the outer $\mathbb{E}_{S^i}[\,\cdot \mid W_{1:l}]$ averages over which $m$ exemplars happened to be stored, holding the features at their learned value. Equivalently,  $Q_{A_l,Y|i,W_{1:l}}$ is the conditional distribution of a random training example from task $i$ given the learned representation parameters. For old tasks, in general, one has $Q_{A_l,Y|i,W_{1:l}}\neq P_{A_l,Y|i,W_{1:l}}$. The cause is a representation-induced selection effect: although the raw exemplars are drawn i.i.d.\ from $\mathcal{D}_i$---so that at the input layer a uniformly chosen stored example already carries the task marginal---the feature map $W_{1:l}$ is itself fitted to the particular realized buffer. Conditioning on this data-dependent $W_{1:l}$ ties the stored representations to the very sample that shaped them, shifting their averaged law away from the population.

\emph{(iv) Stacked training representations.}
Collecting the layer-$l$ representations of the entire training set gives
\begin{equation}
\label{eq:U_stack}
\mathbf{U}^{(l)} \;\triangleq\; \big\{ (a_l(x), y) : (x, y) \in S \big\} \;\equiv\; \big(\mathbf U_{\mathrm{old}}^{(l)},\; \mathbf U_{\mathrm{new}}^{(l)}\big),
\end{equation}
where $\mathbf{U}_{\mathrm{old}}^{(l)}$ gathers the representations of the buffered exemplars $\bigcup_{i<T} M^i$ and $\mathbf{U}_{\mathrm{new}}^{(l)}$ those of the current task $D^T$. We emphasize that $\mathbf{U}_{\mathrm{old}}^{(l)}$ is a collection of representations, not of raw exemplars: the stored inputs are frozen once written to the buffer, but their images $a_l(\cdot)$ shift as the shared features $W_{1:l}$ are learned. This is exactly why replayed data can carry information about the current optimization even though the buffer content is fixed.

\emph{(v) Decoupled reference.}
Let $P_{\mathbf{U}^{(l)} \mid W_{1:l}}$ denote the joint law of this stacked sequence. Because all coordinates share the data-dependent map $W_{1:l}$ and the buffer is drawn without replacement, the entries of $\mathbf{U}^{(l)}$ are statistically dependent. To isolate this coupling, we compare against an idealized Decoupled Reference in which every coordinate is drawn independently, with old-task coordinates from their centroid and current-task coordinates from the population: 
\begin{equation}
\label{eq:decoupled}
\widetilde{Q}_{\mathbf{U}^{(l)}|W_{1:l}}\!\triangleq\!\left(\bigotimes_{i=1}^{T-1} (Q_{A_l,Y|i,W_{1:l}})^{\otimes m}\right)\otimes (P_{A_l,Y|T,W_{1:l}})^{\otimes n}.
\end{equation}
The tensor exponents record the sample budgets: $(\cdot)^{\otimes m}$ places $m$ i.i.d.\ copies for each of the $T-1$ past tasks (one per stored exemplar), and $(\cdot)^{\otimes n}$ places $n$ i.i.d.\ copies for the current task. The reference is therefore an i.i.d.\ process across the stored and current coordinates: it preserves every coordinate's marginal but strips away the cross-sample coupling created by non-replacement sampling and by the shared feature map.

\section{A Layer-wise Decomposition of Replay Generalization}
\label{Main Decomposition}
This section contains the backbone result of the paper. Theorem~\ref{thm:synergy_drift} is the only place where the full replay generalization gap is decomposed. The remainder of the paper revisits its two components separately: the next subsection presents the decomposition itself, Section~\ref{sec:wasserstein_relaxation} relaxes the drift branch geometrically when KL control becomes vacuous, and Section~\ref{sec:algorithm} upper bounds the optimization branch for a concrete noisy optimizer.

\subsection{Main Decomposition: The Synergy-Drift Bound}
The classical stability--plasticity framework is useful \cite{wang2024comprehensive}, but it does not by itself isolate the replay-specific bias induced by finite memory nor the dependence created by repeatedly optimizing on replayed and current samples. The next theorem addresses both issues simultaneously: it isolates replay-induced drift and then decomposes the remaining optimization complexity into stability, plasticity, interaction, and residual coupling.
\begin{theorem}[Hierarchical Synergy--Drift Bound]
\label{thm:synergy_drift}
Let $W$ be the output of a learner minimizing the empirical risk $\hat{\mathcal L}(W)$.
Assume the loss function $\ell(W,Z)$ is $\sigma$-subgaussian for all $Z\in\mathcal{Z}$ and $W\in\mathcal{W}$. For any split layer $l \in \{0, \dots, L\}$, we have
\begin{align*}
|\mathrm{gen}_W|
\le
\underbrace{(T-1)\sqrt{2\sigma^2\,\mathcal K^{(l)}}}_{\text{\bf Replay-centroid Drift}} + \underbrace{\sqrt{\frac{2\sigma^2}{N_{\mathrm{eff}}}
\Big(\mathcal S^{(l)}+\mathcal P^{(l)}-\mathcal R^{(l)}+\mathcal C^{(l)}\Big)}}_{\text{\bf Optimization Variance}},
\end{align*}
where $N_{\mathrm{eff}} \triangleq \frac{1}{(T-1)/m + 1/n}$ is the effective sample size, $\,\mathcal K^{(l)}\,\triangleq\,\frac{1}{T-1}\sum_{i=1}^{T-1}\mathbb E_{W_{1:l}}[D_{\mathrm{KL}}( Q_{A_l,Y|i,W_{1:l}}\|P_{A_l,Y|i,W_{1:l}})]$, $\mathcal C^{(l)} \triangleq \mathbb E_{W_{1:l}}[ D_{\mathrm{KL}}( P_{\mathbf U^{(l)}\mid W_{1:l}} \| \widetilde{Q}_{\mathbf{U}^{(l)}|W_{1:l}}) ]$.
\end{theorem}
Theorem~\ref{thm:synergy_drift} splits replay generalization into two parts. The replay-centroid drift term captures the bias induced by finite memory: at layer $l$, the mismatch between the replay centroid $Q_{A_l,Y|i,W_{1:l}}$ and the corresponding population distribution $P_{A_l,Y|i,W_{1:l}}$. A key feature of this term is that it generally does not vanish as training proceeds, even though the raw exemplars are drawn i.i.d. from the task. The i.i.d. sampling guarantees only that the stored data match the population in input space, which is why the drift is zero at the input layer. In feature space the situation differs, because the map $W_{1:l}$ is itself fitted to the particular stored buffer. Conditioning on this buffer-dependent map couples the stored representations to the same samples that shaped the map, and this selection effect displaces their averaged law away from the population. The second term measures the optimization dependence under the actual training representations.

Its prefactor $\sqrt{2\sigma^2/N_{\mathrm{eff}}}$ is set by the effective sample size $N_{\mathrm{eff}}=\frac{1}{(T-1)/m+1/n}$ for the replay-current empirical objective, and the variance term scales as $1/\sqrt{N_{\mathrm{eff}}}$. When replay is effectively unlimited ($m\to\infty$), the old-task contribution $(T-1)/m$ vanishes and the current-task dependence recovers the usual $\mathcal O(1/\sqrt n)$ rate. In contrast, when the current-task size grows under a fixed replay budget ($n\to\infty$), $N_{\mathrm{eff}}$ approaches the finite limit $m/(T-1)$, so the variance term does not vanish but settles at a finite floor of order $\sqrt{(T-1)/m}$, which we call the finite-memory variance floor. The variance term should not be considered a measure of forgetting. It bounds the train-population discrepancy, whereas forgetting is the degradation of old-task population risk, carried by the drift term $\mathcal K^{(l)}$, which does not vanish during training, together with the cross-task dependence accumulated along the optimization trajectory (Corollary~\ref{cor:Cumulative information budget}).

The optimization term further admits a stability--plasticity--synergy (SPS) decomposition. Here $\mathcal{S}^{(l)} \triangleq I(\mathbf{U}_{\mathrm{old}}^{(l)};W_{l+1:L}\mid W_{1:l})$ measures how much information the suffix weights retain about replayed old-task representations; this is the stability term, the learner's information-theoretic memory of past tasks. Correspondingly, $\mathcal{P}^{(l)} \triangleq I(\mathbf{U}_{\mathrm{new}}^{(l)};W_{l+1:L}\mid W_{1:l})$ measures the dependence on the current task; this is the plasticity term, the extent to which the suffix encodes and adapts to the new task. The interaction term $\mathcal{R}^{(l)} \triangleq I(\mathbf{U}_{\mathrm{old}}^{(l)};\mathbf{U}_{\mathrm{new}}^{(l)};W_{l+1:L}\mid W_{1:l})$ is signed: positive values tighten the bound and correspond to beneficial sharing between old and current representations, whereas negative values enlarge the bound and correspond to interference. The theorem does not assert that $\mathcal{R}^{(l)}$ is always nonnegative. Finally, $\mathcal C^{(l)}$ measures the residual dependence between the actual replay sequence and an idealized independent reference. This term accounts for statistical coupling introduced by sampling without replacement and by the shared learned representation map $W_{1:l}$.

\begin{corollary}[Input-Layer Bound]
\label{cor:input_layer}
At the input layer $l=0$, the hierarchical generalization bound simplifies to:
\begin{align*}
|\mathrm{gen}_W|
\;\le\;
\sqrt{\frac{2\sigma^2}{N_{\mathrm{eff}}}\,I\!\big(S;W\big)},
\end{align*}
where $S$ is the original $(X,Y)$ pair in the replay buffer and the current task dataset.
\end{corollary}
Specifically, at the input layer, $\mathbf{U}^{(0)}$ represents the original data pair. From Theorem~\ref{thm:synergy_drift}, we have $I(S;W)=I(\mathbf{U}_{\mathrm{old}}^{(0)};W)+I(\mathbf{U}_{\mathrm{new}}^{(0)};W)-I(\mathbf{U}_{\mathrm{old}}^{(0)};\mathbf{U}_{\mathrm{new}}^{(0)};W)$. The replay centroid $Q_{X,Y|i,W}$ coincides with the population $P_{X,Y|i,W}$ and the training sequence is block-wise i.i.d., therefore $\mathcal{K}^{(0)}=0$ and $\mathcal{C}^{(0)}=0$. In the special case $T=1$, replay disappears: the empirical and population risks reduce to their standard single-task form, $N_{\mathrm{eff}} = n$, and the bound recovers the result of \cite{xu2017information}. When $m=n$, the buffer retains every past example and $S$ consists of $Tn$ i.i.d.\ samples; the bound then matches the classical on-average stability bound \cite{steinke2020reasoning} for $Tn$ samples. Corollary~\ref{cor:input_layer} thus extends the standard mutual-information framework to replay-based continual learning. The CL structure therefore does not arise at the input layer itself. Instead, replay-specific effects emerge through the finite-memory factor $N_{\mathrm{eff}}$, the representation-level drift term $\mathcal K^{(l)}$, and the interaction terms appearing for $l>0$, which become relevant only when multiple tasks are learned under memory constraints.

Then, when considering $l=L$, we obtain a bound dominated by pure drift:
\begin{corollary}[Output-Layer Bound]
\label{cor:output_layer}
At the output layer $l=L$, the hierarchical generalization bound simplifies to:
\[
|\mathrm{gen}_W| \;\le\; (T-1)\sqrt{2\sigma^2\,\mathcal K^{(L)}} \;+\; \sqrt{\frac{2\sigma^2}{N_{\mathrm{eff}}}\,\mathcal C^{(L)}},
\]
where $\mathcal K^{(L)}=\frac{1}{T-1}\sum_{i=1}^{T-1}\mathbb E_W[D_{\mathrm{KL}}( Q_{A_L,Y|i,W}\|P_{A_L,Y|i,W})]$, $\mathcal C^{(L)} = \mathbb E_W[D_{\mathrm{KL}}( P_{\mathbf U^{(L)}\mid W} \| \widetilde{Q}_{\mathbf{U}^{(L)}|W})]$.
\end{corollary}
In particular, Corollary~\ref{cor:output_layer} can be viewed as a degenerate limit of Theorem~\ref{thm:synergy_drift}. At the output layer, there is no trainable mapping downstream of the representation, hence the stability--plasticity--synergy tradeoff over $W_{l+1:L}$ disappears and
$\mathcal{S}^{(L)}=\mathcal{P}^{(L)}=\mathcal{R}^{(L)}=0$, leaving generalization solely governed by the replay-centroid drift $\mathcal{K}^{(L)}$ and dependence divergence $\mathcal{C}^{(L)}$. Read together with Corollary~\ref{cor:input_layer}, this highlights a depth-wise shift of the dominant error: from parameter-information complexity at shallow splits to representational mismatch at deep ones. This shift is directly actionable. Because neither end of the network balances the two costs, the bound predicts that stabilization techniques such as feature distillation or partial freezing \cite{douillard2020podnet} should target the interior layer that minimizes the drift–sensitivity product, rather than being applied uniformly across layers. Figure~\ref{fig:bound_gap} locates this layer empirically, and the feature-distillation probe confirms that stabilizing the interior basin outperforms stabilizing either end.

\subsection{Geometric Relaxation of the Drift Term}
\label{sec:wasserstein_relaxation}
Theorem~\ref{thm:synergy_drift} identifies replay-induced drift as one of the two sources of replay generalization error. Its KL form is natural from the information-theoretic proof, but it is not always the most informative geometry for replay. When an old task is represented by finitely many stored examples while the corresponding population representation distribution is continuous, KL-based drift can become vacuous under support mismatch even if the two measures remain geometrically close. The next result therefore revisits only the drift branch of Theorem~\ref{thm:synergy_drift}: it keeps the layer-wise representation viewpoint but replaces KL by Wasserstein geometry.

\begin{theorem}[Hierarchical Wasserstein Bound]
\label{thm:wasserstein}
Assume the loss function $\ell(\cdot, y)$ is $\rho_{0}$-Lipschitz and the activation functions $\phi_l$ are $\rho_{l}$-Lipschitz. Then, at time $T$, we have:
\begin{align*}
\big|\mathrm{gen}_W\big|
\le
\min_{l\in\{0,\dots,L\}}
\mathbb E\bigg[
\bar\rho_l(W)\cdot\sum_{i=1}^{T}
\mathcal W_1\!\Big(\hat{P}_{A_l,Y|S^i,W_{1:l}},\,P_{A_l,Y|i,W_{1:l}}\Big)
\bigg],
\end{align*}
where $\bar\rho_l(W):=\rho_0\left(1\vee\prod_{h=l+1}^{L}\rho_h\|W_h\|_{\mathrm{op}}\right).$
\end{theorem}

Theorem~\ref{thm:wasserstein} relaxes the drift component of the generalization bound into a geometric form, complementing rather than replacing the full stability-plasticity-interaction (SPS) decomposition of the optimization term in Theorem~\ref{thm:synergy_drift}. The Wasserstein distance $\mathcal W_1$ quantifies the distribution mismatch between replay-induced representation--label pairs $(A_l,Y)$ and their task-$i$ population counterparts, while  $\bar\rho_l(W)$ measures how strongly the suffix network amplifies this mismatch. This relaxation admits a finite, geometrically interpretable bound even under support mismatch, at the cost of losing the explicit structural decomposition of the optimization term.

The layer minimization reveals a depth-dependent trade-off. At shallow layers, representations tend to be  shared across tasks, keeping $\mathcal W_1$ small, but any residual mismatch is  amplified by a long suffix network, inflating $\bar\rho_l(W)$. At deeper layers, this amplification shrinks while representations become increasingly task-specific, enlarging the drift. The optimal layer $l^\star$ is where these two competing effects are balanced, yielding the tightest bound --- a phenomenon we formalize as the ``generalization funnel layer'' in the following. 

\begin{corollary}[Generalization Funnel Layer]
\label{cor:funnel}
Define a \textbf{Generalization Funnel Layer} as any minimizer of the geometric upper-bound proxy:
\begin{align*}
l^\star 
\in
\mathop{\arg\min}_{l \in \{0, \dots, L\}} \mathbb{E}\bigg[ \rho_0\big(\!\prod_{j=l+1}^L \rho_j \|W_j\|_{\mathrm{op}}\big) \cdot \sum_{i=1}^{T}
\mathcal W_1\big( \hat{P}_{A_l,Y|S^i,W_{1:l}},P_{A_l,Y|i,W_{1:l}} \big) \bigg].
\end{align*}
\end{corollary}
Corollary~\ref{cor:funnel}  does not claim that every architecture possesses a unique, intrinsic funnel layer. Rather, it identifies a principled criterion: minimizing the product of suffix sensitivity and representation drift for selecting which layer to be stabilized.  When drift increases with depth while suffix sensitivity decreases, this product attains its minimum at an intermediate layer, providing a bound-level rationale for intermediate-layer stabilization strategies such as feature distillation or partial freezing \cite{shen2023reducing,sorrenti2023selective}.

\begin{remark}[General layer families]
\label{rem:residual}
The factor $\bar\rho_l(W)$ enters Theorem~\ref{thm:wasserstein} only as an upper bound on the Lipschitz constant of the suffix map $a_l\mapsto\hat y$, and its proof uses nothing about the layers beyond the fact that Lipschitz constants compose. The product $\prod_{h>l}\rho_h\|W_h\|_{\mathrm{op}}$ is thus the closed form of this constant for the feedforward layers of~\eqref{eq:net_recursion}; for any other family of Lipschitz layers the same bound holds with $\rho_h\|W_h\|_{\mathrm{op}}$ replaced by the constant $\kappa_h:=\mathrm{Lip}(g_{W_h})$ of the layer-$h$ map $g_{W_h}:a_{h-1}\mapsto a_h$, giving $\bar\rho_l(W)=\rho_0\big(1\vee\prod_{h>l}\kappa_h\big)$. A residual block $g_{W_h}(u)=P_h u+F_h(u)$, with shortcut $P_h$ (the identity when the dimensions agree) and branch $F_h$, satisfies
\[
\kappa_h\;\le\;\|P_h\|_{\mathrm{op}}+\mathrm{Lip}(F_h),\qquad \mathrm{Lip}(F_h)\le\textstyle\prod_{k}\rho_{h,k}\|W_{h,k}\|_{\mathrm{op}},
\]
the product running over the layers of the branch $F_h$ and any post-addition activation contributing only its own (unit, for ReLU) Lipschitz factor. The residual networks used in our experiments are therefore instances of Theorem~\ref{thm:wasserstein} and Corollary~\ref{cor:funnel}: only the per-layer constant $\kappa_h$ changes, and the drift--sensitivity trade-off is read from the same product.
\end{remark}

\section{Operationalizing the Optimization Term: An SGLD Instantiation}
\label{sec:algorithm}

Theorem~\ref{thm:synergy_drift} leaves the optimization branch $\mathcal S^{(l)}+\mathcal P^{(l)}-\mathcal R^{(l)}+\mathcal C^{(l)}$ in an optimizer-agnostic form. This section makes it concrete for one explicit noisy optimizer, stochastic gradient Langevin dynamics (SGLD) \cite{chen2023stability}. The result is specific to this branch and to SGLD: it bounds the optimization complexity, complementing the geometric treatment of the drift term in Section~\ref{sec:wasserstein_relaxation}.

\subsection{Algorithm-Based Hierarchical Bounds}
\label{subsec:Algorithm-Based Hierarchical Bounds}

For analytical convenience, we condition on the prefix parameters $W_{1:l}$ (the bottom $l$ layers) and analyze the generalization performance of the suffix parameters $\Theta=W_{l+1:L}$. Conducting this analysis for an arbitrary split layer l produces a unified layer-wise framework. Let $\{W_r\}_{r=0}^R$ denote the training trajectory of the SGLD-based CL algorithm for current task $T$, where $W_0$ is initialized from the optimized parameters of task $T-1$. To derive a generalization bound applicable to an arbitrary layer $l$, we fix $W_{1:l}$ and update only the remaining parameters $\Theta_r^{(T)}:=W_{l+1:L,r}\in\mathbb R^{d_{l+1:L}}$. At each iteration $r=1,\cdots,R$, we sample a replay mini-batch $B_r^{\mathrm{old}}$ of size $b_{\mathrm{old}}$ from $\mathcal{M}^{1:T-1}$ and a current-task mini-batch $B_r^{\mathrm{new}}$ of size $b_{\mathrm{new}}$ from $D^T$. Define the corresponding stochastic gradient descent directions $ G_r^{\mathrm{old}} := -\frac{1}{b_{\mathrm{old}}}\sum_{Z\in B_r^{\mathrm{old}}}\nabla_{\Theta}\ell((W_{1:l},\Theta_{r-1}),Z)$ and $G_r^{\mathrm{new}}:= -\frac{1}{b_{\mathrm{new}}}\sum_{Z\in B_r^{\mathrm{new}}}\nabla_{\Theta}\ell((W_{1:l},\Theta_{r-1}),Z).$
The update rule at iteration $r$ can then be formalized by
\[
\Theta_r^{(T)}=\Theta_{r-1}^{(T)}+\eta_r G_r+N_r,\qquad N_r\sim \mathcal N(0,\tau_r^2 I_{d_{l+1:L}}),
\]
where the gradient mixture is $G_r= \lambda_{\mathrm{old}}G_r^{\mathrm{old}}+\lambda_{\mathrm{new}}G_r^{\mathrm{new}}$ with weights $\lambda_{\mathrm{old}}=\frac{T-1}{T}$ and $\lambda_{\mathrm{new}}=\frac{1}{T}$. These weights point along the gradient of $\hat{\mathcal L}(W)$, whose overall scale is absorbed into the learning rate $\eta_r$, so the update shares the minimizer of $\hat{\mathcal L}(W)$. The term $N_r$ is the isotropic Gaussian noise injected at each step.

Notably, in traditional single-task SGLD analysis, the initialization $W_0$ is typically assumed to be independent of the current dataset. However, this independence generally fails in CL. At task $T$, the initialization $\Theta_0^{(T)}$ is inherited from the previous task $\Theta_0^{(T)}=\Theta_R^{(T-1)}$. Consequently, it may already encode information about the current representation $\mathbf{U}_T^{(l)}=(\mathbf{U}_{\mathrm{old}}^{(l)},\mathbf{U}_{\mathrm{new}}^{(l)})$ especially when $\mathbf U_{\mathrm{old}}^{(l)}$ is
constructed by replaying past data. To accurately characterize the optimization at time $T$, we must disentangle the inherited information carried by $\Theta_0^{(T)}$ from the new information created during task-$T$ optimization, and then further decompose the latter along the SGLD trajectory.

\begin{proposition}[Heritage]
\label{prop:cross_task_recursion}
Fix a layer $l$ and condition on $W_{1:l}$. At task $T$, we have
\begin{align*}
I(\mathbf U_T^{(l)};\Theta_R^{(T)}| W_{1:l})
&\le
I(\mathbf U_T^{(l)};\Theta_0^{(T)}| W_{1:l})
   + I(\mathbf U_T^{(l)};\Theta_R^{(T)}| \Theta_0^{(T)},W_{1:l}) \nonumber\\
&=: H_T^{(l)}+\Delta_T^{(l)}.
\end{align*}
\end{proposition}
This proposition ties the SGLD trajectory back to the main decomposition. By the interaction identity, the optimization complexity in Theorem~\ref{thm:synergy_drift} is a single conditional mutual information, $\mathcal S^{(l)}+\mathcal P^{(l)}-\mathcal R^{(l)}=I(\mathbf U_T^{(l)};W_{l+1:L}\mid W_{1:l})=I(\mathbf U_T^{(l)};\Theta_R^{(T)}\mid W_{1:l})$, where $\Theta_R^{(T)}$ is the task-$T$ terminal suffix $W_{l+1:L}$. Proposition~\ref{prop:cross_task_recursion} splits it into a heritage term $H_T^{(l)}$, the dependence already present in the inherited initialization, and an increment $\Delta_T^{(l)}$, the dependence injected by task-$T$ optimization. Crucially, because $\Theta_0^{(t+1)}=\Theta_R^{(t)}$, the heritage term is not generated afresh at task $T$ but carries forward the dependent build up over all previous tasks. Corollary~\ref{cor:Cumulative information budget} makes this accumulation precise as a bound by past increments.
 
\begin{corollary}[Cumulative information budget]
\label{cor:Cumulative information budget}
Fix a layer $l$ and condition on $W_{1:l}$. For any $T\ge 2$, the heritage term admits the following cumulative-budget upper bound:
\[
H_T^{(l)} = I(\mathbf U_T^{(l)};\Theta_0^{(T)}\mid W_{1:l})
\le \sum_{t=1}^{T-1} \Delta_t^{(l)},
\]
where $\Delta_t^{(l)} := I(\mathbf U_t^{(l)};\Theta_R^{(t)}\mid \Theta_0^{(t)},W_{1:l})$ is the within-task information increment
in Proposition~\ref{prop:cross_task_recursion} applied to task $t$.
\end{corollary}
Corollary~\ref{cor:Cumulative information budget} reveals that $H_T^{(l)}$ can be upper bounded by the cumulative sum of past increments. Therefore, to control the cross-task heritage $H_T^{(l)}$, it suffices to control the sequence of within-task increments $\{\Delta_t^{(l)}\}_{t<T}$. This insight reduces the problem of bounding the total complexity to controlling the within-task increment $\Delta_t^{(l)}$. Thus, our analysis now focuses on deriving a sharp, algorithm-dependent bound for general single-task update $\Delta_t^{(l)}$, which we instantiate for SGLD in the next theorem.

\begin{theorem}[Task-Wise SGLD-Based Bound]
\label{thm:Hierarchical SGLD}
Fix a split layer $l\in\{0,\ldots,L-1\}$ and condition on the base parameters $W_{1:l}$, so that the suffix network remains nonempty. Assume the loss is $\sigma$-subgaussian as in Theorem~\ref{thm:synergy_drift}. For task $t$ trained by SGLD for $R_t$ steps, we have
\begin{align*}
\big|\mathrm{gen}_{W^{(t)}}\big|
\le
(t-1)\sqrt{2\sigma^2 \mathcal{K}_t^{(l)}}
+
\sqrt{\frac{2\sigma^2}{N_{\mathrm{eff},t}}
\Big(\mathcal{C}_t^{(l)}+\sum_{s=1}^t \sum_{r=1}^{R_s}\mathbb E\!\left[\frac12\log\det\!\Big(I+\frac{\eta_{s,r}^2}{\tau_{s,r}^2}M^{(l)}_{s,r}\Big)\right]\Big)},
\end{align*}
where $M^{(l)}_{s,r}=\mathbb E\!\big[G_{s,r}G_{s,r}^\top\,\big|\,\Theta^{(s)}_{r-1},W_{1:l}\big]$, and $\mathcal{K}_t^{(l)}$, $\mathcal{C}_t^{(l)}$, and $N_{\mathrm{eff},t}$ denote the corresponding quantities of Theorem~\ref{thm:synergy_drift} with the task horizon $T$ replaced by $t$.
\end{theorem}
Theorem~\ref{thm:Hierarchical SGLD} characterizes the generalization dynamics of SGLD-based CL algorithms from a hierarchical perspective. It operationalizes only the optimization branch: it reduces the suffix information complexity to a cumulative trajectory budget of log-determinants. This budget is governed by the ratio of learning rate to injected noise $\eta^2_{s,r}/\tau^2_{s,r}$ and the conditional second-moment matrix $M^{(l)}_{s,r}$. Unlike standard SGLD bounds that depend only on gradient covariance \cite{negrea2019information}, $M^{(l)}_{s,r}$ explicitly captures the systematic direction of the gradient update. The exact budget is high-dimensional; the next subsection expands it to first order to expose what about replay drives it and to obtain measurable diagnostics.

\subsection{Gradient-Moment Diagnostics of Task Interaction}
\label{sec:Gradient-Moment Diagnostics}

While Theorem~\ref{thm:Hierarchical SGLD} bounds the optimization term through a cumulative log-determinant budget, the mechanism by which replay affects this budget remains implicit. We therefore perform a second-order expansion that separates per-step cost into two orthogonal forces: (i) Optimization Instability, driven by gradient covariance; and (ii) Interaction Cost, driven by the alignment of old-task and new-task mean gradients in a sensitivity-aware metric. The resulting alignment quantities are motivated proxies for task interaction in the SGLD model. This reveals the physical meaning of ``Synergy'' in SGLD: it is not merely a reduction in variance, but a constructive interaction of mean descent directions in a curvature-induced metric.

\textbf{Setup.} Throughout this subsection, we fix a split layer $l$ and condition on the current $(\Theta^{(s)}_{r-1},W_{1:l})$. For each task $s$ and step $r$, let $\alpha_{s,r} := \eta_{s,r}^2/\tau_{s,r}^2$ be the signal-to-noise ratio. Recall that the second moment matrix $M^{(l)}_{s,r}$ is defined over the gradient mixture $G_{s,r} = \lambda_{s,\mathrm{old}}G^{\mathrm{old}}_{s,r} + \lambda_{s,\mathrm{new}}G^{\mathrm{new}}_{s,r}$.

\begin{lemma}[Second-moment split]
\label{lem:second_moment_split}
Define the conditional mean and covariance
\[
\mu^{(l)}_{s,r}\!:=\!\mathbb E[G_{s,r}\mid\! \Theta^{(s)}_{r-1},W_{1:l}],\!
V^{(l)}_{s,r}\!:=\!\mathrm{Cov}(G_{s,r}\mid\! \Theta^{(s)}_{r-1},W_{1:l}).
\]
Then
\[
M^{(l)}_{s,r}
=
V^{(l)}_{s,r}+\mu^{(l)}_{s,r}\mu^{(l)\top}_{s,r}.
\]
\end{lemma}
Lemma~\ref{lem:second_moment_split} separates stochasticity from signal: the covariance term $V^{(l)}_{s,r}$ quantifies gradient instability when learning task $s$, while the rank-one term $\mu^{(l)}_{s,r}\mu^{(l)\top}_{s,r}$ encodes the systematic descent direction, representing the task interaction.

\begin{proposition}[Instability--mean decomposition]
\label{prop:logdet_split}
Let $ A_{s,r}^{(l)}:=I+\alpha_{s,r}V_{s,r}^{(l)}\succ 0$ with $\alpha_{s,r}=\eta_{s,r}^2/\tau_{s,r}^2$. Then
\begin{align*}
\frac12\log\det\!\big(I+\alpha_{s,r}M_{s,r}^{(l)}\big)
=
\underbrace{\frac12\log\det\!\big(A_{s,r}^{(l)}\big)}_{\textbf{instability}}
+
\underbrace{\frac12\log\!\Big(1+\alpha_{s,r}\,\mu_{s,r}^{(l)\top}\big(A_{s,r}^{(l)}\big)^{-1}\mu_{s,r}^{(l)}\Big)}_{\textbf{interaction cost}}.
\end{align*}
\end{proposition}
Proposition~\ref{prop:logdet_split} splits each step's budget into two sources. The instability term is governed by the gradient covariance $V_{s,r}^{(l)}$; a large $V_{s,r}^{(l)}$ marks a sharp region of the landscape, which in CL typically means the replay gradients are inconsistent with the current parameters, resulting in catastrophic forgetting. The interaction cost is the mean-gradient energy, which captures the task alignment but is filtered by $(A_{s,r}^{(l)})^{-1}$; this matrix suppresses the mean gradient $\mu_{s,r}^{(l)}$ in directions where gradients are most unstable. This filtering has a consequence worth stating: when the instability term $\frac12\log\det\!\big(A_{s,r}^{(l)}\big)$ is large, the filter shrinks and the second term falls, yet the reduction is not a gain. It implies that the algorithm discards usable gradient signal along sharp directions rather than making progress, remaining stuck in a sharp region. Consequently, synergistic interaction requires that the two tasks' gradients align within flat, stable directions, which the Euclidean inner product does not distinguish. The next definition supplies a metric that does.
\begin{definition}[Sensitivity metric]
\label{def:metric_alignment}
Define the sensitivity metric
\[
H^{(l)}_{s,r}:=\big(A^{(l)}_{s,r}\big)^{-1}\succ 0.
\]
For vectors $u,v\neq 0$, define the $H^{(l)}_{s,r}$-inner product and norm by
$\langle u,v\rangle_{H}:=u^\top H^{(l)}_{s,r}v$ and $\|u\|_H:=\sqrt{\langle u,u\rangle_H}$,
and the induced cosine
\[
\cos_H(u,v):=\frac{\langle u,v\rangle_H}{\|u\|_H\|v\|_H}\in[-1,1].
\]
\end{definition}
Geometrically, $H_{s,r}^{(l)}$ induces a stability-aware Mahalanobis metric: it discounts directions of high curvature while prioritizing flat subspaces. This is consistent with the intuition above.
\begin{corollary}[Alignment of old/new gradients]
\label{cor:mean_energy_alignment}
Let $\mu^{\mathrm{old},(l)}_{s,r}:=\mathbb E[G^{\mathrm{old}}_{s,r}\mid \Theta^{(s)}_{r-1},W_{1:l}],
\mu^{\mathrm{new},(l)}_{s,r}:=\mathbb E[G^{\mathrm{new}}_{s,r}\mid \Theta^{(s)}_{r-1},W_{1:l}]$, so that $\mu^{(l)}_{s,r}=\lambda_{s,\mathrm{old}}\mu^{\mathrm{old},(l)}_{s,r}+\lambda_{s,\mathrm{new}}\mu^{\mathrm{new},(l)}_{s,r}.$
Then under the sensitivity metric $H^{(l)}_{s,r}$ in Definition~\ref{def:metric_alignment},
\begin{align*}
\mu^{(l)\top}_{s,r}H^{(l)}_{s,r}\mu^{(l)}_{s,r}
=
\lambda_{s,\mathrm{old}}^2\big\|\mu^{\mathrm{old},(l)}_{s,r}\big\|_{H}^2
+\lambda_{s,\mathrm{new}}^2\big\|\mu^{\mathrm{new},(l)}_{s,r}\big\|_{H}^2+
2\lambda_{s,\mathrm{old}}\lambda_{s,\mathrm{new}}
\big\|\mu^{\mathrm{old},(l)}_{s,r}\big\|_{H}
\big\|\mu^{\mathrm{new},(l)}_{s,r}\big\|_{H}
\cos_{H}(\mu^{\mathrm{old},(l)}_{s,r},\mu^{\mathrm{new},(l)}_{s,r}),
\end{align*}
where the information-geometric alignment
\[
\mathrm{Align}=\cos_{H}\Big(\mu^{\mathrm{old},(l)}_{s,r},\mu^{\mathrm{new},(l)}_{s,r}\Big).
\]
\end{corollary}
Corollary~\ref{cor:mean_energy_alignment} gives the interaction term $\mathcal R^{(l)}$ of Theorem~\ref{thm:synergy_drift} to a concrete geometric reading: whether old and new task gradients align determines whether the per-step budget is small or large. When the two gradients align within the stable subspaces identified by $H$ (constructive synergy), they reinforce a common descent direction, the mean-gradient energy $\|\mu\|_H$ is spent on progress, and the budget stays low. When they conflict (interference), the updates fight rather than reinforce: $\|\mu\|_H$ and the instability factor $A_{s,r}^{(l)}$ remain large, the optimization fails to settle, and the cumulative budget surges. Alignment and interference are thus the two regimes that drive the generalization cost down or up, respectively.

This sensitivity-aware alignment is related to the gradient-projection family---GEM, A-GEM, and orthogonal projection \cite{lopez2017gradient,chaudhry2018efficient,saha2021gradient}: both address the conflict between old- and current-task gradients, but ours differs in three concrete ways. First, the gradients compared here are the replay and current mini-batch gradients of the running optimizer, not gradients stored on a fixed past subspace. Second, alignment is read in the sensitivity metric $H_{s,r}^{(l)}=(I+\alpha_{s,r}V_{s,r}^{(l)})^{-1}$ rather than the Euclidean inner product, so unstable directions are discounted instead of all coordinates being weighted equally. Third, the quantity is diagnostic and feeds the soft, per-step mixing rule of Corollary~\ref{cor:Optimal mixing}, not a hard orthogonality constraint on the update. The alignment cosine is therefore a reading of the bound rather than a projection imposed on training.

\begin{figure*}[t]
\centering
\begin{minipage}[c]{0.24\linewidth}
    \centering
    \textbf{Log-Det Budget}
\end{minipage}
\hfill
\begin{minipage}[c]{0.24\linewidth}
    \centering
    \textbf{Alignment $\cos_H$}
\end{minipage}
\hfill
\begin{minipage}[c]{0.24\linewidth}
    \centering
    \textbf{Forgetting Dynamics}
\end{minipage}
\hfill
\begin{minipage}[c]{0.24\linewidth}
    \centering
    \textbf{Statistics over 20 seeds}
\end{minipage}

\begin{minipage}[c]{0.24\linewidth}
    \centering
    \subfloat[MNIST+\label{subfig:main1_mnistpos_logdet}]{%
        \includegraphics[width=\linewidth]{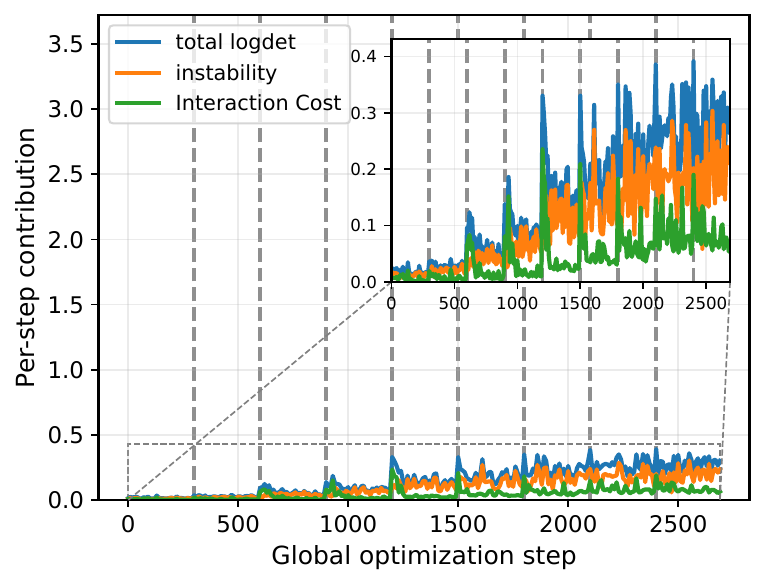}%
    }
\end{minipage}
\hfill
\begin{minipage}[c]{0.24\linewidth}
    \centering
    \subfloat[MNIST+\label{subfig:main3_mnistpos_alignment}]{%
        \includegraphics[width=\linewidth]{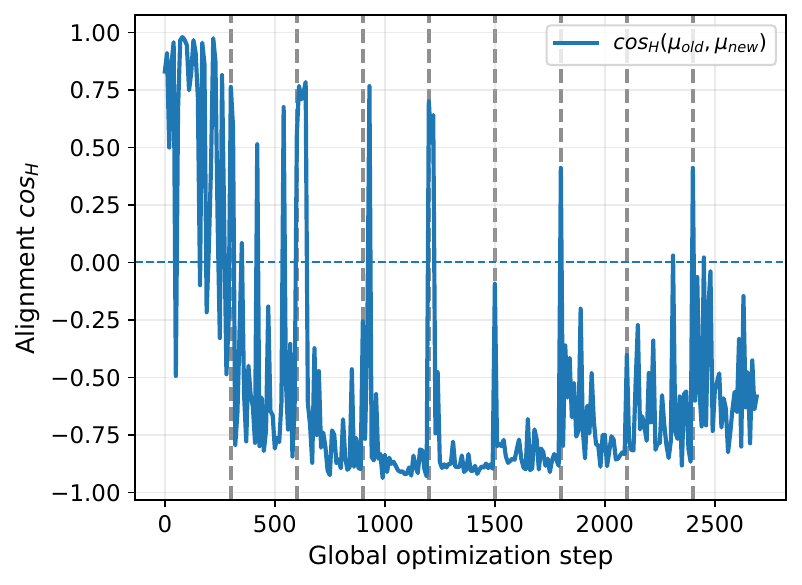}%
    }
\end{minipage}
\hfill
\begin{minipage}[c]{0.24\linewidth}
    \centering
    \subfloat[MNIST+\label{subfig:main5_mnistpos_forgetting}]{%
        \includegraphics[width=\linewidth]{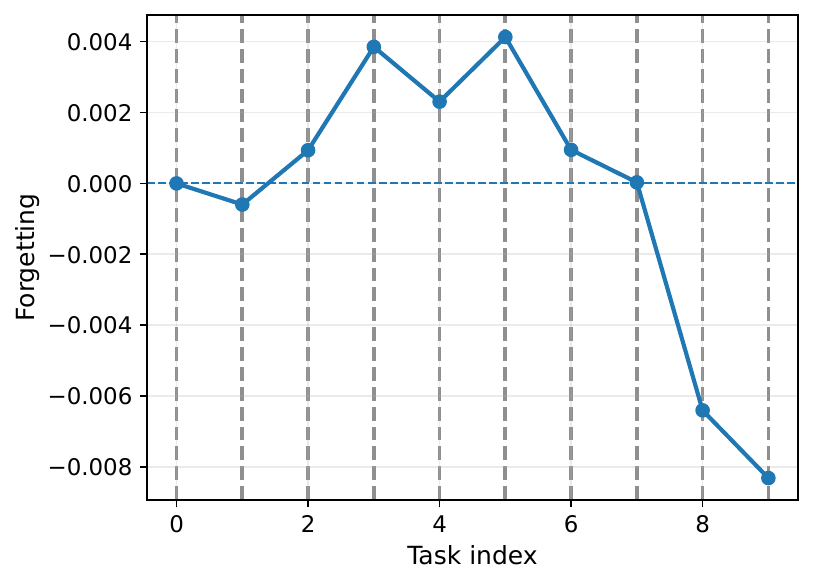}%
    }
\end{minipage}
\hfill
\begin{minipage}[c]{0.24\linewidth}
    \centering
    \subfloat[Alignment Signal $\cos_H$\label{subfig:main7_box_alignment}]{%
        \includegraphics[width=\linewidth]{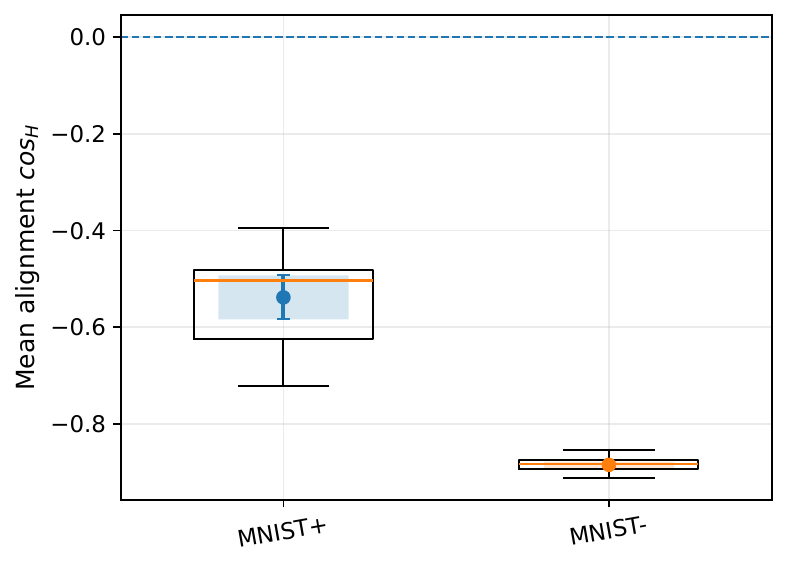}%
    }
\end{minipage}

\vspace{0.1cm}

\begin{minipage}[c]{0.24\linewidth}
    \centering
    \subfloat[MNIST-\label{subfig:main2_mnistneg_logdet}]{%
        \includegraphics[width=\linewidth]{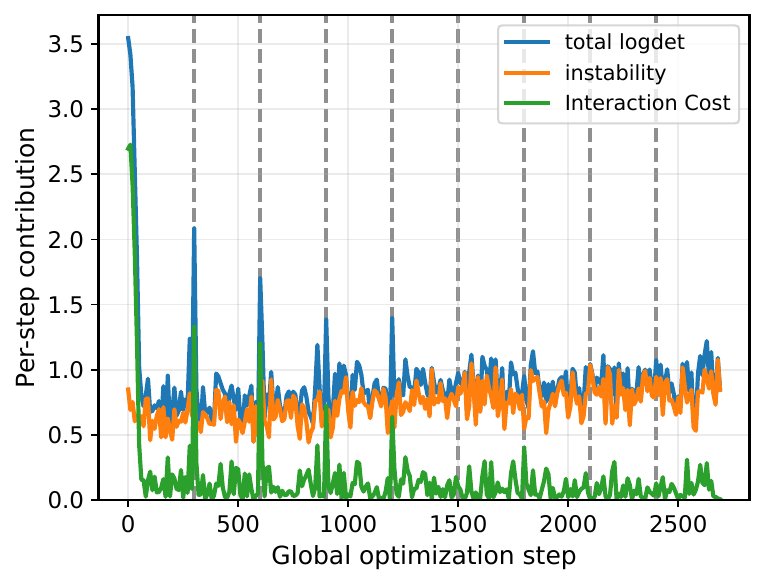}%
    }
\end{minipage}
\hfill
\begin{minipage}[c]{0.24\linewidth}
    \centering
    \subfloat[MNIST-\label{subfig:main4_mnistneg_alignment}]{%
        \includegraphics[width=\linewidth]{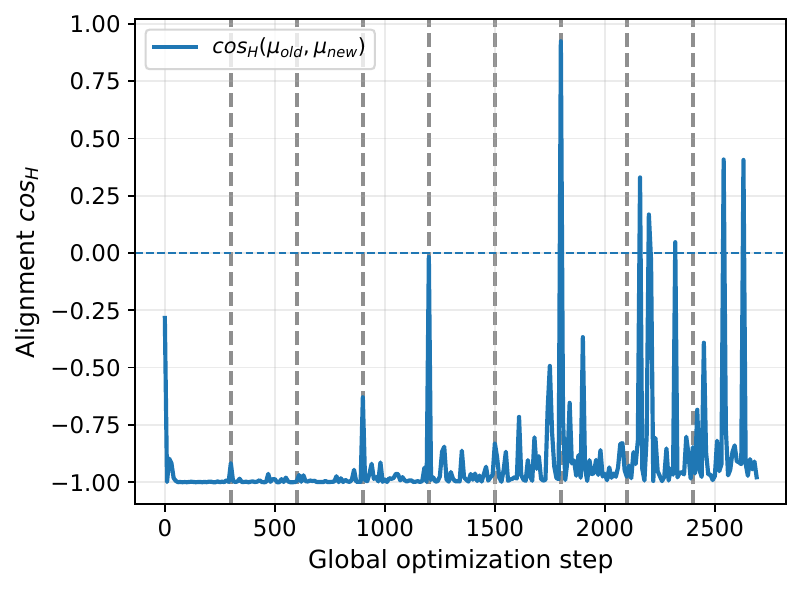}%
    }
\end{minipage}
\hfill
\begin{minipage}[c]{0.24\linewidth}
    \centering
    \subfloat[MNIST- \textbf{100$\times$y-axis}\label{subfig:main6_mnistneg_forgetting}]{%
        \includegraphics[width=\linewidth]{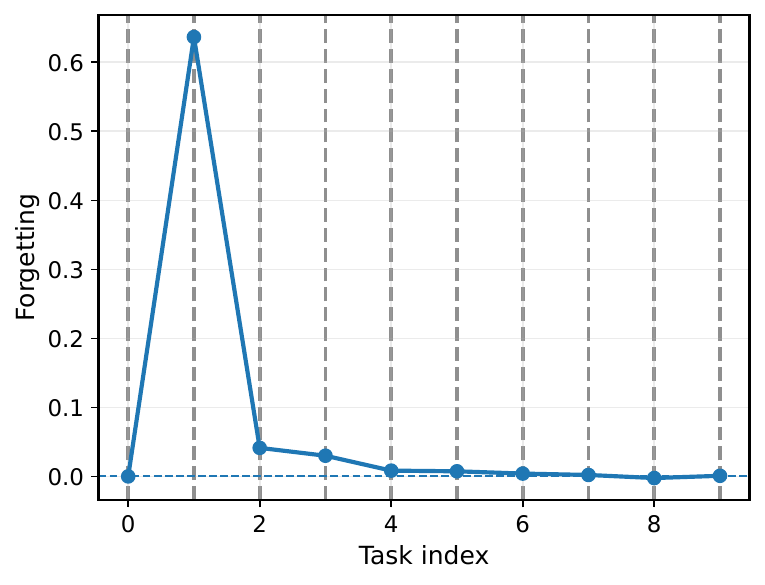}%
    }
\end{minipage}
\hfill
\begin{minipage}[c]{0.24\linewidth}
    \centering
    \subfloat[Budget ratio under $\lambda^{\star}$\label{subfig:main8_lambda_star}]{%
        \includegraphics[width=\linewidth]{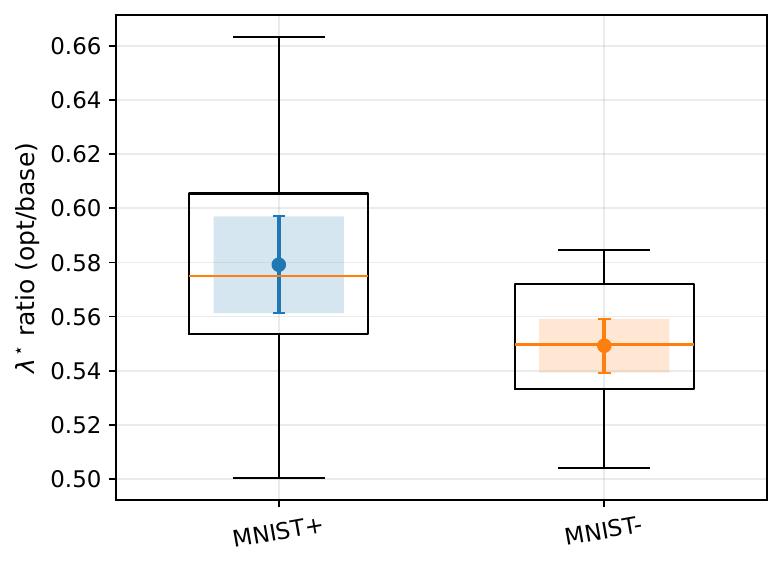}%
    }
\end{minipage}
\caption{
Decoupling Generalization Dynamics: Synergy vs. Interference on MNIST. Vertical dashed lines indicate task transitions.
Top (Rotate-MNIST): Gradient alignment ($\cos_H$) stays well above its interference-regime level (\ref{subfig:main3_mnistpos_alignment}), accompanied by stable retention (\ref{subfig:main5_mnistpos_forgetting}) and a low trajectory budget (\ref{subfig:main1_mnistpos_logdet}).
Bottom (Flip-MNIST): Persistent anti-alignment ($\cos_H \approx -1$, \ref{subfig:main4_mnistneg_alignment}) coincides with a sharp growth of the trajectory budget (\ref{subfig:main2_mnistneg_logdet}) and the onset of catastrophic forgetting (\ref{subfig:main6_mnistneg_forgetting}). (\ref{subfig:main7_box_alignment}) $\cos_H$ reliably distinguishes synergy from interference. (\ref{subfig:main8_lambda_star}) The one-step mixing coefficient $\lambda^{\star}$ (Corollary~\ref{cor:Optimal mixing}) reduces the local per-step budget ratio relative to fixed mixing, with a larger effect under interference; we present $\lambda^\star$ only as a local diagnostic/control signal, not as a global replay policy.}
\label{fig:mnist_dynamics}
\end{figure*}

\subsection{Interpretable Diagnostics}
We conclude our theoretical analysis by translating the trajectory budget into practical scalar signals and a local sensitivity-weighted mixing rule. First, a first-order relaxation of the budget yields decomposable diagnostics:
\begin{corollary}[Scalar diagnostic decomposition]
\label{Scalar diagnostic decomposition}
Conditioning on $(\Theta^{(s)}_{r-1},W_{1:l})$, the per-step trajectory budget admits the scalar upper bound
\begin{align*}
\sum_{s=1}^t\sum_{r=1}^{R_s}\mathbb  E\left[\frac12\log\det\Big(I+\alpha_{s,r}M^{(l)}_{s,r}\Big)\right]
\le
\frac12\sum_{s=1}^t\sum_{r=1}^{R_s}\alpha_{s,r}\,\mathbb E\Big[\mathrm{tr}\big(V^{(l)}_{s,r}\big)+\|\mu^{(l)}_{s,r}\|^2\Big].
\end{align*}
Moreover, the two scalar components admit the explicit decompositions $\mathrm{tr}\big(V^{(l)}_{s,r}\big)=
\lambda_{s,\mathrm{old}}^2\mathrm{tr}\big(\Sigma^{\mathrm{old},(l)}_{s,r}\big)
+\lambda_{s,\mathrm{new}}^2\mathrm{tr}\big(\Sigma^{\mathrm{new},(l)}_{s,r}\big),$ and $\|\mu^{(l)}_{s,r}\|^2 = \lambda_{s,\mathrm{old}}^2\|\mu^{\mathrm{old},(l)}_{s,r}\|^2 +\lambda_{s,\mathrm{new}}^2\|\mu^{\mathrm{new},(l)}_{s,r}\|^2 +2\lambda_{s,\mathrm{old}}\lambda_{s,\mathrm{new}}
\langle\mu^{\mathrm{old},(l)}_{s,r},\mu^{\mathrm{new},(l)}_{s,r}\rangle$.
\end{corollary}
Corollary~\ref{Scalar diagnostic decomposition} yields three monitorable scalar quantities. $\mathrm{tr}(\Sigma^{\mathrm{old},(l)}_{s,r})$ measures gradient inconsistency within the replay buffer under the current representation, with elevated values serving as an early indicator of catastrophic forgetting when replay no longer provides stable optimization anchors. $\mathrm{tr}(\Sigma^{\mathrm{new},(l)}_{s,r})$ indicates the intrinsic stochasticity of the new task gradients.
$\langle\mu^{\mathrm{old},(l)}_{s,r},\mu^{\mathrm{new},(l)}_{s,r}\rangle$ directly quantifies cross-task alignment, distinguishing positive transfer from destructive interference. 

\begin{corollary}[Sensitivity-weighted local mixing]
\label{cor:Optimal mixing}
Fix any sensitivity metric $H^{(l)}_{s,r}$. For $\lambda\in[0,1]$, consider the mixed vector $\mu_{s,r}^{(l)}(\lambda):=\lambda \mu^{\mathrm{old},(l)}_{s,r}+(1-\lambda)\mu^{\mathrm{new},(l)}_{s,r}$ and its metric energy
$\|\mu_{s,r}^{(l)}(\lambda)\|_H^2$.
Then $\|\mu^{(l)}_{s,r}(\lambda)\|_H^2$ is minimized over $\lambda\in[0,1]$ by
\begin{equation*}
\lambda^\star
\!=\!
\mathrm{clip}_{[0,1]}\!\left(\!
\frac{{\mu^{\mathrm{new},(l)}_{s,r}}^\top H(\mu^{\mathrm{new},(l)}_{s,r}-\mu^{\mathrm{old},(l)}_{s,r})}{(\mu^{\mathrm{old},(l)}_{s,r}\!- \!\mu^{\mathrm{new},(l)}_{s,r})^\top H( \mu^{\mathrm{old},(l)}_{s,r}\!- \!\mu^{\mathrm{new},(l)}_{s,r})}
\!\right),
\end{equation*}
with the convention $\lambda^\star=0$ if $\mu^{\mathrm{old},(l)}_{s,r}=\mu^{\mathrm{new},(l)}_{s,r}$ .
\end{corollary}
Corollary~\ref{cor:Optimal mixing} derives a local mixing coefficient $\lambda^\star$ that minimizes the $H$-metric energy of the mixed mean update. It should therefore be interpreted as a one-step geometric diagnostic/control rule rather than as a universal globally optimal replay policy. Without an explicit current-task progress constraint, minimizing this local energy need not optimize the overall continual-learning objective. Under task alignment, the energy is relatively insensitive to the mixing ratio, permitting flexible interpolation between replay and current-task gradients. Under conflict, $\lambda^\star$ identifies the mixture that least excites the sensitive directions encoded in $H$; replay-heavy values should therefore be read as warning signals of severe interference rather than as standalone prescriptions.

\section{Experiments}\label{sec:exp}

In this section, we empirically evaluate the predictions made by Theorems~\ref{thm:synergy_drift}, \ref{thm:wasserstein} and~\ref{thm:Hierarchical SGLD} together with their corollaries. Our goal is to determine whether the bounds and diagnostics derived in Sections~\ref{Main Decomposition}--\ref{sec:algorithm} retain predictive power across (i) controlled settings in which ground-truth distributional objects are accessible, and (ii) standard replay benchmarks in which these objects must be estimated from finite samples.

\subsection{Experimental Details}

\paragraph{Datasets, architectures, and replay methods}
Controlled experiments use a synthetic Gaussian-mixture stream and a binary MNIST stream optimized by SGLD, both of which permit closed-form or large-sample estimation of the replay centroid $Q_{A_l,Y|i,W_{1:l}}$ and of the task population $P_{A_l,Y|i,W_{1:l}}$. Standard-benchmark experiments use Split-CIFAR-100 (10 disjoint 10-class tasks) and Split-TinyImageNet (20 disjoint 10-class tasks). Three replay methods are evaluated as instantiations of the framework: experience replay (ER) \cite{rolnick2019experience}, DER++ \cite{buzzega2020dark}, and iCaRL \cite{rebuffi2017icarl}. All benchmark configurations use buffer capacity $2{,}000$, $20$ epochs per task, mini-batch size $64$, replay batch size $64$, and SGD with learning rate $0.03$, momentum $0.9$, and weight decay $5\times 10^{-4}$. For every (benchmark, method) pair, we report $2$ class orderings $\times\,3$ random seeds $=6$ independent runs.

\paragraph{Depth family and the suffix amplification factor}
The funnel study of Section~\ref{sec:wasserstein_relaxation} additionally trains a family of residual networks, ResNet-$\{10,18,26,34\}$, on Split-CIFAR-100 and Split-TinyImageNet with both ER and DER++, spanning $6$ to $18$ representation cuts; the $12$ (method, dataset, depth) cells with ER and DER++ enter the analysis below. Theory and experiment here concern a single quantity, the Lipschitz constant of the suffix map $a_\ell\mapsto\hat y$, which sets how strongly the network above layer $\ell$ amplifies a representation mismatch. Theorem~\ref{thm:wasserstein} bounds it in closed form by $\bar\rho_\ell(W)$, a bound covered by Remark~\ref{rem:residual} once $\|W_h\|_{\mathrm{op}}$ is read as the per-block Lipschitz constant. Because that worst-case product is many orders of magnitude loose in deep networks, we additionally estimate, for a depth-resolved location, a data-dependent surrogate for the same Lipschitz constant, the on-distribution Jacobian norm of $a_\ell\mapsto\hat y$, which is a lower bound on it.

\paragraph{Performance and diagnostic observables}
Let $A_{t,i}$ denote the test accuracy on task $i$ evaluated immediately after training on task $t$ completes. The standard class-incremental metrics are
$$
\mathrm{AA}_t \triangleq \frac{1}{t}\sum_{i\le t} A_{t,i}, \qquad
\mathrm{AF}_t \triangleq \frac{1}{t-1}\sum_{i<t}\left( \max_{s<t} A_{s,i} - A_{t,i} \right), \quad t \ge 2.
$$
Here $\mathrm{AA}_t$ is the average task accuracy, and $\mathrm{AF}_t$ represents the average forgetting over all previously learned tasks.

After each task transition $t\!\to\!t{+}1$ we additionally compute, for every prior task $i\le t$, the diagnostic quantities: the sensitivity-aware gradient alignment $\cos_H(i,t)$ derived from the SGLD analysis of Section~\ref{sec:algorithm}, the layer-$\ell$ drift estimate $D_\ell(i,t)$ from the geometric relaxation of Section~\ref{sec:wasserstein_relaxation}, the suffix sensitivity $\bar\rho_\ell(W_t)$ from Theorem~\ref{thm:wasserstein}, and the gradient-covariance traces $\mathrm{tr}\big(\Sigma_\mathrm{old}(i,t)\big)$ and $\mathrm{tr}\big(\Sigma_\mathrm{new}(t)\big)$. The central response variable is the pairwise forgetting 
$$
\Delta F_{i,t} \triangleq A_{t,i} - A_{t+1,i}.
$$

\paragraph{Statistical methodology}
Two choices make the benchmark analysis robust to the small-sample artifacts that affect
transition-level studies. First, the unit of analysis is the (old-task, new-task) pair
rather than the task boundary: each Split-TinyImageNet run contributes $190$ pairs and each
Split-CIFAR-100 run $45$, for $3{,}690$ pairwise observations in total, which is two orders of magnitude more than a one-point-per-transition analysis. Second, because any quantity that drifts monotonically with task index would appear spuriously predictive, for each diagnostic we
report the task-controlled partial Pearson correlation with pairwise forgetting after
partialling out the new-task index and the old-task age. To respect the dependence structure of
the data, the partial correlation is computed within each run (over its $45$ or $190$
pairs); the reported $95\%$ bootstrap interval and sign consistency are then taken over the
six runs of each cell, so the interval reflects between-run variability rather than treating
non-independent pairs as independent samples. We compare each partial correlation against two
naive monotone baselines: the raw correlation of forgetting with task index and with old-task age.
Unless stated otherwise, the sensitivity metric $H^{(l)}_{s,r}$ is used in its diagonal
approximation, so that all diagnostics are computed in $O(d)$ time and require no $d\times d$ matrix.

\begin{figure}[t]
\centering
\begin{subfigure}{0.48\textwidth}
    \centering
    \includegraphics[width=\linewidth]{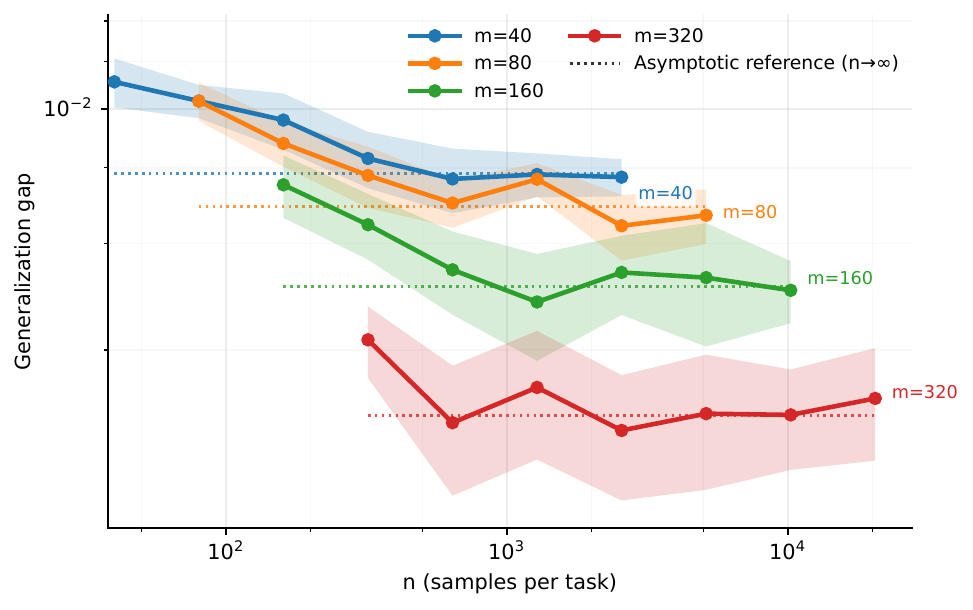}
    \caption{Gap scaling}
    \label{fig:exp_A1_gap_scaling}
\end{subfigure}
\hfill
\begin{subfigure}{0.48\textwidth}
    \centering
    \includegraphics[width=\linewidth]{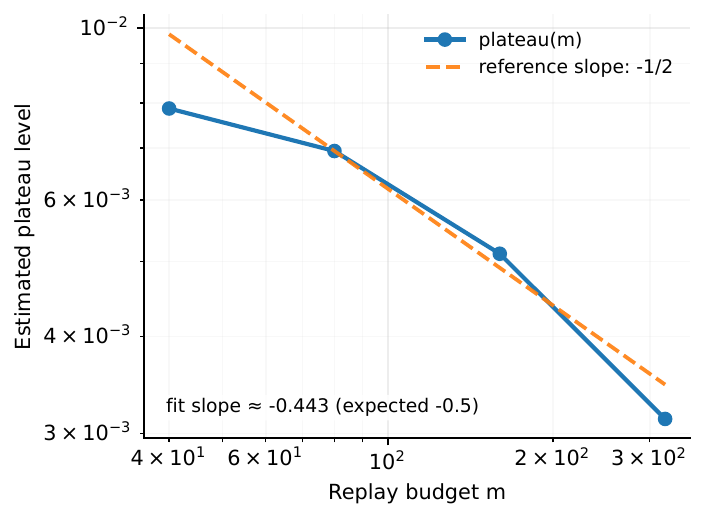}
    \caption{plateau scaling}
    \label{fig:exp_A2_plateau_scaling}
\end{subfigure}
\caption{Scaling laws of the generalization gap for fixed $m$ or $n$.}
\label{fig:2}
\end{figure}

\subsection{Controlled Test of the Effective-Sample-Size Term}
This subsection isolates the optimization-variance branch of Theorem~\ref{thm:synergy_drift}, in which the dependence on memory size $m$, current-task size $n$, and sequence length $T$ enters through the effective sample size $N_\mathrm{eff} = \frac{1}{(T-1)/m + 1/n}.$
Two scaling predictions in $m$ and $n$ follow directly from the prefactor $\sqrt{2\sigma^2/N_\mathrm{eff}}$ in the bound. First, for fixed $(m,T)$, increasing the current-task size $n$ does not drive the variance term to zero but instead yields a finite plateau $\lim_{n\to\infty} N_\mathrm{eff} =\frac{m}{T-1},$
which we refer to as the finite-memory variance floor. Second, in the regime $m \ll n$, $N_\mathrm{eff} \approx m/(T-1)$, so the variance term scales as $m^{-1/2}$ in the buffer size. The empirical gap reported below is the task-averaged discrepancy, equal to the total-risk gap of Theorem~\ref{thm:synergy_drift} divided by the fixed task count $T$.

We construct a $T$-task Gaussian-mixture stream in which each task $\mathcal{D}_i$ defines a binary classification problem between two isotropic Gaussians with task-specific means and a shared covariance. The learner is a linear-Gaussian model, for which the population risk $L(W)$ and the empirical risk $\hat L(W)$ are available in closed form, as are all KL terms appearing in Theorem~\ref{thm:synergy_drift}. The empirical generalization gap $|L(W)-\hat L(W)|$ is averaged over $K=200$ independent replications of the buffer construction and the current-task draw.
Figure~\ref{fig:exp_A1_gap_scaling} plots the empirical gap as a function of $n$ for $T=8$ and three values of $m \in \{40,80,160,320\}$. Each curve flattens as $n$ grows, and the asymptote agrees with the prediction: doubling $n$ in the saturated regime produces a change in the gap that is statistically indistinguishable from zero, while doubling $m$ shifts the asymptote downward by a factor close to $\sqrt{2}$. Figure~\ref{fig:exp_A2_plateau_scaling} reports the same data in the log-log regime, together with an ordinary-least-squares fit of $\log\mathrm{gap}$ versus $\log m$. The fitted slope is $\widehat{\beta} = -0.443\pm0.027$ (95\% bootstrap confidence interval $[-0.498,-0.391]$, $K=200$ replications per point), so the interval lies just above $-1/2$, and the estimate drifts systematically from $-0.454$ to $-0.423$ as the tail threshold used to define the plateau grows. The prediction under test is the $m^{-1/2}$ scaling of the variance \emph{prefactor} $\sqrt{2\sigma^2/N_\mathrm{eff}}$ in isolation. However, what we fit is the full empirical gap, not the prefactor in isolation: it consists of the variance prefactor $\sqrt{2\sigma^2/N_{\mathrm{eff}}}$ multiplied by the square root of the information content $\mathcal S^{(\ell)}+\mathcal P^{(\ell)}-\mathcal R^{(\ell)}+\mathcal C^{(\ell)}$, with a residual drift contribution that persists at finite $n$. Any mild $m$-dependence of that information content shifts the composite exponent of the gap off the prefactor's $-1/2$; the measured value sitting slightly above $-1/2$, and stable across tail thresholds, is consistent with such a small perturbation.

The closed-form setting lets us test a sharper claim: that the bound tracks the gap as its controlling quantities vary, not merely that it attains the correct asymptote. Writing the task-averaged variance proxy as $1/\sqrt{N_\mathrm{eff}}$, we first sweep the memory and sample budgets on a $4\times 4$ grid ($m\in\{40,80,160,320\}$, $n\in\{500,1000,2000,5000\}$, $T=8$). Across the $16$ cells the proxy and the measured gap are not merely rank-correlated (Spearman $\rho=0.94$ at the cell level) but essentially linear: a straight-line fit of the gap against $1/\sqrt{N_\mathrm{eff}}$ gives $R^2=0.97$ (Pearson $r=0.98$), so the proxy reproduces the functional form of the $(m,n)$ dependence, not only its ordering. The variance branch is thus quantitatively faithful in $(m,n)$ up to its loose absolute constant. A complementary $T$-sweep (fixed $m{=}160,\,n{=}1000$, $T\in\{2,4,8,16\}$, $20$ seeds) isolates the $N_\mathrm{eff}$ prefactor from the information content of the bound. We stress at the outset that the variance prefactor $\sqrt{2\sigma^2/N_\mathrm{eff}}=\sqrt{2\sigma^2\big((T-1)/m+1/n\big)}$ grows with $T$, as $\sqrt{T/m}$ in the total-risk gap; the quantity plotted in Fig.~\ref{fig:bound_gap_T} is the task-averaged proxy $\tfrac1T\sqrt{2\sigma^2/N_\mathrm{eff}}\sim 1/\sqrt{Tm}$, which decreases only because the fixed $1/T$ normalization (the reported gap is the total-risk gap of Theorem~\ref{thm:synergy_drift} divided by $T$) rescales every term alike. In either convention the prefactor's $T$-growth is at most $\sqrt T$, whereas the measured task-averaged gap rises from $0.0045$ at $T{=}2$ to $0.0156$ at $T{=}16$---faster than the prefactor and in the opposite direction to its task-averaged proxy. The prefactor alone therefore cannot account for the $T$-dependence. The growth is carried by the bound's information content: the replay-centroid drift $\mathcal K^{(\ell)}$, which enters the drift term with an explicit $(T-1)$ factor, and the optimization information $\mathcal S^{(\ell)}+\mathcal P^{(\ell)}-\mathcal R^{(\ell)}$, whose heritage component is non-decreasing in the number of tasks by Corollary~\ref{cor:Cumulative information budget}. The $T$-sweep thus separates the finite-memory rate prefactor $N_\mathrm{eff}$ (validated by the $(m,n)$ grid above) from the information content that accumulates with $T$.

\begin{figure}[t]
\centering
\begin{subfigure}{0.40\linewidth}\centering\includegraphics[width=\linewidth]{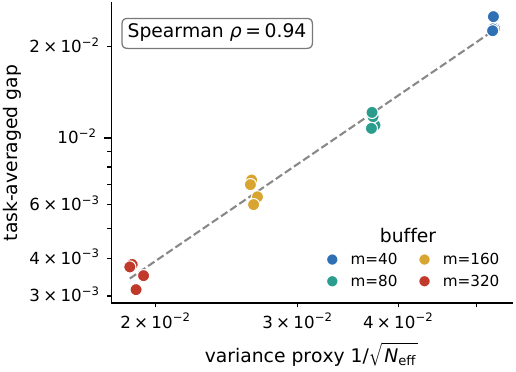}\caption{Variance proxy vs.\ gap over the $(m,n)$ grid}\label{fig:bound_gap_grid}\end{subfigure}\hfill
\begin{subfigure}{0.46\linewidth}\centering\includegraphics[width=\linewidth]{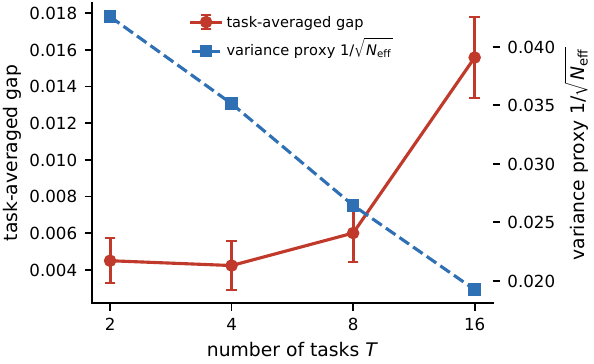}\caption{$T$-sweep: gap rises while the variance proxy falls}\label{fig:bound_gap_T}\end{subfigure}
\caption{The bound is loose in constant but order-correct, and its rate prefactor separates cleanly from its information content.}
\label{fig:bound_gap}
\end{figure}
\subsection{The Drift--Sensitivity Trade-off and the Empirical Funnel}
We now turn to the geometric drift component developed in Section~\ref{sec:wasserstein_relaxation}. This subsection examines whether such a funnel is empirically observable, and whether its location is consistent across controlled and standard-benchmark settings.

Three quantities are tracked at each layer index $\ell$ over the course of training. The drift $D_\ell$ is estimated as a sliced Wasserstein-1 distance between the buffer-induced and the population-induced empirical distributions of $(A_\ell, Y)$, averaged over old tasks. Suffix sensitivity is calculated in two ways. The first option is the worst-case factor $\bar\rho_\ell(W)$ from Theorem~\ref{thm:wasserstein}. We construct this factor by taking the product of layer-wise spectral norms across layers $\ell+1$ through $L$, and every operator norm is computed via power iteration. This is a certified Lipschitz bound that we use for the controlled Gaussian setting in Fig.~\ref{fig:exp_B1_funnel_gen_ci}. The second option relies on the on-distribution Jacobian norm of the suffix map $a_\ell\mapsto\hat y$. This data-driven quantity acts as a lower bound for the Lipschitz constant. We apply it to find the funnel in the deep residual networks of Fig.~\ref{fig:funnel_curves}, since the first method produces results that are too loose to be useful. We next compute the product proxy $\mathrm{Gen}_\ell = D_\ell\cdot(\text{suffix sensitivity})$ for each layer, and the empirical funnel $\hat\ell^\star$ is determined by taking the argmin of this proxy.

\begin{figure*}[t]
\centering
\begin{subfigure}{0.40\textwidth}
    \centering
    \includegraphics[width=\linewidth]{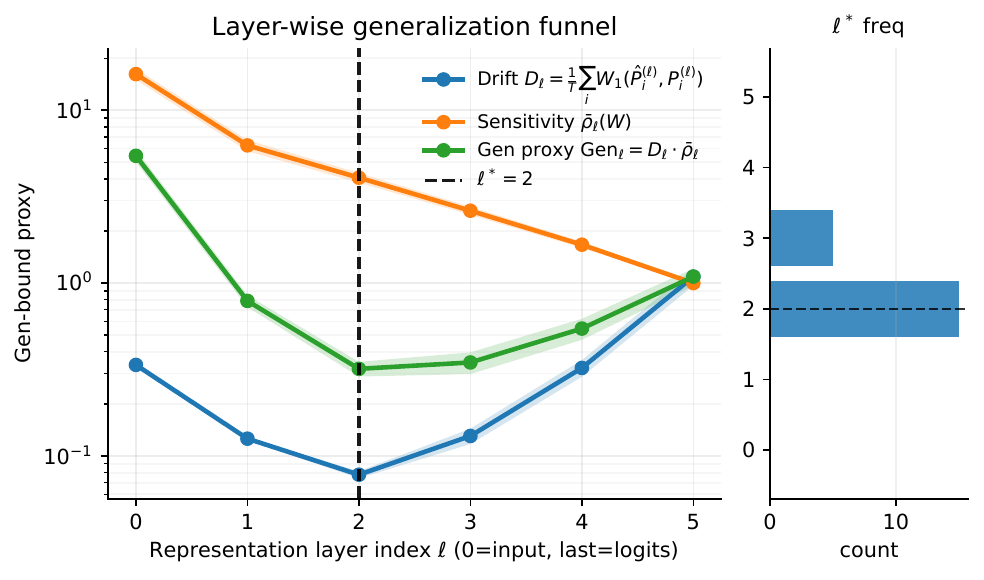}
    \caption{Controlled Gaussian (closed-form populations)}
    \label{fig:exp_B1_funnel_gen_ci}
\end{subfigure}
\hfill
\begin{subfigure}{0.57\textwidth}
    \centering
    \includegraphics[width=\linewidth]{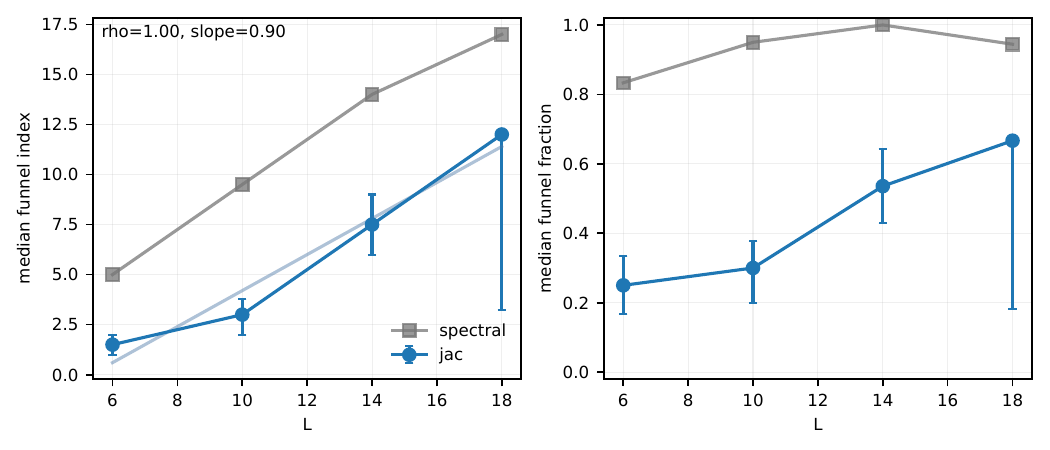}
    \caption{Funnel location moves deeper with network depth}
    \label{fig:funnel_depth}
\end{subfigure}

\vspace{1.5em}

\begin{subfigure}{\textwidth}
    \centering
    \includegraphics[width=\linewidth]{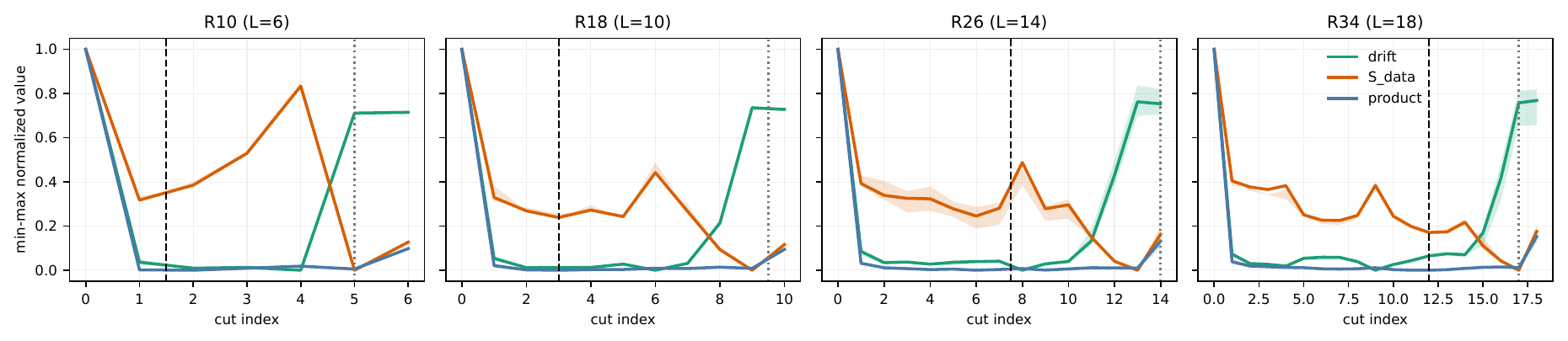}
    \caption{ResNet-$\{10,18,26,34\}$: drift, suffix amplification, and their product across depth}
    \label{fig:funnel_curves}
\end{subfigure}

\caption{The drift--sensitivity trade-off and the empirical funnel. 
    (a) On the controlled Gaussian setting with closed-form populations, the product is minimized at an interior layer ($\ell^\star=2$). 
    (b) The same trade-off holds in deep residual networks, where suffix amplification is measured on-distribution, with the minimum appearing in a broad interior region at every depth. 
    (c) This interior region moves systematically deeper as the network depth increases.}
\label{fig:funnel}
\end{figure*}

Figure~\ref{fig:exp_B1_funnel_gen_ci} reports the three curves on the Controlled Gaussian, where ground-truth populations $P_{A_\ell, Y \mid i, W_{1:\ell}}$ are accessible. The drift $D_\ell$ decreases first and then increases, while $\bar\rho_\ell(W)$ falls; their product is convex with minimizer concentrated at $\ell^\star = 2$ on a four-layer architecture, and the two ends $\ell\in\{0,L\}$ are never selected. In low dimensions, where the worst-case factor $\bar\rho_\ell$ has small dynamic range, the predicted interior minimum is directly observable.

Figure~\ref{fig:funnel_curves} extends the construction to deep residual networks---ResNet-$\{10,18,26,34\}$---across both replay methods and both benchmarks, for $12$ (method, dataset, depth) cells. Two findings hold uniformly. First, the drift forms a pronounced ``bathtub'': $D_\ell$ is large at both the input and the final representation and small across a broad interior, so the drift--amplification product is minimized strictly inside the network in every one of the $12$ cells. Second, the interior minimizer moves systematically deeper as the network deepens: its absolute index increases with $L$, and the Spearman's $\rho$ between funnel index and depth falls between $0.8$ and $1.0$ across all three method--dataset pairs (Fig.~\ref{fig:funnel_depth}). Because the product is flat across a broad basin where values within $20\%$ of the minimum span $30$--$47\%$ of the interior, we report the funnel as an interior region rather than a single sharply-resolved layer.

A feature-distillation probe that stabilizes individual blocks, consistent with Corollary~\ref{cor:funnel}, confirms that stabilizing an interior block improves over the no-distillation baseline and over stabilizing either end, but the two interior blocks nearest the basin are not statistically separated under our seed budget---as expected of a minimizer lying in a flat interior region. On real benchmarks the funnel proxy is the weakest of our diagnostics (Table~\ref{tab:partial-corr}): its partial correlation with forgetting is significant and correctly signed only for ER ($-0.37$ and $-0.30$), and is weak or absent for DER++ and iCaRL. We therefore present the funnel as a structural result about the drift--sensitivity trade-off in controlled and depth-resolved settings, and not on the same footing as the alignment diagnostic of Section~\ref{sec:sgld-diag}, whose benchmark transfer is both stronger and consistent across the gradient-replay methods.

\subsection{SGLD Optimization-Branch Diagnostics}
\label{sec:sgld-diag}
Theorem~\ref{thm:Hierarchical SGLD} reduces the optimization branch of Theorem~\ref{thm:synergy_drift} to a trajectory-level log-determinant budget, and Lemma~\ref{lem:second_moment_split} and Proposition~\ref{prop:logdet_split} decompose the per-step
cost into a covariance-driven instability term and a mean-driven interaction cost governed by the sensitivity-aware alignment $\cos_H$. We evaluate these diagnostics at two levels: first on a controlled MNIST stream where the gradient moments admit clean estimation and the
mechanism can be examined directly, and then on standard replay benchmarks where every quantity
must be estimated under the noise of practical pipelines.

\subsubsection{Controlled MNIST, the mechanism in isolation}
On the controlled MNIST stream the gradient moments are estimable, so the decomposition of Proposition~\ref{prop:logdet_split} can be examined mechanism by mechanism (Fig.~\ref{fig:mnist_dynamics}).
In the positive-transfer regime (Rotate-MNIST), the per-run mean alignment $\cos_H$ sits well above its interference-regime level ($\approx-0.5$ versus $\approx-0.85$ on the box-plots of Fig.~\ref{fig:mnist_dynamics}(d)),
the per-step interaction cost and the gradient covariance remain small, and retention is stable.
In the interference regime (Flip-MNIST), $\cos_H$ collapses to strongly negative values (per-run mean $\approx-0.85$), the log-determinant trajectory
budget grows sharply, and forgetting rises. The alignment signal cleanly separates the two regimes
across $20$ seeds (Fig.~\ref{fig:mnist_dynamics}(d)), consistent with Corollary~\ref{cor:mean_energy_alignment} read as a statement about relative alignment: the synergy regime sustains a markedly higher (less negative) $\cos_H$ than the interference regime. We read the corollary's sign prediction in this relative sense, since under the diagonal-$H$ approximation the per-run mean $\cos_H$ is negative in both regimes; it is their separation, not an absolute sign, that the corollary ties to constructive versus destructive interaction and that predicts forgetting. We note
one diagnostic caveat. Under severe interference the model collapses toward chance, so once train
and test risks reach the same entropy floor the terminal forgetting metric saturates and becomes
uninformative; the trajectory budget and $\cos_H$ remain informative precisely in this regime,
which is why we treat them as the primary diagnostics rather than terminal forgetting alone.

\paragraph{The curvature-aware metric is not the Euclidean inner product}
A natural objection is that $\cos_H$ merely re-expresses the Euclidean gradient cosine $\cos_E$ used by orthogonal-projection methods (GEM/OGD). The controlled streams refute this directly: because both quantities are computed at every step from the same replay and current-task gradients, the comparison is not a post-hoc re-instrumentation. On Rotate/Flip-MNIST, $\cos_H$ separates the synergy and interference regimes more sharply than $\cos_E$ (standardized between-regime mean difference $3.8$ vs.\ $3.0$) and predicts per-run forgetting more strongly ($-0.89$ vs.\ $-0.84$); the same ordering holds on the synthetic Gaussian stream ($\cos_H$ separation $2.9$ vs.\ $2.0$, forgetting correlation $-0.64$ vs.\ $-0.55$). These gaps are statistically reliable rather than artifacts of the seed budget: a paired bootstrap over runs places the $\cos_H-\cos_E$ difference above zero with a $95\%$ CI that excludes zero for both the regime separation ($\Delta d$ CI $[0.7,1.4]$ on MNIST) and the forgetting correlation ($\Delta|r|$ CI $[0.02,0.09]$ on MNIST and $[0.01,0.16]$ on the synthetic stream). The improvement is modest but consistent and significant, just as the theory predicts. The curvature reweighting $H=(I+\alpha V)^{-1}$ discounts alignment along unstable directions (Proposition~\ref{prop:logdet_split}). Consequently, $\cos_H$ is an intrinsic measure derived from our bound, not merely a hidden variant of the Euclidean OGD cosine.

\begin{figure}[t]
\centering
\begin{subfigure}{0.49\linewidth}\centering\includegraphics[width=\linewidth]{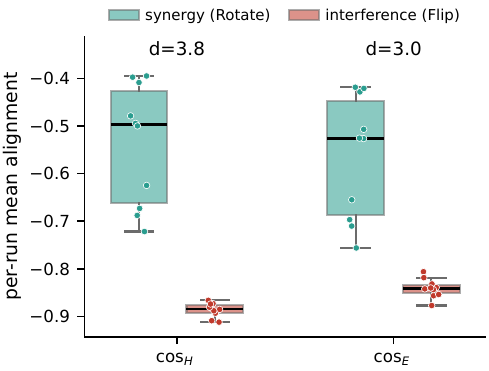}\caption{Regime separation: $\cos_H$ vs.\ $\cos_E$}\label{fig:cosh_sep}\end{subfigure}\hfill
\begin{subfigure}{0.49\linewidth}\centering\includegraphics[width=\linewidth]{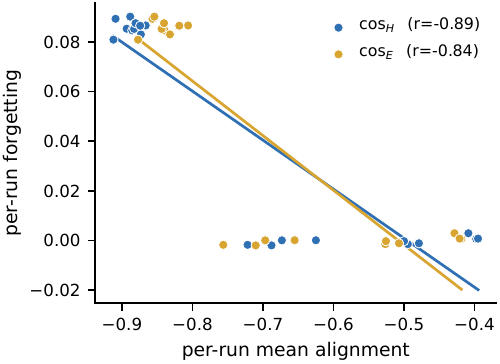}\caption{Forgetting prediction: $\cos_H$ vs.\ $\cos_E$}\label{fig:cosh_pred}\end{subfigure}
\caption{The sensitivity-aware alignment $\cos_H$ improves on the Euclidean gradient cosine $\cos_E$ (the OGD/GEM quantity), both logged from identical gradients. (a) $\cos_H$ separates synergy from interference more sharply (larger between-regime effect size). (b) $\cos_H$ is a stronger predictor of per-run forgetting. Rotate/Flip-MNIST; the synthetic Gaussian stream shows the same ordering.}
\label{fig:cosh_ablation}
\end{figure}

\subsubsection{Standard benchmarks: do the diagnostics track forgetting?}
We now test the central empirical claim of the optimization branch: that the abstract interaction
term $\mathcal{R}^{(l)}$, made operational through the alignment $\cos_H$ of
Corollary~\ref{cor:mean_energy_alignment}, tracks forgetting once estimated on standard replay benchmarks.
Table~\ref{tab:partial-corr} and Fig.~\ref{fig:forest} report the task-controlled partial
correlations of the diagnostics with pairwise forgetting across all six (benchmark, method) cells.

\begin{table*}[t]
\centering
\caption{Task-controlled partial Pearson correlation between theory-derived diagnostics and
pairwise forgetting $\Delta F_{i,t}$, after partialling out the new-task index and the old-task age.
Each cell aggregates $6$ runs ($2$ class orders $\times\,3$ seeds); $^{\dagger}$ marks a $95\%$
bootstrap CI that excludes zero. The last two columns are naive monotone baselines (raw correlation
of forgetting with task index / old-task age). Negative values are theory-consistent: alignment of
old/new gradients reduces forgetting. ``--'' denotes a quantity not recorded for iCaRL under our
diagnostic instrumentation, whose distillation-based objective does not expose the same
replay/current gradient split as ER and DER++.}
\label{tab:partial-corr}
\begin{tabular}{l rrrr rr}
\toprule
Benchmark / Method & $\cos_H$ & $\langle\mu_{\mathrm{old}},\mu_{\mathrm{new}}\rangle$
 & $\mathrm{tr}(\Sigma_{\mathrm{old}})$ & funnel $D_\ell\bar\rho_\ell$
 & task-idx (null) & old-age (null)\\
\midrule
S-CIFAR100 / ER        & $-0.948^{\dagger}$ & $-0.915^{\dagger}$ & $-0.840^{\dagger}$ & $-0.372^{\dagger}$ & $-0.29$ & $-0.62$\\
S-CIFAR100 / DER++     & $-0.794^{\dagger}$ & $-0.561^{\dagger}$ & $-0.185$           & $-0.043$           & $-0.04$ & $-0.48$\\
S-CIFAR100 / iCaRL     & $+0.076$           & $+0.062$           & --                 & $-0.088$           & $-0.43$ & $-0.11$\\
S-TinyImageNet / ER    & $-0.796^{\dagger}$ & $-0.281^{\dagger}$ & $-0.261^{\dagger}$ & $-0.296^{\dagger}$ & $-0.23$ & $-0.44$\\
S-TinyImageNet / DER++ & $-0.645^{\dagger}$ & $-0.230^{\dagger}$ & $-0.098$           & $-0.127^{\dagger}$ & $-0.13$ & $-0.41$\\
S-TinyImageNet / iCaRL & $+0.132^{\dagger}$ & $-0.156^{\dagger}$ & --                 & $+0.074^{\dagger}$ & $-0.28$ & $-0.10$\\
\bottomrule
\end{tabular}
\end{table*}

\begin{figure}[t]
\centering
\includegraphics[width=0.7\textwidth]{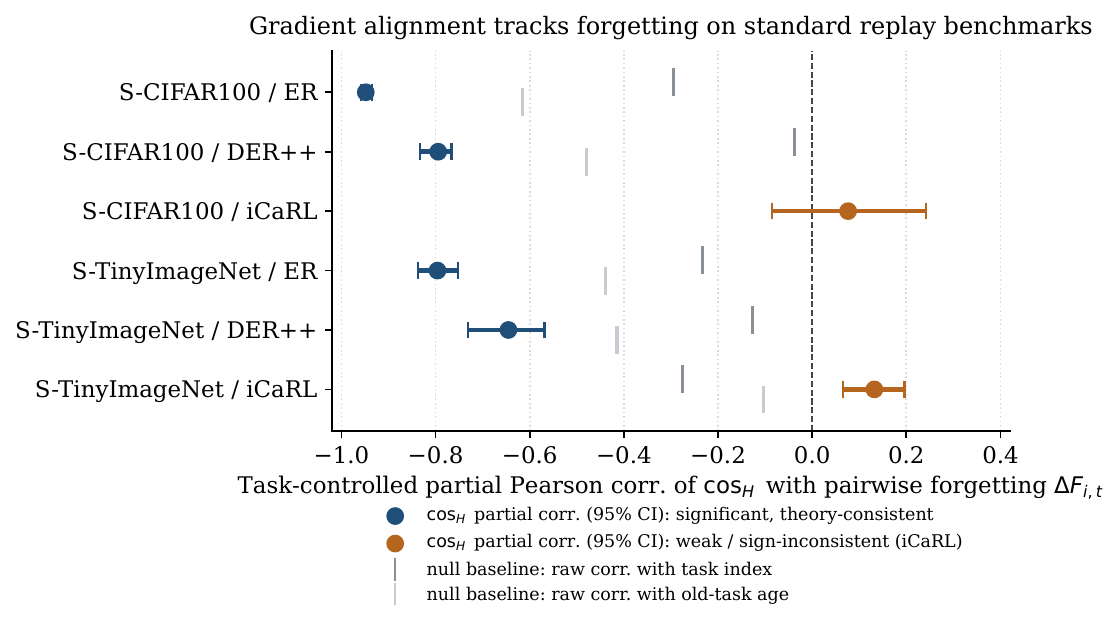}
\caption{Task-controlled partial Pearson correlation of $\cos_H$ with pairwise forgetting on the
six (benchmark, method) cells, with $95\%$ bootstrap intervals. Blue: significant and
theory-consistent (CI excludes zero, negative sign); orange: weak or sign-inconsistent (iCaRL).
Grey ticks are the two naive monotone baselines (raw correlation of forgetting with task index and
with old-task age). For ER and DER++ the alignment signal is several times stronger than either
monotone baseline, on both benchmarks.}
\label{fig:forest}
\end{figure}

\paragraph{Alignment predicts forgetting for gradient-replay methods}
For the two methods whose updates are literally a mixture of replay and current-task gradients---ER
and DER++---the alignment term is a strong and stable predictor of forgetting on both benchmarks. The
partial correlation of $\cos_H$ with pairwise forgetting is $-0.948$ (ER) and $-0.794$ (DER++) on
Split-CIFAR-100, and $-0.796$ (ER) and $-0.645$ (DER++) on Split-TinyImageNet, with confidence
intervals that exclude zero and the same sign in all six runs of every cell. The sign is exactly
the one the theory predicts: when old- and new-task gradients align in the stability-aware metric
$H$, forgetting is small; when they anti-align, forgetting is large. Crucially, this is not an
artifact of generic task-order drift. The task-controlled $\cos_H$ signal is several times stronger
than the two naive monotone baselines (the raw correlation of forgetting with task index never
exceeds $0.29$ in magnitude, and with old-task age never exceeds $0.62$), confirming that the
predictive content survives after the dominant monotone trends are removed. The raw alignment inner
product $\langle\mu_{\mathrm{old}},\mu_{\mathrm{new}}\rangle$ and, for ER, the replay-gradient
inconsistency $\mathrm{tr}(\Sigma_{\mathrm{old}})$ ($-0.840$ on Split-CIFAR-100) are theory-consistent
as well, while the geometric funnel proxy transfers more weakly but with the correct sign for ER on
both benchmarks ($-0.372$ and $-0.296$).

\paragraph{A scope boundary, stated honestly}
The iCaRL columns are included deliberately as a boundary of the framework rather than omitted.
iCaRL classifies by nearest-mean-of-exemplars and is trained with knowledge distillation, so its
forgetting mechanism departs from the replay-plus-current gradient mixture that the SGLD model of
Section~\ref{sec:algorithm} assumes. Consistently, $\cos_H$ is no longer a reliable predictor for iCaRL:
its CI includes zero on Split-CIFAR-100 and the sign even reverses on Split-TinyImageNet. The
raw alignment $\langle\mu_{\mathrm{old}},\mu_{\mathrm{new}}\rangle$ still retains the theory-consistent
negative sign after partialling ($-0.156$ on Split-TinyImageNet), which is what one would expect if
only the part of iCaRL's update that resembles a gradient step carries the alignment signal. This is
the honest reading of the framework: the gradient-alignment diagnostic is most faithful precisely
for methods whose optimization matches the modeled dynamics, and weaker for exemplar-classifier
methods---a delimitation the theory itself anticipates, rather than a failure to be hidden.

\paragraph{Correlation with actual old-task test error}
Finally, the geometric branch also transfers to actual test error rather than only to the forgetting
increment. Merging the diagnostic logs with the accuracy logs on Split-CIFAR-100 with iCaRL, the
layer-wise drift and the funnel proxy correlate with current old-task test error at $+0.262$
($95\%$ CI $[0.132,0.392]$) and $+0.257$ ($[0.128,0.385]$) respectively. Taken together, the benchmark study addresses its core motivating question of whether the quantities appearing in the bound retain meaning when estimated under realistic conditions. It delivers affirmative conclusions for the alignment term and gradient-replay methods, while clearly defining the limitations of the diagnostics.

\section{Limitations}
We do not claim that the bounds are numerically tight for deep networks: as with all
mutual-information generalization bounds, the absolute multiplicative constants are loose. What our
experiments validate is sharper and, we argue, more useful---three falsifiable predictions about the
structure of the gap. First, the predicted functional form: the variance floor scales as $m^{-1/2}$, the measured exponent is $-0.443$ (CI $[-0.498,-0.391]$), and the bound is order-correct beyond this asymptote: across a memory--sample grid, its variance proxy ranks the measured gap with Spearman $0.94$. Second, the predicted geometry: the drift--sensitivity trade-off, with suffix sensitivity measured on-distribution, is minimized strictly inside the network in all $12$ (method, dataset, depth) cells we examine, and the interior region moves systematically deeper as depth grows; we claim an interior region and its depth trend, located empirically rather than certified by the worst-case Lipschitz factor, which is itself too loose to localize a funnel in deep networks. Third, the predicted monotone relationship: the alignment term that the bound makes responsible for the interaction
component negatively tracks forgetting on real benchmarks, with partial correlations from $-0.65$ to
$-0.95$ for gradient-replay methods. A bound can be loose in its constants yet correct in the
dependencies it predicts; our experiments target the latter.

\section{Conclusion}
We developed a layer-wise information-theoretic theory of replay generalization in continual learning. The central message is that replay generalization is governed by two coupled mechanisms: a replay-induced representation drift created by finite memory, and an optimization-complexity term created by shared downstream parameters. The main Synergy--Drift theorem isolates these two mechanisms in a single decomposition. The Wasserstein result revisits the first branch when KL-based drift is ill-posed under support mismatch, while the SGLD result instantiates the second branch through trajectory-level gradient moments. Together, these results provide a common language for studying replay bias, task interaction, and layer-wise sensitivity in continual learning.

\bibliographystyle{IEEEtran}
\bibliography{IEEEexample}

\newpage
\appendices
\onecolumn
\section{Prerequisite Definitions and Lemmas}
\begin{definition}[$\sigma$-sub-Gaussian]
    A random variable $X$ is $\sigma$-sub-Gaussian if for any $\lambda$, $\ln \mathbb{E}[e^{\lambda (X-\mathbb{E}X) }]\leq \lambda^2 \sigma^2/2$.
\end{definition}
\begin{definition}[Kullback-Leibler Divergence]
    Let $P$ and $Q$ be probability measures on a measurable space $(\mathcal{X}, \Sigma)$. If $P$ is absolutely continuous with respect to $Q$ (denoted $P \ll Q$), the Kullback-Leibler (KL) divergence from $P$ to $Q$ is defined as: $D_{\mathrm{KL}}(P \| Q) \triangleq \int_{\mathcal{X}} \log \left( \frac{\mathrm{d}P}{\mathrm{d}Q} \right) \mathrm{d}P$, where $\frac{\mathrm{d}P}{\mathrm{d}Q}$ denotes the Radon-Nikodym derivative.
\end{definition}
\begin{definition}[Mutual Information]
     Let $X$ and $Y$ be random variables taking values in measurable spaces $\mathcal{X}$ and $\mathcal{Y}$, respectively. Let $P_{X,Y}$ denote their joint distribution on the product space $\mathcal{X} \times \mathcal{Y}$, and let $P_X$ and $P_Y$ denote the corresponding marginal distributions. The mutual information between $X$ and $Y$ is defined as the KL divergence between the joint distribution and the product of the marginals: $I(X; Y) \triangleq D_{\mathrm{KL}}(P_{X,Y} \| P_X \otimes P_Y).$
\end{definition}
\begin{lemma}[Interaction information identity]
\label{lem:interaction}
    For any random variables $A,B,C,D$,
    \[
    I(A,B;C\mid D)=I(A;C\mid D)+I(B;C\mid D)-I(A;B;C\mid D).
    \]
\end{lemma}
\begin{definition}[Wasserstein Distance]
    Let $P$ and $Q$ be probability measures on $\mathcal{X} \subseteq \mathbb{R}^k$. The Wasserstein-$p$ distance ($p \geq 1$) is defined as:  $\mathcal{W}_p(P, Q) \triangleq \left( \inf_{\gamma \in \Gamma(P, Q)} \mathbb{E}_{(x, x') \sim \gamma} \left[ \|x - x'\|^p \right] \right)^{1/p},$ where $\Gamma(P, Q)$ is the set of all joint distributions on $\mathcal{X} \times \mathcal{X}$ with marginals $P$ and $Q$, and $\|\cdot\|$ denotes the Euclidean norm.
\end{definition}
\begin{lemma}[Donsker-Varadhan formula(Theorem 5.2.1 in \cite{gray2011entropy})]
\label{lem:DV}
     Let $P$ and $Q$ be probability measures on $\mathcal{X}$ such that $P \ll Q$. Then:
     \begin{equation*}
         D_{\mathrm{KL}}(P \| Q) = \sup_{f \in L_\infty(\mathcal{X})} \left\{ \mathbb{E}_{P}[f] - \log \mathbb{E}_{Q}\left[e^{f}\right] \right\},
     \end{equation*}
     where the supremum is taken over the set of all bounded measurable functions $f: \mathcal{X} \to \mathbb{R}$. If $X$ is $\sigma$-sub-Gaussian under $Q$, then
        \begin{equation*}
            \left| \mathbb{E}_P[X] - \mathbb{E}_Q[X] \right| \le \sigma \sqrt{2 D_{\mathrm{KL}}(P \| Q)}.
        \end{equation*}
\end{lemma}
\begin{lemma}[Lemma 1 in \cite{harutyunyan2021information}]
\label{lem:harutyunyan2021information}
    Let $(X, Y) \sim P_{XY}$ and let $\bar{Y}$ be an independent copy of $Y$ such that $(X, \bar{Y}) \sim P_X \otimes P_Y$. If the function $f(X, \bar{Y})$ is $\sigma$-sub-gaussian with respect to the distribution $P_X \otimes P_Y$, then: 
    \begin{equation*}
        \left| \mathbb{E}[f(X, Y)] - \mathbb{E}[f(X, \bar{Y})] \right| \leq \sqrt{2\sigma^2 I(X; Y)}. 
    \end{equation*}
\end{lemma}
\begin{lemma}[Kantorovich-Rubinstein Duality]
\label{lem:KR}
    Let $P$ and $Q$ be probability measures defined on a metric space $(\mathcal{X}, c)$. Then, the $1$-Wasserstein distance is given by
    \begin{equation*}
        \mathcal{W}_1(P, Q) = \sup_{f \in \mathrm{Lip}_1} \left\{ \int_{\mathcal{X}} f \, \mathrm{d}P - \int_{\mathcal{X}} f \, \mathrm{d}Q \right\},
    \end{equation*}
    where $\mathrm{Lip}_1$ denotes the set of 1-Lipschitz functions with respect to the metric $c$, i.e., $|f(x) - f(x')| \le c(x, x')$ for any $f \in \mathrm{Lip}_1$ and $x, x' \in \mathcal{X}$.
\end{lemma}
\begin{lemma}[Lemma A.11 in \cite{dong2024towards}]
\label{lem:Gassian}
    Let $X \sim \mathcal{N}(0,\Sigma)$ and let $Y$ be any zero-mean random vector with $\mathrm{Cov}[Y]=\Sigma$. Then $h(Y)\leq h(X)$.
\end{lemma}
\begin{lemma}[Entropy upper bound under second-moment constraint]
\label{lem:Gassian_second_moment_entropy}
Let $X\in\mathbb R^d$ be any random vector with finite second moment matrix $M:=\mathbb E[XX^\top]\in\mathbb R^{d\times d}$.
Then the differential entropy satisfies
\[
h(X)\le \frac12\log\Big((2\pi e)^d\det(M)\Big).
\]
\end{lemma}
\begin{proof}
Let $\mu:=\mathbb E[X]$ and $\widetilde X:=X-\mu$. Since entropy is translation-invariant, $h(X)=h(\widetilde X)$.
Let $\Sigma:=\mathrm{Cov}(X)=\mathbb E[\widetilde X\widetilde X^\top]$.
By Lemma~\ref{lem:Gassian}, we have
\[
h(\widetilde X)\le \frac12\log\Big((2\pi e)^d\det(\Sigma)\Big).
\]
Next note that
\[
M-\Sigma=\mathbb E[XX^\top]-\mathbb E[\widetilde X\widetilde X^\top]=\mu\mu^\top\succeq 0,
\]
hence $M\succeq \Sigma\succeq 0$. Since determinant is monotone under the Loewner order on PSD matrices,
$\det(M)\ge \det(\Sigma)$, so
\[
h(X)=h(\widetilde X)\le \frac12\log\Big((2\pi e)^d\det(\Sigma)\Big)
\le \frac12\log\Big((2\pi e)^d\det(M)\Big).
\]
\end{proof}
\section{Omitted Proofs}
\subsection{Proof of Theorem~\ref{thm:synergy_drift}}
\begin{restatetheorem}{\ref{thm:synergy_drift}}[Restate]
Let $W$ be the output of a learner minimizing the empirical risk $\hat{\mathcal L}(W)$.
Assume the loss function $\ell(W,Z)$ is $\sigma$-subgaussian for all $Z\in\mathcal{Z}$ and $W\in\mathcal{W}$. For any split layer $l \in \{0, \dots, L\}$, we have
\[
|\mathrm{gen}_W|
\le
\underbrace{(T-1)\sqrt{2\sigma^2\,\mathcal K^{(l)}}}_{\text{\bf Replay-centroid Drift}}
+
\underbrace{
\sqrt{\frac{2\sigma^2}{N_{\mathrm{eff}}}
\Big(\mathcal S^{(l)}+\mathcal P^{(l)}-\mathcal R^{(l)}+\mathcal C^{(l)}\Big)}
}_{\text{\bf Optimization Variance}},
\]
where $N_{\mathrm{eff}} \triangleq \frac{1}{(T-1)/m + 1/n}$ is the effective sample size and $\mathcal K^{(l)}\triangleq \frac{1}{T-1}\sum_{i=1}^{T-1}\mathbb E_{W_{1:l}}\left[D_{\mathrm{KL}}\Big( Q_{A_l,Y|i,W_{1:l}}\big\|P_{A_l,Y|i,W_{1:l}}\Big)\right]$, and $\mathcal C^{(l)} \triangleq \mathbb E_{W_{1:l}}\left[ D_{\mathrm{KL}}\Big( P_{\mathbf U^{(l)}\mid W_{1:l}} \big\| \widetilde{Q}_{\mathbf{U}^{(l)}|W_{1:l}} \Big) \right]$.
\end{restatetheorem}
\begin{proof}
Fix a layer $l\in\{0,\dots,L\}$.
Recall that $a_0(x)=x$ and, for $h\in\{1,\dots,L\}$, $a_h(x)=\phi_h(W_h a_{h-1}(x))$. We also adopt the boundary convention that $g_{W_{L+1:L}}$ is the identity map when $l=L$.
To make the layer-$l$ representation parametrization unambiguous, we explicitly decompose the network at layer $l$.
Define the top sub-network mapping from the layer-$l$ representation to logits as
\[
g_{W_{l+1:L}}:\mathbb{R}^{d_l}\to\mathbb{R}^K,\qquad
f_W(x)=g_{W_{l+1:L}}(a_l(x)).
\]
Accordingly, define the loss as a function of the layer-$l$ representation--label pair:
\[
F_W(t,y)\triangleq \ell\big(g_{W_{l+1:L}}(t),y\big),
\qquad\text{so that}\qquad
\ell(W,(X,Y))=F_W(a_l(X),Y).
\]
By definition,
\[
\mathrm{gen}_W=\mathbb{E}_{W,\,D^{1:T}}[\mathcal L(W)-\hat{\mathcal L}(W)].
\]
Expand $\mathcal L(W)$ and $\hat{\mathcal L}(W)$:
\[
\mathcal L(W)=\sum_{t=1}^T\mathbb E_{Z\sim\mathcal D_t}[\ell(W,Z)],
\qquad
\hat{\mathcal L}(W)=\sum_{i=1}^{T-1}\frac{1}{m}\sum_{j\in\mathcal I_i}\ell(W,Z_j^i)+\frac{1}{n}\sum_{j=1}^n\ell(W,Z_j^T).
\]
Therefore,
\begin{align}
\mathcal L(W)-\hat{\mathcal L}(W)
&=
\sum_{i=1}^{T-1}\left(\mathbb E_{Z\sim\mathcal D_i}[\ell(W,Z)]-\frac{1}{m}\sum_{j\in\mathcal I_i}\ell(W,Z_j^i)\right)
\,+\,
\left(\mathbb E_{Z\sim\mathcal D_T}[\ell(W,Z)]-\frac{1}{n}\sum_{j=1}^n\ell(W,Z_j^T)\right).
\label{eq:basic_old_new_split}
\end{align}
For each old task $i\in[T-1]$, define the layer-$l$ population risk and replay empirical risk
\[
L_i(W)\triangleq \mathbb E_{(A_l,Y)\sim P_{A_l,Y|i,W_{1:l}}}\!\big[F_W(A_l,Y)\big],
\qquad
\widehat L_i(W)\triangleq \frac{1}{m}\sum_{j\in\mathcal I_i}F_W(A_{j,l}^i,Y_j^i).
\]
Then the old-task block in \eqref{eq:basic_old_new_split} can be written as
\begin{align}
\Delta_{\mathrm{old}}
&\triangleq
\sum_{i=1}^{T-1}\left(\mathbb E_{Z\sim\mathcal D_i}[\ell(W,Z)]-\frac{1}{m}\sum_{j\in\mathcal I_i}\ell(W,Z_j^i)\right)\nonumber\\
&=
\sum_{i=1}^{T-1}\Big(L_i(W)-\widehat L_i(W)\Big).
\label{eq:old_block_taskwise_block}
\end{align}
For each old task $i\in[T-1]$, recall the task-wise replay proxy
$\hat{P}_{A_l,Y|S^i,W_{1:l}}$ and its conditional centroid law
\begin{equation}
\label{eq:conditional centroid}
Q_{A_l,Y|i,W_{1:l}}
\;\triangleq\;
\mathbb E\!\left[
\hat{P}_{A_l,Y|S^i,W_{1:l}}
\ \middle|\ W_{1:l}
\right].
\end{equation}
Add and subtract $\mathbb E_{(A_l,Y)\sim Q_{A_l,Y|i,W_{1:l}}}[F_W(A_l,Y)]$ inside each
$L_i(W)-\widehat L_i(W)$ in \eqref{eq:old_block_taskwise_block} to obtain
\begin{equation}
\Delta_{\mathrm{old}}
=
\underbrace{\sum_{i=1}^{T-1}\Big(
\mathbb E_{P_{A_l,Y|i,W_{1:l}}}[F_W]
-
\mathbb E_{Q_{A_l,Y|i,W_{1:l}}}[F_W]
\Big)}_{\triangleq\ \Delta_{\mathrm{drift}}^{(l)}}
+
\underbrace{\sum_{i=1}^{T-1}\Big(
\mathbb E_{Q_{A_l,Y|i,W_{1:l}}}[F_W]
-
\widehat L_i(W)
\Big)}_{\triangleq\ \Delta_{\mathrm{old,var}}^{(l)}}.
\label{eq:drift_def_block}
\end{equation}
For the new task block in \eqref{eq:basic_old_new_split}, define
\[
\Delta_{\mathrm{new,var}}^{(l)}
\triangleq
\left(\mathbb E_{Z\sim\mathcal D_T}[\ell(W,Z)]-\frac{1}{n}\sum_{j=1}^n\ell(W,Z_j^T)\right)
=
\left(
\mathbb E_{P_{A_l,Y|T,W_{1:l}}}[F_W(A_l,Y)]
-\frac{1}{n}\sum_{j=1}^n F_W(A_{j,l}^T,Y_j^T)
\right),
\]
where $P_{A_l,Y|T,W_{1:l}}$ is the law of $(a_l(X),Y)$ under $Z=(X,Y)\sim\mathcal D_T$. Combining \eqref{eq:basic_old_new_split} and \eqref{eq:drift_def_block} yields the exact decomposition
\begin{equation}
\mathcal L(W)-\hat{\mathcal L}(W)
=
\Delta_{\mathrm{drift}}^{(l)}+\Delta_{\mathrm{old,var}}^{(l)}+\Delta_{\mathrm{new,var}}^{(l)}.
\label{eq:exact_three_term_decomp}
\end{equation}
Let $\Delta_{\mathrm{var}}^{(l)}\triangleq \Delta_{\mathrm{old,var}}^{(l)}+\Delta_{\mathrm{new,var}}^{(l)}$.
Taking expectation and using $|a+b|\le |a|+|b|$ gives
\begin{equation}
|\mathrm{gen}_W|
=
\left|\mathbb E[\mathcal L(W)-\hat{\mathcal L}(W)]\right|
\le
\mathbb E\big[|\Delta_{\mathrm{drift}}^{(l)}|\big]+|\mathbb E\big[\Delta_{\mathrm{var}}^{(l)}\big]|.
\label{eq:gen_drift_var_bound}
\end{equation}
Fix $i\in[T-1]$ and condition on $W_{1:l}$. Let
$Q_i \triangleq P_{A_l,Y|i,W_{1:l}}$ and $\tilde Q_i \triangleq Q_{A_l,Y|i,W_{1:l}}$.
If the loss function $\ell(W,(X, Y))$ is $\sigma$-sub-Gaussian for all $W \in \mathcal{W}$ and $Z\in\mathcal{Z}$, it follows that for any fixed $W$,
the random variable $F_W(A_l,Y)$ is $\sigma$-subgaussian under $Q_i$ (and also under $\tilde Q_i$). By Lemma~\ref{lem:DV}, we can obtain
\[
\left|
\mathbb E_{Q_i}[F_W(A_l,Y)]-\mathbb E_{\tilde Q_i}[F_W(A_l,Y)]
\right|
\le
\sqrt{2\sigma^2\,D_{\mathrm{KL}}(\tilde Q_i\|Q_i)}.
\]
Substitute into $\Delta_{\mathrm{drift}}^{(l)}$
\[
|\Delta_{\mathrm{drift}}^{(l)}|
\le
\sum_{i=1}^{T-1}\sqrt{2\sigma^2\,D_{\mathrm{KL}}\!\big(Q_{A_l,Y|i,W_{1:l}}\ \big\|\ P_{A_l,Y|i,W_{1:l}}\big)}.
\]
Take the expectation and use Jensen and  Cauchy--Schwarz's inequality,
\begin{align}
\mathbb E\!\big[|\Delta_{\mathrm{drift}}^{(l)}|\big]
&\le
\sqrt{(T-1)\sum_{i=1}^{T-1}
\mathbb E\!\left[2\sigma^2\,D_{\mathrm{KL}}\!\big(Q_{A_l,Y|i,W_{1:l}}\ \big\|\ P_{A_l,Y|i,W_{1:l}}\big)\right]}
\nonumber\\
&=
(T-1)\sqrt{2\sigma^2\cdot
\frac{1}{T-1}\sum_{i=1}^{T-1}
\mathbb E\!\left[D_{\mathrm{KL}}\!\big(Q_{A_l,Y|i,W_{1:l}}\ \big\|\ P_{A_l,Y|i,W_{1:l}}\big)\right]}.
\label{eq:drift_bound_before_convexity_taskwise_block}
\end{align}
For each fixed $W_{1:l}$ and each $i$,
take expectation over $W_{1:l}$ and define
\[
\mathcal K^{(l)}
\triangleq
\frac{1}{T-1}\sum_{i=1}^{T-1}
\mathbb E\!\left[
D_{\mathrm{KL}}\!\big(Q_{A_l,Y|i,W_{1:l}}\ \big\|\ P_{A_l,Y|i,W_{1:l}}\big)
\right].
\]
Combining with \eqref{eq:drift_bound_before_convexity_taskwise_block} that
\begin{equation}
\mathbb E\!\big[|\Delta_{\mathrm{drift}}^{(l)}|\big]
\le
(T-1)\,\sqrt{2\sigma^2\,\mathcal K^{(l)}}.
\label{eq:drift_bound_final}
\end{equation}
Next, we bound $|\mathbb E[\Delta_{\mathrm{var}}^{(l)}]|$.
Recall the weights
\[
a\triangleq \frac{1}{m}\quad\text{(each old replay sample)},\qquad
b\triangleq \frac{1}{n}\quad\text{(each new-task sample)}.
\]
Let $N = (T-1)m+n$. Enumerate all representation--label pairs used in the
variance term as a single sequence
\[
U_1^{(l)},\ldots,U_N^{(l)},
\quad\text{where}\quad
\{U_k^{(l)}\}_{k=1}^{(T-1)m}=\mathbf U_{\mathrm{old}}^{(l)},\ \ 
\{U_k^{(l)}\}_{k=(T-1)m+1}^{N}=\mathbf U_{\mathrm{new}}^{(l)}.
\]
Define the corresponding weights
\[
w_k \triangleq
\begin{cases}
a, & 1\le k\le (T-1)m,\\
b, & (T-1)m<k\le N.
\end{cases}
\]
For each index $k$, define its centering marginal (given $W_{1:l}$) as
\[
Q_k(\cdot)\triangleq P_{U^{(l)}\mid \mathsf{src}(k),\,W_{1:l}}
=
\begin{cases}
Q_{A_l,Y|i,W_{1:l}}, & \mathsf{src}(k)=i\in[T-1],\\
P_{A_l,Y|T,W_{1:l}}, & \mathsf{src}(k)=\mathrm{new}.
\end{cases}
\]
where $\mathsf{src}(k)\in\{1,\dots,T-1,\mathrm{new}\}$ indicates the task source of $U_k^{(l)}$:
for $k\le (T-1)m$, $\mathsf{src}(k)=i$ if $U_k^{(l)}$ is a replayed pair from old task $i$, and for
$k>(T-1)m$, $\mathsf{src}(k)=\mathrm{new}$.

Define the random centering functional $\bar{F}_k(W)\triangleq\mathbb{E}_{U\sim Q_k}[F_W(U)],$ which is a random variable because it still depends on the top parameters $W_{l+1:L}$ through $F_W$.
With this notation, the variance term can be written as 
 \begin{equation}
 \Delta_{\mathrm{var}}^{(l)}
 = 
 \sum_{k=1}^{N} w_k\Big( 
 \bar{F}_k(W)
 - 
 F_W(U_k^{(l)}) \Big).
 \label{eq:var_compact_GHOST} 
 \end{equation} 

Introduce the filtration $\mathcal F_0\subset\mathcal F_1\subset\cdots\subset\mathcal F_N$ with
\[
\mathcal F_{k}\triangleq \sigma\!\big(W_{1:l},U_1^{(l)},\ldots,U_{k}^{(l)}\big),\qquad k=0,1,\ldots,N,
\]
(where $\mathcal F_0=\sigma(W_{1:l})$).

Fix any $k\in[N]$ and condition on $\mathcal F_{k-1}$.
Define the conditional joint law
\[
\mu_k \triangleq P_{W_{l+1:L},\,U_k^{(l)}\mid \mathcal F_{k-1}}.
\]
Define the corresponding reference conditional law
\begin{equation}
\nu_k^\star
\triangleq
P_{W_{l+1:L}\mid \mathcal F_{k-1}}\ \otimes\ Q_k .
\label{eq:nu_star_def}
\end{equation}
Taking expectation of \eqref{eq:var_compact_GHOST} and using the tower property and the definition of $\Delta_{\mathrm{var}}^{(l)}$
, we have
\begin{align}
\mathbb E\big[\Delta_{\mathrm{var}}^{(l)}\big]
&=
\sum_{k=1}^N w_k\,
\mathbb E\!\left[
\mathbb E\!\big[\bar{F}_k(W)\mid \mathcal F_{k-1}\big]
-
\mathbb E\!\big[F_W(U_k^{(l)})\mid \mathcal F_{k-1}\big]
\right],
\label{eq:var_bridge_exact_nustar}
\end{align}
Now we identify the two conditional expectations in \eqref{eq:var_bridge_exact_nustar} with expectations under $\nu_k^\star$ and $\mu_k$.
By definition $\bar{F}_k(W)=\mathbb{E}_{U\sim Q_k}[F_W(U)]$, $\mathcal{F}_{k-1}$ averages over the remaining randomness in $W_{l+1:L}$. Using $\nu_k^\star$ in \eqref{eq:nu_star_def},
\begin{equation}
\mathbb{E}
\left[
\bar{F}_k(W)\mid\mathcal{F}_{k-1}
\right]
=
\mathbb{E}_{\nu_k^\star}
\left[
F_W(U)\mid\mathcal{F}_{k-1}
\right]
\label{eq:thm1:bridge1}
,\end{equation}
and explicitly
\[
\mathbb{E}_{\nu_k^\star}[F_W(U)\mid\mathcal{F}_{k-1}]
=
\int\left(\int F_{(W_{1:l},w)}(u)Q_k(du)\right)P_{W_{l+1:L}|\mathcal{F}_{k-1}}(dw),\]
which is indeed $\mathcal{F}_{k-1}$-measurable.
By definition of $\mu_k$ in \eqref{eq:nu_star_def},
\begin{equation}
\mathbb{E}{\left[F_W(U_k^{(l)})\mid\mathcal{F}_{k-1}\right]}
=
\mathbb{E}_{\mu_k}{\left[F_W(U_k^{(l)})\mid\mathcal{F}_{k-1}\right]}.
\label{eq:thm1:bridge2}
\end{equation}
Substituting \eqref{eq:thm1:bridge1} and \eqref{eq:thm1:bridge2} into \eqref{eq:var_bridge_exact_nustar} gives the exact bridge identity,
\begin{equation}
\mathbb{E}[\Delta_{\mathrm{var}}^{(l)}]
=
\sum_{k=1}^N
w_k\mathbb{E}
\left[
\mathbb{E}_{\nu_k^\star}
\left[
F_W(U)\mid\mathcal{F}_{k-1}
\right]
-
\mathbb{E}_{\mu_k}
\left[
F_W(U_k^{(l)})\mid\mathcal{F}_{k-1}
\right]
\right].
\end{equation}
Applying $|\,\mathbb E[X]\,|\le \mathbb E[|X|]$ and the triangle inequality gives 
\begin{equation}
\Big|\mathbb E\big[\Delta_{\mathrm{var}}^{(l)}\big]\Big|
\le
\sum_{k=1}^N w_k\,
\mathbb E\!\left[
\Big|
\mathbb E_{\nu_k^\star}\!\big[F_W(U)\mid \mathcal F_{k-1}\big]
-
\mathbb E_{\mu_k}\!\big[F_W(U_k^{(l)})\mid \mathcal F_{k-1}\big]
\Big|
\right].
\label{eq:var_sum_step_nustar}
\end{equation}
Under the assumption that the loss is $\sigma$-subgaussian,
for any fixed $W$ the random variable $F_W(U)$ is $\sigma$-subgaussian under $Q_k$.
Hence, conditioning on $\mathcal F_{k-1}$, the same subgaussian proxy holds under $\nu^\star_k$, because it is a mixture over $W_{l+1:L}$ with the same $\sigma$. Therefore Lemma~\ref{lem:DV} applied conditionally on $\mathcal{F}_{k-1}$ yields
\begin{equation}
\Big|
\mathbb E_{\nu_k^\star}\!\big[F_W(U)\mid \mathcal F_{k-1}\big]
-
\mathbb E_{\mu_k}\!\big[F_W(U_k^{(l)})\mid \mathcal F_{k-1}\big]
\Big|
\le
\sqrt{2\sigma^2\,D_{\mathrm{KL}}(\mu_k\|\nu_k^\star)}.
\label{eq:single_step_dv_nustar}
\end{equation}
Let $\nu_k \triangleq P_{W_{l+1:L}\mid \mathcal F_{k-1}}\otimes P_{U_k^{(l)}\mid \mathcal F_{k-1}}$
be the product of the true marginals under $\mathcal{F}_{k-1}$. Then by the KL chain rule,
\begin{align}
D_{\mathrm{KL}}(\mu_k\|\nu_k^\star)
&=
D_{\mathrm{KL}}(\mu_k\|\nu_k)
+
\mathbb E_{\mu_k}\!\left[
\log\frac{d\nu_k}{d\nu_k^\star}(W_{l+1:L},U_k^{(l)})
\right]
\nonumber\\
&=
I\!\big(U_k^{(l)};\,W_{l+1:L}\mid \mathcal F_{k-1}\big)
+
D_{\mathrm{KL}}\!\Big(
P_{U_k^{(l)}\mid \mathcal F_{k-1}}
\ \big\|\
Q_k
\Big).
\label{eq:KL_MI_plus_mismatch}
\end{align}

Define
\[
I_k \triangleq I\!\big(U_k^{(l)};\,W_{l+1:L}\mid \mathcal F_{k-1}\big),
\qquad
M_k \triangleq D_{\mathrm{KL}}\!\Big(
P_{U_k^{(l)}\mid \mathcal F_{k-1}}
\ \big\|\
Q_k
\Big).
\]
Then by \eqref{eq:var_sum_step_nustar}--\eqref{eq:KL_MI_plus_mismatch},
\begin{equation}
\Big|
\mathbb E_{\nu_k^\star}\!\big[F_W(U)\mid \mathcal F_{k-1}\big]
-
\mathbb E_{\mu_k}\!\big[F_W(U_k^{(l)})\mid \mathcal F_{k-1}\big]
\Big|
\le
\sqrt{2\sigma^2\,(I_k+M_k)}.
\label{eq:single_step_split_MI_mismatch}
\end{equation}

Substituting \eqref{eq:single_step_split_MI_mismatch} into \eqref{eq:var_sum_step_nustar} yields
\begin{equation}
\Big|\mathbb E\big[\Delta_{\mathrm{var}}^{(l)}\big]\Big|
\le
\sqrt{2\sigma^2}\,
\mathbb E\!\left[
\sum_{k=1}^N w_k\sqrt{I_k+M_k}
\right].
\label{eq:var_two_term_MI_mismatch}
\end{equation}

Applying Cauchy--Schwarz to the inner weighted sum gives
\begin{equation}
\sum_{k=1}^N w_k\sqrt{I_k+M_k}
\le
\sqrt{\sum_{k=1}^N w_k^2}\cdot \sqrt{\sum_{k=1}^N (I_k+M_k)}.
\label{eq:cs_aggregate_no_split}
\end{equation}
Using Jensen's inequality $\mathbb E[\sqrt{X}]\le \sqrt{\mathbb E[X]}$ and noting,
\[
\sum_{k=1}^Nw_k^2=(T-1)m{\left(\frac{1}{m}\right)}^2+n{\left(\frac{1}{n}\right)}^2={\left(\frac{T-1}{m}+\frac{1}{n}\right)}=\frac{1}{N_{\mathrm{eff}}},
\]
we obtain
\begin{equation}
\Big|\mathbb E\big[\Delta_{\mathrm{var}}^{(l)}\big]\Big|
\le
\sqrt{\frac{2\sigma^2}{N_{\mathrm{eff}}}}\,
\sqrt{
\mathbb E\!\left[\sum_{k=1}^N I_k\right]
+
\mathbb E\!\left[\sum_{k=1}^N M_k\right]
}.
\label{eq:var_synergy_part}
\end{equation}

By the chain rule for mutual information and Lemma~\ref{lem:interaction},
\begin{equation*}
\mathbb E\!\left[\sum_{k=1}^N I_k\right]
=
I(\mathbf U^{(l)};\,W_{l+1:L}\mid W_{1:l})
=
\mathcal S^{(l)}+\mathcal P^{(l)}-\mathcal R^{(l)}.
\end{equation*}
Moreover, by the chain rule for KL under the task-wise block product centering
$\widetilde{Q}_{\mathbf{U}^{(l)}|W_{1:l}}=\bigotimes_{k=1}^NQ_{k}$, we have
\begin{equation*}
\mathbb E\!\left[\sum_{k=1}^N M_k\right]
=
\mathbb E\!\left[
D_{\mathrm{KL}}\!\Big(
P_{\mathbf U^{(l)}\mid W_{1:l}}
\ \big\|\ 
\widetilde{Q}_{\mathbf{U}^{(l)}|W_{1:l}}
\Big)
\right]
=
\mathcal C^{(l)}.
\end{equation*}
Combining \eqref{eq:var_synergy_part}, we can get
\begin{equation}
\Big|\mathbb E\big[\Delta_{\mathrm{var}}^{(l)}\big]\Big|
\le
\sqrt{\frac{2\sigma^2}{N_{\mathrm{eff}}}
\Big(\mathcal S^{(l)}+\mathcal P^{(l)}-\mathcal R^{(l)}+\mathcal C^{(l)}\Big)}.
\label{eq:var_bound_with_mismatch}
\end{equation}
Finally, substituting \eqref{eq:drift_bound_final} and \eqref{eq:var_bound_with_mismatch} into
\eqref{eq:gen_drift_var_bound} completes the proof.
\end{proof}

\subsection{Proof of Corollary~\ref{cor:input_layer}}
\begin{restatecorollary}{\ref{cor:input_layer}}[Restate]
At the input layer $l=0$, the hierarchical generalization bound simplifies to:
    \begin{align*}
    |\mathrm{gen}_W|
    \;\le\;
    \sqrt{\frac{2\sigma^2}{N_{\mathrm{eff}}}\,I\!\big(S;W\big)},
    \end{align*}
where $S$ is the original $(X,Y)$ pair in the replay buffer and the current task dataset.
\end{restatecorollary}
\begin{proof}
The proof adopts the same framework as Theorem~\ref{thm:synergy_drift}. At $l=0$, the representation $A_0=X$ and $W_{1:0}$ is empty. We use the same notation $F_W(Z)=\ell(W,Z)$ for $Z=(X,Y)\in\mathcal Z$.
By definition,
\[
\mathrm{gen}_W=\mathbb E_{W,D^{1:T}}\big[\mathcal L(W)-\hat{\mathcal L}(W)\big],
\]
where $\mathcal L$ and $\hat{\mathcal L}$ are population and empirical risks.
Expanding $\mathcal L(W)$ and $\hat{\mathcal L}(W)$ as in~\eqref{eq:basic_old_new_split} yields the exact old and new split
\begin{align}
\mathcal L(W)-\hat{\mathcal L}(W)
&=
\sum_{i=1}^{T-1}\left(\mathbb E_{Z\sim\mathcal D_i}[F_W(Z)]-\frac{1}{m}\sum_{j\in\mathcal I_i}F_W(Z_j^i)\right)
+\left(\mathbb E_{Z\sim\mathcal D_T}[F_W(Z)]-\frac{1}{n}\sum_{j=1}^nF_W(Z_j^T)\right).
\label{eq:input_old_new_split}
\end{align}
Recall $P_{A_l,Y|i,W_{1:l}}$ denote the conditional law of the
representation--label pair $(A_l,Y)$ induced by $(X,Y)\sim\mathcal D_i$ under the current
representation parameters $W_{1:l}$. Therefore, when $l=0$, for each old task $i\in[T-1]$, the input-layer population law is $P_{X,Y|i} \equiv \mathcal D_i$. Similarly, the replay proxy $\hat{P}_{A_0,Y|S^i}=P_{S^i}=\frac{1}{m}\sum_{j\in\mathcal I_i}\delta_{Z_j^i}(\cdot)$, which is uniform on the buffer subset. Define its unconditional centroid law at $l=0$ by
\begin{equation}
\label{eq:input_centroid_def}
Q_{A_0,Y|i}\triangleq \mathbb E\big[\hat{P}_{S^i}\big].
\end{equation}
Using the same add-and-subtract step as~\eqref{eq:drift_def_block} (with $W_{1:0}$ void),
we obtain the exact decomposition
\begin{equation}
\label{eq:input_exact_decomp}
\mathcal L(W)-\hat{\mathcal L}(W)
=
\Delta_{\mathrm{drift}}^{(0)}+\Delta_{\mathrm{var}}^{(0)},
\end{equation}
where
\begin{align}
\Delta_{\mathrm{drift}}^{(0)}
&\triangleq
\sum_{i=1}^{T-1}\Big(
\mathbb E_{Z\sim \mathcal D_i}[F_W(Z)]-\mathbb E_{Z\sim Q_{A_0,Y|i}}[F_W(Z)]
\Big),
\label{eq:input_drift_def}\\
\Delta_{\mathrm{var}}^{(0)}
&\triangleq
\sum_{i=1}^{T-1}\Big(
\mathbb E_{Z\sim Q_{A_0,Y|i}}[F_W(Z)]-\frac{1}{m}\sum_{j\in\mathcal I_i}F_W(Z_j^i)
\Big)
+
\Big(
\mathbb E_{Z\sim \mathcal D_T}[F_W(Z)]-\frac{1}{n}\sum_{j=1}^nF_W(Z_j^T)
\Big).
\label{eq:input_var_def}
\end{align}
At the input layer, the replay centroid coincides with the task distribution:
indeed, by symmetry of uniform subsampling from i.i.d.\ samples,
a uniformly chosen element from the random subset $M^i$ has marginal law $\mathcal D_i$,
hence $Q_{A_0,Y|i}=\mathcal D_i=P_{X,Y|i}$.
Therefore, $\Delta_{\mathrm{drift}}^{(0)}=0$ in~\eqref{eq:input_drift_def}, and
\begin{equation}
\label{eq:input_gen_reduction}
|\mathrm{gen}_W|
=
\big|\mathbb E[\mathcal L(W)-\hat{\mathcal L}(W)]\big|
=
\big|\mathbb E[\Delta_{\mathrm{var}}^{(0)}]\big|.
\end{equation}
Let $N=(T-1)m+n$. Enumerate all examples used in~\eqref{eq:input_var_def} as a single sequence
$U_1^{(0)},\dots,U_N^{(0)}\in\mathcal Z$, where the first $(T-1)m$ terms come from the replay buffers
and the last $n$ terms come from the current task dataset.
Define the weights
\[
a\triangleq \frac{1}{m},\qquad b\triangleq \frac{1}{n},\qquad
w_k\triangleq
\begin{cases}
a, & 1\le k\le (T-1)m,\\
b, & (T-1)m<k\le N,
\end{cases}
\]
and let $\mathsf{src}(k)\in\{1,\dots,T-1,\mathrm{new}\}$ denote the source task of $U_k^{(0)}$. For each $k$, let
\[
Q_k(\cdot)\triangleq P_{U^{(0)}\mid \mathsf{src}(k)}
=
\begin{cases}
\mathcal D_i, & \mathsf{src}(k)=i\in[T-1],\\
\mathcal D_T, & \mathsf{src}(k)=\mathrm{new},
\end{cases}
\]
Define the random centering functional $\bar{F}_k(W)\triangleq\mathbb{E}_{U\sim Q_k}[F_W(U)]$, with this notation, the variance term can be written as 
 \begin{equation}
 \Delta_{\mathrm{var}}^{(0)}
 = 
 \sum_{k=1}^{N} w_k\Big( 
 \bar{F}_k(W)
 - 
 F_W(U_k^{(0)}) \Big).
 \label{eq:cor1:var_compact} 
 \end{equation} 
Define the filtration
\[
\mathcal F_k\triangleq \sigma\!\big(U_1^{(0)},\dots,U_k^{(0)}\big),\qquad k=0,1,\dots,N.
\]
For each $k$, let the joint law be $\mu_k \triangleq P_{W,\,U_k^{(0)}\mid \mathcal F_{k-1}}$ and the corresponding reference conditional law $\nu_k^\star \triangleq P_{W\mid \mathcal F_{k-1}}\ \otimes\ Q_k$.
Taking expectation of \eqref{eq:cor1:var_compact} and using the tower property and the definition of $\Delta_{\mathrm{var}}^{(0)}$
, we have
\begin{align}
\mathbb E\big[\Delta_{\mathrm{var}}^{(0)}\big]
&=
\sum_{k=1}^N w_k\,
\mathbb E\!\left[
\mathbb E\!\big[\bar{F}_k(W)\mid \mathcal F_{k-1}\big]
-
\mathbb E\!\big[F_W(U_k^{(0)})\mid \mathcal F_{k-1}\big]
\right],
\label{eq:cor1:var_bridge_exact_nustar}
\end{align}
Now we identify the two conditional expectations in \eqref{eq:cor1:var_bridge_exact_nustar} with expectations under $\nu_k^\star$ and $\mu_k$.
By definition $\bar{F}_k(W)=\mathbb{E}_{U\sim Q_k}[F_W(U)]$, $\mathcal{F}_{k-1}$ averages over the remaining randomness in $W$. Using $\nu_k^\star$,
\begin{equation}
\mathbb{E}
\left[
\bar{F}_k(W)\mid\mathcal{F}_{k-1}
\right]
=
\mathbb{E}_{\nu_k^\star}
\left[
F_W(U)\mid\mathcal{F}_{k-1}
\right]
\label{eq:cor1:bridge1}
,\end{equation}
and explicitly
\[
\mathbb{E}_{\nu_k^\star}[F_W(U)\mid\mathcal{F}_{k-1}]
=
\int\left(\int F_w(u)\,Q_k(du)\right)P_{W|\mathcal{F}_{k-1}}(dw),\]
which is indeed $\mathcal{F}_{k-1}$-measurable.
By definition of $\mu_k$,
\begin{equation}
\mathbb{E}{\left[F_W(U_k^{(0)})\mid\mathcal{F}_{k-1}\right]}
=
\mathbb{E}_{\mu_k}{\left[F_W(U_k^{(0)})\mid\mathcal{F}_{k-1}\right]}.
\label{eq:cor1:bridge2}
\end{equation}
Substituting \eqref{eq:cor1:bridge1} and \eqref{eq:cor1:bridge2} into \eqref{eq:cor1:var_bridge_exact_nustar} gives the exact bridge identity,
\begin{equation}
\mathbb{E}[\Delta_{\mathrm{var}}^{(0)}]
=
\sum_{k=1}^N
w_k\mathbb{E}
\left[
\mathbb{E}_{\nu_k^\star}
\left[
F_W(U)\mid\mathcal{F}_{k-1}
\right]
-
\mathbb{E}_{\mu_k}
\left[
F_W(U_k^{(0)})\mid\mathcal{F}_{k-1}
\right]
\right].
\end{equation}
Applying $|\,\mathbb E[X]\,|\le \mathbb E[|X|]$ and the triangle inequality gives 
\begin{equation}
\Big|\mathbb E\big[\Delta_{\mathrm{var}}^{(0)}\big]\Big|
\le
\sum_{k=1}^N w_k\,
\mathbb E\!\left[
\Big|
\mathbb E_{\nu_k^\star}\!\big[F_W(U)\mid \mathcal F_{k-1}\big]
-
\mathbb E_{\mu_k}\!\big[F_W(U_k^{(0)})\mid \mathcal F_{k-1}\big]
\Big|
\right].
\label{eq:cor1:var_sum_step_nustar}
\end{equation}
Under the assumption that the loss is $\sigma$-subgaussian,
for any fixed $W$ the random variable $F_W(U)$ is $\sigma$-subgaussian under $Q_k$.
Hence, conditioning on $\mathcal F_{k-1}$, the same subgaussian proxy holds under $\nu^\star_k$, because it is a mixture over $W$ with the same $\sigma$. Therefore Lemma~\ref{lem:DV} applied conditionally on $\mathcal{F}_{k-1}$ yields
\begin{equation}
\Big|
\mathbb E_{\nu_k^\star}\!\big[F_W(U)\mid \mathcal F_{k-1}\big]
-
\mathbb E_{\mu_k}\!\big[F_W(U_k^{(0)})\mid \mathcal F_{k-1}\big]
\Big|
\le
\sqrt{2\sigma^2\,D_{\mathrm{KL}}(\mu_k\|\nu_k^\star)}.
\label{eq:cor1:single_step_dv_nustar}
\end{equation}
Let $\nu_k \triangleq P_{W\mid \mathcal F_{k-1}}\otimes P_{U_k^{(0)}\mid \mathcal F_{k-1}}$
be the product of the true marginals under $\mathcal{F}_{k-1}$. Then by the KL chain rule,
\begin{equation}
D_{\mathrm{KL}}(\mu_k\|\nu_k^\star)
=
I\!\big(U_k^{(0)};\,W\mid \mathcal F_{k-1}\big)
+
D_{\mathrm{KL}}\!\Big(
P_{U_k^{(0)}\mid \mathcal F_{k-1}}
\ \big\|\
Q_k
\Big),
\label{eq:cor1:KL_MI_plus_mismatch}
\end{equation}

Define
\[
I_k \triangleq I\!\big(U_k^{(0)};\,W\mid \mathcal F_{k-1}\big),
\qquad
M_k \triangleq D_{\mathrm{KL}}\!\Big(
P_{U_k^{(0)}\mid \mathcal F_{k-1}}
\ \big\|\
Q_k
\Big).
\]
Substituting into \eqref{eq:cor1:var_sum_step_nustar}  and following exactly the
Cauchy--Schwarz + Jensen aggregation in~\eqref{eq:var_two_term_MI_mismatch}--\eqref{eq:var_synergy_part}
yields
\begin{equation}
\Big|\mathbb E\big[\Delta_{\mathrm{var}}^{(0)}\big]\Big|
\le
\sqrt{\frac{2\sigma^2}{N_{\mathrm{eff}}}}\,
\sqrt{
\mathbb E\!\left[\sum_{k=1}^N I_k\right]
+
\mathbb E\!\left[\sum_{k=1}^N M_k\right]
}.
\label{eq:cor1:var_synergy_part}
\end{equation}

By the mutual information chain rule,
\begin{equation*}
\mathbb E\!\left[\sum_{k=1}^N I_k\right]
=
I(S;W).
\end{equation*}
Moreover, by the chain rule for KL under the task-wise block product centering
$\widetilde{Q}_{\mathbf{U}^{(0)}}\triangleq\left(\bigotimes_{i=1}^{T-1}(\mathcal{D}_i)^{\otimes m}\right)\otimes(\mathcal{D}_T)^{\otimes n}$, we have
\[
\mathbb E\!\left[\sum_{k=1}^N M_k\right]
=
D_{\mathrm{KL}}\!\big(P_{S}\ \big\|\ \widetilde{Q}_{\mathbf{U}^{(0)}}\big).
\]
At $l=0$, the sequence $S$ is block-wise i.i.d.\ by construction (each task dataset is i.i.d.,
tasks are mutually independent, and each buffer contributes $m$ distinct coordinates from an i.i.d.\ task dataset);
hence $P_{S}=\widetilde{Q}_{\mathbf{U}^{(0)}}$ and the above KL term is zero.
Combining with~\eqref{eq:input_gen_reduction} and~\eqref{eq:cor1:var_synergy_part} yields
\[
|\mathrm{gen}_W|
\le
\sqrt{\frac{2\sigma^2}{N_{\mathrm{eff}}}\,I(S;W)}.
\]
\end{proof}

\subsection{Proof of Corollary~\ref{cor:output_layer}}
\begin{restatecorollary}{\ref{cor:output_layer}}[Restate]
At the output layer $l=L$, the hierarchical generalization bound simplifies to:
\begin{align*}
|\mathrm{gen}_W|
\;\le\;
(T-1)\sqrt{2\sigma^2\,\mathcal K^{(L)}}
\;+\;
\sqrt{\frac{2\sigma^2}{N_{\mathrm{eff}}}\,\mathcal C^{(L)}},
\end{align*}
where $\mathcal K^{(L)}=\frac{1}{T-1}\sum_{i=1}^{T-1}\mathbb E_W[D_{\mathrm{KL}}( Q_{A_L,Y|i,W}\|P_{A_L,Y|i,W})]$, $\mathcal C^{(L)} = \mathbb E_W[D_{\mathrm{KL}}( P_{\mathbf U^{(L)}\mid W} \| \widetilde{Q}_{\mathbf{U}^{(L)}|W})]$.
\end{restatecorollary}
\begin{proof}
We follow the proof structure of Theorem~\ref{thm:synergy_drift} and highlight the simplifications at $l=L$. At $l=L$, the representation is the network output $A_L=f_W(X)$. The top subnetwork above layer $L$ is empty, so we may regard $g_{W_{L+1:L}}$ as the identity map and write the loss as a function of the output--label pair:
\[
F_W(t,y)\triangleq \ell(t,y),
\qquad\text{so that}\qquad
\ell(W,(X,Y)) = F_W(A_L,Y).
\]
All layer-$L$ distributions below are understood as conditional laws given the full parameter $W=W_{1:L}$. As in~\eqref{eq:basic_old_new_split}, the population--empirical
gap admits the exact split into old and new tasks. Introducing the layer-$L$
population law $P_{A_L,Y|i,W_{1:L}}$ of $(A_L,Y)$ under $(X,Y)\sim\mathcal D_i$
and the replay proxy $\hat{P}_{A_L,Y|S^i,W_{1:L}}$ which is uniform over
$\{(A_{j,L}^i,Y_j^i):j\in\mathcal I_i\}$, recall the replay centroid
\begin{equation}
    \label{eq:output_centroid_def}
    Q_{A_L,Y|i,W_{1:L}}= \mathbb E\!\left[\hat{P}_{A_L,Y|S^i,W_{1:L}}\ \middle|\ W\right].
\end{equation}
Adding and subtracting $\mathbb E_{Q_{A_L,Y|i,W_{1:L}}}[F_W(A_L,Y)]$
inside each old-task gap yields an exact decomposition
\begin{equation}
\label{eq:output_exact_decomp}
\mathcal L(W)-\hat{\mathcal L}(W)
=
\Delta_{\mathrm{drift}}^{(L)}+\Delta_{\mathrm{var}}^{(L)},
\end{equation}
where
\begin{align}
\Delta_{\mathrm{drift}}^{(L)}
&\triangleq
\sum_{i=1}^{T-1}\Big(
\mathbb E_{Z\sim \mathcal D_i}[F_W(Z)]-\mathbb E_{Z\sim Q_{A_L,Y|i,W_{1:L}}}[F_W(Z)]
\Big),
\label{eq:output_drift_def}\\
\Delta_{\mathrm{var}}^{(L)}
&\triangleq
\sum_{i=1}^{T-1}\Big(
\mathbb E_{Z\sim Q_{A_L,Y|i,W_{1:L}}}[F_W(Z)]-\frac{1}{m}\sum_{j\in\mathcal I_i}F_W(Z_j^i)
\Big)
+
\Big(
\mathbb E_{Z\sim \mathcal D_T}[F_W(Z)]-\frac{1}{n}\sum_{j=1}^nF_W(Z_j^T)
\Big).
\label{eq:output_var_def}
\end{align}
By definition,
\begin{equation}
\label{eq:output_gen_reduction}
|\mathrm{gen}_W|
=
\left|\mathbb E[\mathcal L(W)-\hat{\mathcal L}(W)]\right|
\le
\mathbb E\big[|\Delta_{\mathrm{drift}}^{(L)}|\big]+|\mathbb E\big[\Delta_{\mathrm{var}}^{(L)}\big]|.
\end{equation}
Fix an old task $i\in[T-1]$ and condition on $W$. Since $F_W(A_L,Y)$ is $\sigma$-subgaussian under both
$P_{A_L,Y|i,W_{1:L}}$ and $Q_{A_L,Y|i,W_{1:L}}$, by Lemma~\ref{lem:DV}, we can obtain
\[
\Big|
\mathbb E_{P_{A_L,Y|i,W_{1:L}}}[F_W]
-
\mathbb E_{Q_{A_L,Y|i,W_{1:L}}}[F_W]
\Big|
\le
\sqrt{2\sigma^2\,D_{\mathrm{KL}}\!\Big(Q_{A_L,Y|i,W_{1:L}}\ \big\|\ P_{A_L,Y|i,W_{1:L}}\Big)}.
\]
Averaging over $i$ and applying Cauchy--Schwarz exactly as in the main proof gives
\begin{equation}
\label{eq:cor2:drift_bound_final}
\mathbb E\!\big[|\Delta_{\mathrm{drift}}^{(L)}|\big]
\le
(T-1)\sqrt{2\sigma^2\,\mathcal K^{(L)}},
\end{equation}
with $\mathcal K^{(L)}$ as stated.
Let $N=(T-1)m+n$ and enumerate the layer-$L$ training pairs as a single sequence
$U_1^{(L)},\dots,U_N^{(L)}$, with weights
\[
a\triangleq \frac{1}{m},\qquad b\triangleq \frac{1}{n},\qquad
w_k\triangleq
\begin{cases}
a, & 1\le k\le (T-1)m,\\
b, & (T-1)m<k\le N,
\end{cases}
\]
and let $\mathsf{src}(k)\in\{1,\dots,T-1,\mathrm{new}\}$ denote the source task of $U_k^{(L)}$. For each $k$, let
\[
Q_k(\cdot)\triangleq P_{U^{(L)}\mid \mathsf{src}(k),\,W}
=
\begin{cases}
Q_{A_L,Y|i,W_{1:L}}, & \mathsf{src}(k)=i\in[T-1],\\
P_{A_L,Y|T,W_{1:L}}, & \mathsf{src}(k)=\mathrm{new}.
\end{cases}
\]
With this notation, the variance term can be written as 
 \begin{equation}
 \Delta_{\mathrm{var}}^{(L)}
 = 
 \sum_{k=1}^{N} w_k\Big( 
 \mathbb E_{U\sim Q_k}[F_W(U)]
 - 
 F_W(U_k^{(L)}) \Big).
 \label{eq:cor2:var_compact} 
 \end{equation} 
Define the filtration ($\mathcal{F}_0=\sigma(W)$)
\[
\mathcal F_k\triangleq \sigma\!\big(W,U_1^{(L)},\dots,U_k^{(L)}\big),\qquad k=0,1,\dots,N.
\]
For each $k$, define the true conditional law of the $k$-th coordinate
\[
\mu_k \triangleq P_{U_k^{(L)}\mid \mathcal F_{k-1}}.
\]
Since at $l=L$ the ``top parameter'' $W_{L+1:L}$ is empty (a deterministic
constant), the reference law $\nu_k^\star$ reduces to the centering marginal:
\[
\nu_k^\star \triangleq Q_k.
\]
Taking expectation of \eqref{eq:cor2:var_compact} and using the tower property, we have
\begin{align}
\mathbb E\big[\Delta_{\mathrm{var}}^{(L)}\big]
&=
\sum_{k=1}^N w_k\,
\mathbb E\!\left[
\mathbb E\!\big[{F}_W(U)\mid \mathcal F_{k-1}\big]
-
\mathbb E\!\big[F_W(U_k^{(L)})\mid \mathcal F_{k-1}\big]
\right],
\label{eq:cor2:var_bridge_exact_nustar}
\end{align}
Since $W$ is $\mathcal F_{k-1}$-measurable, $\mathbb E\big[{F}_W(U)\mid \mathcal F_{k-1}\big]=\mathbb{E}_{\nu_k^\star} \left[ F_W(U) \right]$. By definition of $\mu_k$, $\mathbb{E}[F_{W}(U_{k}^{(L)})\mid\mathcal{F}_{k-1}]=\mathbb{E}_{\mu_k}[F_W(U_k^{(L)})\mid\mathcal{F}_{k-1}]$. Substituting into \eqref{eq:cor2:var_bridge_exact_nustar} gives the bridge identity,
\begin{equation}
\label{eq:cor2:bridge1}
\mathbb{E}[\Delta_{\mathrm{var}}^{(L)}]
=
\sum_{k=1}^N
w_k\mathbb{E}
\left[
\mathbb{E}_{\nu_k^\star}
\left[
F_W(U)
\right]
-
\mathbb{E}_{\mu_k}
\left[
F_W(U_k^{(L)})\mid\mathcal{F}_{k-1}
\right]
\right].
\end{equation}
Applying $|\,\mathbb E[X]\,|\le \mathbb E[|X|]$ and the triangle inequality gives 
\begin{equation}
\Big|\mathbb E\big[\Delta_{\mathrm{var}}^{(L)}\big]\Big|
\le
\sum_{k=1}^N w_k\,
\mathbb E\!\left[
\Big|
\mathbb E_{\nu_k^\star}\!\big[F_W(U)\big]
-
\mathbb E_{\mu_k}\!\big[F_W(U_k^{(L)})\mid \mathcal F_{k-1}\big]
\Big|
\right].
\label{eq:cor2:var_sum_step_nustar}
\end{equation}
By Lemma~\ref{lem:DV}, we can obtain
\[
\Big|
\mathbb E_{\nu_k^{\star}}[F_W(U)]
-
\mathbb E_{\mu_k}[F_W(U_k^{(L)})\mid \mathcal F_{k-1}]
\Big|
\le
\sqrt{2\sigma^2\,D_{\mathrm{KL}}(\mu_k\|Q_k)}.
\]
Aggregating over $k$ with Cauchy--Schwarz and Jensen gives
\[
\Big|\mathbb E[\Delta_{\mathrm{var}}^{(L)}]\Big|
\le
\sqrt{\frac{2\sigma^2}{N_{\mathrm{eff}}}}
\sqrt{
\mathbb E\!\left[\sum_{k=1}^N D_{\mathrm{KL}}\!\big(\mu_k\|Q_k\big)\right]
}.
\]
Finally, by the chain rule for KL under the product reference
$\widetilde{Q}_{\mathbf{U}^{(L)}|W}=\bigotimes_{k=1}^N Q_k$, we have
\[
\mathbb E\!\left[\sum_{k=1}^N D_{\mathrm{KL}}\!\big(\mu_k\|Q_k\big)\right]
=
\mathbb E\!\left[
D_{\mathrm{KL}}\!\Big(P_{\mathbf U^{(L)}\mid W}\ \big\|\ \widetilde{Q}_{\mathbf{U}^{(L)}|W}\Big)
\right]
=
\mathcal C^{(L)}.
\]
Hence
\[
\Big|\mathbb E[\Delta_{\mathrm{var}}^{(L)}]\Big|
\le
\sqrt{\frac{2\sigma^2}{N_{\mathrm{eff}}}\,\mathcal C^{(L)}}.
\]
Then $\mathcal S^{(L)},\mathcal P^{(L)},\mathcal R^{(L)}$
are mutual/interaction information involving $W_{L+1:L}$ conditioned on $W_{1:L}=W$.
Since $W_{L+1:L}$ is deterministic (empty), all these quantities are zero:
\[
\mathcal S^{(L)}=\mathcal P^{(L)}=\mathcal R^{(L)}=0.
\]

Combining the drift and variance bounds completes the proof.
\end{proof}

\subsection{Proof of Theorem~\ref{thm:wasserstein}}
\begin{restatetheorem}{\ref{thm:wasserstein}}[Restate]
Assume the loss function $\ell(\cdot, y)$ is $\rho_{0}$-Lipschitz and the activation functions $\phi_l$ are $\rho_{l}$-Lipschitz. Then, at time $T$, we have:
\begin{align*}
\big|\mathrm{gen}_W\big|
\le
\min_{l\in\{0,\dots,L\}}
\mathbb E\Bigg[
\bar\rho_l(W)\cdot\sum_{i=1}^{T} \mathcal W_1\!\Big(\hat{P}_{A_l,Y|S^i,W_{1:l}},\,P_{A_l,Y|i,W_{1:l}}\Big)
\Bigg],
\end{align*}
where $\bar\rho_l(W):=\rho_0\left(1\vee\prod_{h=l+1}^{L}\rho_h\|W_h\|_{\mathrm{op}}\right).$
\end{restatetheorem}
\begin{proof}
Fix an arbitrary $l\in\{0,1,\dots,L\}$. Define the upper sub-network mapping from layer $l$ to the output: $G_{W_{l+1:L}}(a):=g_{W_L}\circ g_{W_{L-1}}\circ\cdots\circ g_{W_{l+1}}(a)$, and define the re-parameterized loss on $(a,y)\in\mathcal A_l\times\mathcal Y$: $f_W^{(l)}(a,y):=\ell\big(G_{W_{l+1:L}}(a),\,y\big)$. We first bound the Lipschitz constant of $f_W^{(l)}$. First, for any layer $h\in\{l+1,\dots,L\}$ and any $u,u'$,
\begin{align*}
\|g_{W_h}(u)-g_{W_h}(u')\|_2
&=\|\phi_h(W_hu)-\phi_h(W_hu')\|_2\\
&\le \rho_h\|W_hu-W_hu'\|_2 \\
&\le \rho_h\|W_h\|_{\mathrm{op}}\,\|u-u'\|_2.
\end{align*}
Hence $g_{W_h}$ is $\rho_h\|W_h\|_{\mathrm{op}}$-Lipschitz. Next, define for integers $p\le q$ the partial composition $ G_{p:q}:=g_{W_q}\circ g_{W_{q-1}}\circ\cdots\circ g_{W_p}$. Then for any $u,u'$,
\begin{align*}
\|G_{p:q}(u)-G_{p:q}(u')\|_2
&=\|g_{W_q}(G_{p:q-1}(u))-g_{W_q}(G_{p:q-1}(u'))\|_2\\
&\le \mathrm{Lip}(g_{W_q})\cdot \|G_{p:q-1}(u)-G_{p:q-1}(u')\|_2\\
&\le \big(\rho_q\|W_q\|_{\mathrm{op}}\big)\cdot \mathrm{Lip}(G_{p:q-1})\cdot \|u-u'\|_2\\
&\le \left(\rho_q\|W_q\|_{\mathrm{op}}\right)\left(\prod_{h=p}^{q-1}\rho_h\|W_h\|_{\mathrm{op}}\right)\|u-u'\|_2,
\end{align*}
Taking $p=l+1$ and $q=L$ gives
\[
\|G_{W_{l+1:L}}(a)-G_{W_{l+1:L}}(a')\|_2
\le
\left(\prod_{h=l+1}^{L}\rho_h\|W_h\|_{\mathrm{op}}\right)\|a-a'\|_2.
\]
Denote $\alpha_l(W):=\prod_{h=l+1}^{L}\rho_h\|W_h\|_{\mathrm{op}}$ for short.

Using that $\ell(\hat y,y)$ is $\rho_0$-Lipschitz in $(\hat y,y)$ under the Euclidean norm, for any $(a,y),(a',y')$ we have
\begin{align*}
|f_{W,l}(a,y)-f_{W,l}(a',y')|
&=\big|\ell(G_{W_{l+1:L}}(a),y)-\ell(G_{W_{l+1:L}}(a'),y')\big|\\
&\le \rho_0\sqrt{\|G_{W_{l+1:L}}(a)-G_{W_{l+1:L}}(a')\|_2^2+\|y-y'\|_2^2}\\
&\le \rho_0\sqrt{\alpha_l(W)^2\|a-a'\|_2^2+\|y-y'\|_2^2}.
\end{align*}
If $\alpha_l(W)\le 1$, then $\alpha_l(W)^2\|a-a'\|_2^2+\|y-y'\|_2^2\le \|a-a'\|_2^2+\|y-y'\|_2^2$.
If $\alpha_l(W)\ge 1$, then $\alpha_l(W)^2\|a-a'\|_2^2+\|y-y'\|_2^2\le \alpha_l(W)^2(\|a-a'\|_2^2+\|y-y'\|_2^2)$.
In both cases,
\[
\sqrt{\alpha_l(W)^2\|a-a'\|_2^2+\|y-y'\|_2^2}
\le (1\vee \alpha_l(W))\sqrt{\|a-a'\|_2^2+\|y-y'\|_2^2}
=(1\vee \alpha_l(W))\,d\big((a,y),(a',y')\big).
\]
Therefore,
\[
|f_{W,l}(a,y)-f_{W,l}(a',y')|
\le \rho_0(1\vee \alpha_l(W))\,d\big((a,y),(a',y')\big)
= \bar\rho_l(W)\,d\big((a,y),(a',y')\big),
\]
which shows $\mathrm{Lip}(f_{W,l})\le \bar\rho_l(W)$ with $\bar\rho_l(W)$ defined 
\begin{equation}
\label{eq:thm2-rho-bar}
\bar\rho_l(W):=\rho_0\left(1\vee\prod_{h=l+1}^{L}\rho_h\|W_h\|_{\mathrm{op}}\right).
\end{equation}
For each task $i\in[T]$, by definition of $P_{A_l,Y|i,W_{1:l}}$ we have
\[
\mathbb E_{Z\sim\mathcal D_i}\ell(W,Z)
=\mathbb E_{(A_l,Y)\sim P_{A_l,Y|i,W_{1:l}}} f_{W,l}(A_l,Y).
\]
Moreover, by definition of $\hat{P}_{A_l,Y|S^i,W_{1:l}}$,
\[
\mathbb E_{U\sim \hat{P}_{A_l,Y|S^i,W_{1:l}}} f_{W,l}(U)
=
\begin{cases}
\frac1m\sum_{j\in\mathcal I_i} f_{W,l}(A_{j,l}^i,Y_j^i)=\frac1m\sum_{j\in\mathcal I_i}\ell(W,Z_j^i), & i\in[T-1],\\[4pt]
\frac1n\sum_{j=1}^{n} f_{W,l}(A_{j,l}^T,Y_j^T)=\frac1n\sum_{j=1}^{n}\ell(W,Z_j^T), & i=T.
\end{cases}
\]
Substituting these identities into population $\mathcal{L}(W)$ and empirical risks $\hat{\mathcal{L}}(W)$ yields, for the fixed realization,
\[
\mathcal L(W)-\hat{\mathcal L}(W)
=
\sum_{i=1}^{T}\left(
\mathbb E_{P_{A_l,Y|i,W_{1:l}}} \big[f_{W,l} \big]
-
\mathbb E_{\hat{P}_{A_l,Y|S^i,W_{1:l}}} \big[f_{W,l} \big]
\right).
\]
By Kantorovich--Rubinstein Lemma~\ref{lem:KR} for the $1$-Wasserstein distance, for any Lipschitz function $f$ on $(\mathcal A_l\times\mathcal Y,d)$,
\[
\big|\mathbb E_{\mu}f-\mathbb E_{\nu}f\big|\le \mathrm{Lip}(f)\,\mathcal W_1(\mu,\nu).
\]
Applying this with $f=f_{W,l}$, $\mu=P_{A_l,Y|i,W_{1:l}}$, and $\nu=\hat{P}_{A_l,Y|S^i,W_{1:l}}$, and using $\mathrm{Lip}(f_{W,l})\le \bar\rho_l(W)$ from the previous steps, we obtain for each $i\in[T]$:
\[
\left|
\mathbb E_{P_{A_l,Y|i,W_{1:l}}} f_{W,l}
-
\mathbb E_{\hat{P}_{A_l,Y|S^i,W_{1:l}}} f_{W,l}
\right|
\le
\bar\rho_l(W)\,
\mathcal W_1\!\Big(\hat{P}_{A_l,Y|S^i,W_{1:l}},\,P_{A_l,Y|i,W_{1:l}}\Big).
\]
Hence, by the triangle inequality,
\[
\big|\mathcal L(W)-\hat{\mathcal L}(W)\big|
\le
\bar\rho_l(W)\cdot \sum_{i=1}^{T}
\mathcal W_1\!\Big(\hat{P}_{A_l,Y|S^i,W_{1:l}},\,P_{A_l,Y|i,W_{1:l}}\Big).
\]
Now take expectation over all randomness, including the generally dependent training sequence law $P_{\mathbf U^{(l)}\mid W_{1:l}}$:
\begin{equation}
\label{eq:thm2-main}
\big|\mathrm{gen}_W\big|
=
\left|\mathbb E\big[\mathcal L(W)-\hat{\mathcal L}(W)\big]\right|
\le
\mathbb E\big|\mathcal L(W)-\hat{\mathcal L}(W)\big|
\le
\mathbb E\left[
\bar\rho_l(W)\cdot \sum_{i=1}^{T}
\mathcal W_1\!\Big(\hat{P}_{A_l,Y|S^i,W_{1:l}},\,P_{A_l,Y|i,W_{1:l}}\Big)
\right].
\end{equation}
Since the bound \eqref{eq:thm2-main} holds for every $l\in\{0,\dots,L\}$, taking the minimum over $l$ gives,
\begin{equation}
\label{eq:thm2-minl}
\big|\mathrm{gen}_W\big|
\;\le\;
\min_{l\in\{0,\dots,L\}}
\mathbb E\left[
\bar\rho_l(W)\cdot\sum_{i=1}^{T}
\mathcal W_1\!\Big(\hat{P}_{A_l,Y|S^i,W_{1:l}},\,P_{A_l,Y|i,W_{1:l}}\Big)
\right].
\end{equation}
\end{proof}
\subsection{Proof of Proposition~\ref{prop:cross_task_recursion}}
\begin{restateproposition}{\ref{prop:cross_task_recursion}}[Restate]
Fix a layer $l$ and condition on $W_{1:l}$. At task $T$, we have
\begin{align*}
I(\mathbf U_T^{(l)};\Theta_R^{(T)}| W_{1:l})
&\le I(\mathbf U_T^{(l)};\Theta_0^{(T)}| W_{1:l})
   + I(\mathbf U_T^{(l)};\Theta_R^{(T)}| \Theta_0^{(T)},W_{1:l}) \nonumber\\
&=: H_T^{(l)}+\Delta_T^{(l)}.
\end{align*}
\end{restateproposition}
\begin{proof}
Condition on $W_{1:l}$ and suppress it in the notation.
By data processing, $I(\mathbf U_T^{(l)};\Theta^{(T)}_R)\le I(\mathbf U_T^{(l)};(\Theta^{(T)}_0,\Theta^{(T)}_R))$.
Applying the chain rule for mutual information yields
\[
I(\mathbf U_T^{(l)};(\Theta^{(T)}_0,\Theta^{(T)}_R))
~=~
I(\mathbf U_T^{(l)};\Theta^{(T)}_0)
+I(\mathbf U_T^{(l)};\Theta^{(T)}_R\mid \Theta^{(T)}_0),
\]
which proves the Proposition.
\end{proof}
\subsection{Proof of Corollary~\ref{cor:Cumulative information budget}}
\begin{restatecorollary}{\ref{cor:Cumulative information budget}}[Restate]
Fix a layer $l$ and condition on $W_{1:l}$. For any $T\ge 2$, the heritage term admits the following cumulative-budget upper bound:
\[
H_T^{(l)} = I(\mathbf U_T^{(l)};\Theta_0^{(T)}\mid W_{1:l})
\le \sum_{t=1}^{T-1} \Delta_t^{(l)},
\]
where $\Delta_t^{(l)} := I(\mathbf U_t^{(l)};\Theta_R^{(t)}\mid \Theta_0^{(t)},W_{1:l})$ is the within-task information increment
in Proposition~\ref{prop:cross_task_recursion} applied to task $t$.
\end{restatecorollary}
\begin{proof}
Throughout the proof we implicitly condition on $W_{1:l}$ to simplify notation. By the recursion $\Theta_0^{(T)}=\Theta_R^{(T-1)}$, we have $ H_T^{(l)} = I(\mathbf U_T^{(l)};\Theta_R^{(T-1)}).$ The task-$t$ training sequence $\mathbf U_t^{(l)}$ is generated by sampling from
(i) the current-task dataset and (ii) a replay memory $\mathcal M^{1:t-1}$ built from the past,
and this sampling procedure does not depend on $\Theta_R^{(t-1)}$ once $\mathcal M^{1:t-1}$ is fixed. Then,
\[
\mathbf U_t^{(l)} \perp \Theta_R^{(t-1)} \mid \mathcal M^{1:t-1}.
\]
Therefore, at task $T$, we have $\mathbf U_T^{(l)} \perp \Theta_R^{(T-1)} \mid \mathcal M^{1:T-1}$, which implies the Markov chain
\[
\mathbf U_T^{(l)} \to \mathcal M^{1:T-1}\to\Theta_R^{(T-1)}.
\]
By the data processing inequality,
\begin{equation}
\label{cor:eq:HT_to_memory}
H_T^{(l)}=I(\mathbf U_T^{(l)};\Theta_R^{(T-1)}) \;\le\; I(\mathcal M^{1:T-1};\Theta_R^{(T-1)}).
\end{equation}
Moreover, since $\mathcal M^{1:T-1}$ is measurable with respect to $\mathbf U_{1:T-1}^{(l)}$, we have Markov chains again,
\[\Theta_R^{(T-1)}\to\mathbf{U}_{1:T-1}^{(l)}\to \mathcal M^{1:T-1}.\]
By the data processing inequality,
\begin{equation}
\label{cor:eq:memory_to_pastU}
I(\mathcal M^{1:T-1};\Theta_R^{(T-1)}) \;\le\; I(\mathbf U_{1:T-1}^{(l)};\Theta_R^{(T-1)}).
\end{equation}

For each $t\ge 1$, define $ J_t := I(\mathbf U_{1:t}^{(l)};\Theta_R^{(t)}).$ We claim that for every $t\ge 1$,
\begin{equation}
\label{cor:eq:J_recursion}
J_t \;\le\; J_{t-1} + \Delta_t^{(l)}.
\end{equation}
Next, our goal is to prove \eqref{cor:eq:J_recursion}. Based on the monotonicity of mutual information
since $\Theta_R^{(t)}$ is a component of $(\Theta_0^{(t)},\Theta_R^{(t)})$, we have
\begin{equation}
\label{cor:eq:augment_second}
I(\mathbf U_{1:t}^{(l)};\Theta_R^{(t)}) \le I(\mathbf U_{1:t}^{(l)}; \Theta_0^{(t)},\Theta_R^{(t)}).
\end{equation}
Apply the chain rule to the RHS:
\begin{align}
I(\mathbf U_{1:t}^{(l)}; \Theta_0^{(t)},\Theta_R^{(t)})
&= I(\mathbf U_{1:t}^{(l)};\Theta_0^{(t)}) + I(\mathbf U_{1:t}^{(l)};\Theta_R^{(t)}\mid \Theta_0^{(t)}).
\label{cor:eq:chain1}
\end{align}

We handle the two terms in \eqref{cor:eq:chain1} separately.

Since $\Theta_0^{(t)}=\Theta_R^{(t-1)}$ is fully determined by the past training procedure,
conditioning on $\mathbf U_{1:t-1}^{(l)}$ leaves no additional dependence on $\mathbf U_t^{(l)}$, thus:
\[
I(\mathbf U_t^{(l)};\Theta_0^{(t)}\mid \mathbf U_{1:t-1}^{(l)}) = 0.
\]
Then, by the chain rule,
\begin{align}
I(\mathbf U_{1:t}^{(l)};\Theta_0^{(t)})
&= I(\mathbf U_{1:t-1}^{(l)};\Theta_0^{(t)}) + I(\mathbf U_t^{(l)};\Theta_0^{(t)}\mid \mathbf U_{1:t-1}^{(l)}) \nonumber\\
&= I(\mathbf U_{1:t-1}^{(l)};\Theta_0^{(t)})
= I(\mathbf U_{1:t-1}^{(l)};\Theta_R^{(t-1)})
= J_{t-1}.
\label{cor:eq:init_term}
\end{align}
Decompose $\mathbf U_{1:t}^{(l)}=(\mathbf U_{1:t-1}^{(l)},\mathbf U_t^{(l)})$ and apply the chain rule:
\begin{align}
I(\mathbf U_{1:t}^{(l)};\Theta_R^{(t)}\mid \Theta_0^{(t)})
&= I(\mathbf U_t^{(l)};\Theta_R^{(t)}\mid \Theta_0^{(t)}) + I(\mathbf U_{1:t-1}^{(l)};\Theta_R^{(t)}\mid \Theta_0^{(t)},\mathbf U_t^{(l)}).
\label{cor:eq:chain2}
\end{align}
For each $t$, the final parameter $\Theta_R^{(t)}$ is conditionally independent of the
past training sequences $\mathbf U_{1:t-1}^{(l)}$ given the initialization $\Theta_0^{(t)}$ and the current training sequence $\mathbf U_t^{(l)}$:
\[
\Theta_R^{(t)} \perp \mathbf U_{1:t-1}^{(l)} \mid (\Theta_0^{(t)},\mathbf U_t^{(l)}, W_{1:l}).
\] 
Therefore the second term $I(\mathbf U_{1:t-1}^{(l)};\Theta_R^{(t)}\mid \Theta_0^{(t)},\mathbf U_t^{(l)})$ on the RHS of \eqref{cor:eq:chain2} is zero, hence
\begin{equation}
\label{cor:eq:increment_term}
I(\mathbf U_{1:t}^{(l)};\Theta_R^{(t)}\mid \Theta_0^{(t)}) = I(\mathbf U_t^{(l)};\Theta_R^{(t)}\mid \Theta_0^{(t)}) = \Delta_t^{(l)}.
\end{equation}

Plugging \eqref{cor:eq:init_term} and \eqref{cor:eq:increment_term} into \eqref{cor:eq:chain1}, and then using \eqref{cor:eq:augment_second}, yields
\[
J_t = I(\mathbf U_{1:t}^{(l)};\Theta_R^{(t)})
\le I(\mathbf U_{1:t}^{(l)}; \Theta_0^{(t)},\Theta_R^{(t)})
= J_{t-1} + \Delta_t^{(l)},
\]
which establishes \eqref{cor:eq:J_recursion}.
Iterating \eqref{cor:eq:J_recursion} from $t=1$ to $t=T-1$ gives
\[
J_{T-1} \le J_0 + \sum_{t=1}^{T-1}\Delta_t^{(l)}.
\]
Because parameters are randomly initialized before network training, therefore $\Theta_0^{(1)}$ is independent of $\mathbf U_1^{(l)}$, which implies $J_0=I(\mathbf U_{1:0}^{(l)};\Theta_R^{(0)})=0$, we obtain
\begin{equation}
\label{cor:eq:J_sum}
J_{T-1} \le \sum_{t=1}^{T-1}\Delta_t^{(l)}.
\end{equation}

From \eqref{cor:eq:HT_to_memory}--\eqref{cor:eq:memory_to_pastU} we have
\[
H_T^{(l)} \le I(\mathbf U_{1:T-1}^{(l)};\Theta_R^{(T-1)}) = J_{T-1}.
\]
Combining with \eqref{cor:eq:J_sum} yields the desired result:
\[
H_T^{(l)} \le \sum_{t=1}^{T-1}\Delta_t^{(l)}.
\]
\end{proof}
\subsection{Proof of Theorem~\ref{thm:Hierarchical SGLD}}
\begin{restatetheorem}{\ref{thm:Hierarchical SGLD}}[Restate]
Fix a split layer $l\in\{0,\ldots,L-1\}$ and condition on the base parameters $W_{1:l}$, so that the suffix network remains nonempty. Assume the loss is $\sigma$-subgaussian as in Theorem~\ref{thm:synergy_drift}. For task $t$ trained by SGLD for $R_t$ steps, we have
\begin{align*}
\big|\mathrm{gen}_{W^{(t)}}\big|
\le
(t-1)\sqrt{2\sigma^2 \mathcal{K}_t^{(l)}}
+
\sqrt{\frac{2\sigma^2}{N_{\mathrm{eff},t}}
\Big(\mathcal{C}_t^{(l)}+\sum_{s=1}^t \sum_{r=1}^{R_s}\mathbb E\!\left[\frac12\log\det\!\Big(I+\frac{\eta_{s,r}^2}{\tau_{s,r}^2}M^{(l)}_{s,r}\Big)\right]\Big)},
\end{align*}
where $M^{(l)}_{s,r}=\mathbb E\!\big[G_{s,r}G_{s,r}^\top\,\big|\,\Theta^{(s)}_{r-1},W_{1:l}\big]$, and $\mathcal{K}_t^{(l)}$, $\mathcal{C}_t^{(l)}$, and $N_{\mathrm{eff},t}$ denote the corresponding quantities of Theorem~\ref{thm:synergy_drift} with the task horizon $T$ replaced by $t$.
\end{restatetheorem}
\begin{proof}
Throughout the proof we implicitly fix the split layer $l$ and condition on $W_{1:l}$ to simplify notation.
Apply Theorem~\ref{thm:synergy_drift} to the $t$-task scenario by substituting $t$ for $T$ in the statement of Theorem~\ref{thm:synergy_drift}. Then for any fixed split layer $l$, \begin{equation}
\label{eq:thm32_at_t} 
\big|\mathrm{gen}_{W^{(t)}}\big|
\le
(t-1)\sqrt{2\sigma^2 \mathcal{K}_t^{(l)}}
+ 
\sqrt{\frac{2\sigma^2}{N_{\mathrm{eff},t}} \Big(\mathcal{S}_t^{(l)}+\mathcal{P}_t^{(l)}-\mathcal{R}_t^{(l)}+\mathcal{C}_t^{(l)}\Big)}.
\end{equation}
By the interaction-information identity used in Section~\ref{subsec:Algorithm-Based Hierarchical Bounds}, the SPS combination equals a conditional mutual information between the task-$t$ training sequence and the final upper parameters:
\begin{equation}
\label{eq:sps_is_mi}
\mathcal{S}_t^{(l)}+\mathcal{P}_t^{(l)}-\mathcal{R}_t^{(l)}
=
I\Big(\mathbf{U}_t^{(l)};\Theta^{(t)}_{R_t}\mid W_{1:l}\Big).
\end{equation}
Thus, it suffices to upper bound $I\Big(\mathbf{U}_t^{(l)};\Theta^{(t)}_{R_t} \mid W_{1:l}\Big)$ by a cumulative, algorithmic quantity.
For each task index $s\ge 1$, define the within-task increment
\[
\Delta_s^{(l)}
:=
I\left(\mathbf U_s^{(l)};\Theta_{R_s}^{(s)}\mid\Theta_0^{(s)},W_{1:l}\right),
\]
and the heritage term
\[
H_s^{(l)}
:=
I\left(\mathbf U_s^{(l)};\Theta_0^{(s)}\mid W_{1:l}\right),
\qquad
\text{where }\Theta^{(s)}_0=\Theta^{(s-1)}_{R_{s-1}}\ \text{for}\ s\ge 2.
\]
Applying Proposition~\ref{prop:cross_task_recursion} at task $s$ we have decomposition,
\begin{equation}
\label{eq:prop41_task_s}
I \Big(\mathbf U_s^{(l)};\Theta^{(s)}_{R_s} \mid W_{1:l}\Big)
\le
H_s^{(l)}+\Delta_s^{(l)}.
\end{equation}
Applying Corollary~\ref{cor:Cumulative information budget}, we can obtain a sum-to-$t$ recursion,
\begin{equation}
\label{eq:mi_sum_to_t_Delta}
I\Big(\mathbf{U}_t^{(l)};\Theta^{(t)}_{R_t}\mid W_{1:l}\Big)
\le
\sum_{s=1}^{t-1}\Delta_s^{(l)}+\Delta_t^{(l)}
=
\sum_{s=1}^t \Delta_s^{(l)}.
\end{equation}
Fix a task $s\in\{1,\dots,t\}$. Let $\Theta^{(s)}_{0:R_s}:=(\Theta^{(s)}_0,\Theta^{(s)}_1,\dots,\Theta^{(s)}_{R_s})$ denote the full within-task trajectory.
Since $\Theta^{(s)}_{R_s}$ is a coordinate of $\Theta^{(s)}_{0:R_s}$, by monotonicity of mutual information,
\[
I\!\Big(\mathbf U_s^{(l)};\Theta^{(s)}_{R_s}\,\big|\,\Theta^{(s)}_0,W_{1:l}\Big)
\le
I\!\Big(\mathbf U_s^{(l)};\Theta^{(s)}_{0:R_s}\,\big|\,\Theta^{(s)}_0,W_{1:l}\Big).
\]
Apply the chain rule:
\[
I\!\Big(\mathbf U_s^{(l)};\Theta^{(s)}_{0:R_s}\,\big|\,\Theta^{(s)}_0,W_{1:l}\Big)
=
\sum_{r=1}^{R_s}
I\!\Big(\mathbf U_s^{(l)};\Theta^{(s)}_r\,\big|\,\Theta^{(s)}_{0:r-1},\Theta^{(s)}_0,W_{1:l}\Big),
\]
where there is no $r=0$ term because $\Theta^{(s)}_0$ is already conditioned on.
For each $r\ge 1$, write the conditional mutual information as a differential entropy difference:
\begin{align*}
I\!\Big(\mathbf U_s^{(l)};\Theta^{(s)}_r\,\big|\,\Theta^{(s)}_{0:r-1},\Theta^{(s)}_0,W_{1:l}\Big)
&=
h\!\Big(\Theta^{(s)}_r\,\big|\,\Theta^{(s)}_{0:r-1},\Theta^{(s)}_0,W_{1:l}\Big)
-
h\!\Big(\Theta^{(s)}_r\,\big|\,\mathbf U_s^{(l)},\Theta^{(s)}_{0:r-1},\Theta^{(s)}_0,W_{1:l}\Big).
\end{align*}
First, conditioning reduces entropy, hence
\[
h\!\Big(\Theta^{(s)}_r\,\big|\,\Theta^{(s)}_{0:r-1},\Theta^{(s)}_0,W_{1:l}\Big)
\le
h\!\Big(\Theta^{(s)}_r\,\big|\,\Theta^{(s)}_{r-1},W_{1:l}\Big).
\]
Second, by the Markov structure of the update rule, given $(\mathbf U_s^{(l)},\Theta^{(s)}_{r-1},W_{1:l})$, the next iterate $\Theta^{(s)}_r$ is generated using only the freshly sampled minibatch $B_{s,r}$ (sampled from $\mathbf U_s^{(l)}$) and the fresh Gaussian noise $N_{s,r}$, so it is conditionally independent of $\Theta^{(s)}_{0:r-2}$. Therefore
\[
h\!\Big(\Theta^{(s)}_r\,\big|\,\mathbf U_s^{(l)},\Theta^{(s)}_{0:r-1},\Theta^{(s)}_0,W_{1:l}\Big)
=
h\!\Big(\Theta^{(s)}_r\,\big|\,\mathbf U_s^{(l)},\Theta^{(s)}_{r-1},W_{1:l}\Big).
\]
Combining these two displays gives, for each $r$,
\[
I\!\Big(\mathbf U_s^{(l)};\Theta^{(s)}_r\,\big|\,\Theta^{(s)}_{0:r-1},\Theta^{(s)}_0,W_{1:l}\Big)
\le
I\!\Big(\mathbf U_s^{(l)};\Theta^{(s)}_r\,\big|\,\Theta^{(s)}_{r-1},\Theta^{(s)}_0,W_{1:l}\Big).
\]
Summing over $r$ yields
\[
\Delta_s^{(l)}
=
I\!\Big(\mathbf U_s^{(l)};\Theta^{(s)}_{R_s}\,\big|\,\Theta^{(s)}_0,W_{1:l}\Big)
\le
\sum_{r=1}^{R_s}
I\!\Big(\mathbf U_s^{(l)};\Theta^{(s)}_r\,\big|\,\Theta^{(s)}_{r-1},\Theta^{(s)}_0,W_{1:l}\Big).
\]
Now fix $r$ and apply data processing from the full data $\mathbf U_s^{(l)}$ to the sampled minibatch $B_{s,r}$:
given $(\Theta^{(s)}_{r-1},W_{1:l})$, the algorithm draws $B_{s,r}$ by randomized sampling from $\mathbf U_s^{(l)}$, and the update mapping from $B_{s,r}$ to $\Theta^{(s)}_r$ depends on $\mathbf U_s^{(l)}$ only through $B_{s,r}$.
Thus we have the Markov chain
\[
\mathbf U_s^{(l)} \longrightarrow B_{s,r} \longrightarrow \Theta^{(s)}_r
\quad\text{conditioned on }(\Theta^{(s)}_{r-1},\Theta^{(s)}_0,W_{1:l}),
\]
and hence
\[
I\!\Big(\mathbf U_s^{(l)};\Theta^{(s)}_r\,\big|\,\Theta^{(s)}_{r-1},\Theta^{(s)}_0,W_{1:l}\Big)
\le
I\!\Big(B_{s,r};\Theta^{(s)}_r\,\big|\,\Theta^{(s)}_{r-1},\Theta^{(s)}_0,W_{1:l}\Big).
\]
Therefore,
\begin{equation}
\label{eq:Delta_le_sum_step_MI}
\Delta_s^{(l)} \le \sum_{r=1}^{R_s}
I\!\Big(B_{s,r};\Theta^{(s)}_r\,\big|\,\Theta^{(s)}_{r-1},\Theta^{(s)}_0,W_{1:l}\Big).
\end{equation}
It remains to upper bound each per-step term.
Using the SGLD update $\Theta^{(s)}_r=\Theta^{(s)}_{r-1}+\eta_{s,r}G_{s,r}+N_{s,r}$ with $N_{s,r}\sim\mathcal N(0,\tau_{s,r}^2 I)$, translation invariance of mutual information gives
\[
I\!\Big(B_{s,r};\Theta^{(s)}_r\,\big|\,\Theta^{(s)}_{r-1},\Theta^{(s)}_0,W_{1:l}\Big)
=
I\!\Big(B_{s,r};\eta_{s,r}G_{s,r}+\tau_{s,r}\,\big|\,\Theta^{(s)}_{r-1},\Theta^{(s)}_0,W_{1:l}\Big).
\]
Write this as a differential entropy difference:
\begin{align*}
I\!\Big(B_{s,r};\eta_{s,r}G_{s,r}+\tau_{s,r}\,\big|\,\Theta^{(s)}_{r-1},\Theta^{(s)}_0,W_{1:l}\Big)
&=
h\!\Big(\eta_{s,r}G_{s,r}+\tau_{s,r}\,\big|\,\Theta^{(s)}_{r-1},\Theta^{(s)}_0,W_{1:l}\Big)
-
h(N_{s,r}),
\end{align*}
because $G_{s,r}$ is a deterministic function of $(B_{s,r},\Theta^{(s)}_{r-1},W_{1:l})$ and $N_{s,r}$ is independent of $B_{s,r}$ and Gaussian.
Let $Y_{s,r}:=\eta_{s,r}G_{s,r}+N_{s,r}$ and denote $d:=d_{l+1:L}$.
By the Lemma~\ref{lem:Gassian_second_moment_entropy} bound applied conditionally,
\[
h\!\Big(Y_{s,r}\,\big|\,\Theta^{(s)}_{r-1},\Theta^{(s)}_0,W_{1:l}\Big)
\le
\frac12\log\Big((2\pi e)^d\det\Big(\mathbb E\!\big[Y_{s,r}Y_{s,r}^\top\,\big|\,\Theta^{(s)}_{r-1},\Theta^{(s)}_0,W_{1:l}\big]\Big)\Big).
\]
Compute the conditional second moment using independence of $N_{s,r}$ and $G_{s,r}$ given $(\Theta^{(s)}_{r-1},W_{1:l})$ and $\mathbb E[N_{s,r}]=0$:
\begin{align*}
\mathbb E\!\big[Y_{s,r}Y_{s,r}^\top\,\big|\,\Theta^{(s)}_{r-1},\Theta^{(s)}_0,W_{1:l}\big]
&=
\eta_{s,r}^2 \mathbb E\!\big[G_{s,r}G_{s,r}^\top\,\big|\,\Theta^{(s)}_{r-1},W_{1:l}\big]
+\mathbb E[N_{s,r}N_{s,r}^\top]\\
&=
    \eta_{s,r}^2 M^{(l)}_{s,r} + \tau_{s,r}^2 I.
\end{align*}
Also $h(N_{s,r})=\frac12\log\big((2\pi e)^d\det(\tau_{s,r}^2 I)\big)$ since $N_{s,r}\sim\mathcal N(0,\tau_{s,r}^2 I)$.
Subtracting these two entropies yields
\[
I\!\Big(B_{s,r};\Theta^{(s)}_r\,\big|\,\Theta^{(s)}_{r-1},\Theta^{(s)}_0,W_{1:l}\Big)
\le
\frac12\log\det\!\Big(I+\frac{\eta_{s,r}^2}{\tau_{s,r}^2}M^{(l)}_{s,r}\Big)
\]
Taking expectation and summing over $r$ in \eqref{eq:Delta_le_sum_step_MI}, we obtain
\[
\Delta_s^{(l)} 
\le 
\sum_{r=1}^{R_s}\mathbb E\!\left[\frac12\log\det\!\Big(I+\frac{\eta_{s,r}^2}{\tau_{s,r}^2}M^{(l)}_{s,r}\Big)\right].
\]
This holds for each $s$, and hence
\begin{equation}
\label{eq:mi_sum_to_t_Itraj}
I\!\Big(\mathbf U_t^{(l)};\Theta^{(t)}_{R_t}\,\big|\,W_{1:l}\Big)
\le
\sum_{s=1}^t \Delta_s^{(l)}
\le
\sum_{s=1}^t \sum_{r=1}^{R_s}\mathbb E\!\left[\frac12\log\det\!\Big(I+\frac{\eta_{s,r}^2}{\tau_{s,r}^2}M^{(l)}_{s,r}\Big)\right].
\end{equation}
Substitute \eqref{eq:sps_is_mi} and \eqref{eq:mi_sum_to_t_Itraj} into \eqref{eq:thm32_at_t} to obtain
\[
\big|\mathrm{gen}_{W^{(t)}}\big|
\le
(t-1)\sqrt{2\sigma^2 \mathcal{K}_t^{(l)}}
+
\sqrt{\frac{2\sigma^2}{N_{\mathrm{eff},t}}
\Big(\mathcal{C}_t^{(l)}+\sum_{s=1}^t \sum_{r=1}^{R_s}\mathbb E\!\left[\frac12\log\det\!\Big(I+\frac{\eta_{s,r}^2}{\tau_{s,r}^2}M^{(l)}_{s,r}\Big)\right]\Big)}.
\]
This completes the proof.
\end{proof}
\subsection{Proof of Proposition~\ref{prop:logdet_split}}
\begin{restateproposition}{\ref{prop:logdet_split}}[Restate]
Let $ A_{s,r}^{(l)}:=I+\alpha_{s,r}V_{s,r}^{(l)}\succ 0$ with $\alpha_{s,r}=\eta_{s,r}^2/\tau_{s,r}^2$. Then
\begin{align*}
\frac12\log\det\!\big(I+\alpha_{s,r}M_{s,r}^{(l)}\big)
=
\underbrace{\frac12\log\det\!\big(A_{s,r}^{(l)}\big)}_{\textbf{instability}}
+
\underbrace{\frac12\log\!\Big(1+\alpha_{s,r}\,\mu_{s,r}^{(l)\top}\big(A_{s,r}^{(l)}\big)^{-1}\mu_{s,r}^{(l)}\Big)}_{\textbf{interaction cost}}.
\end{align*}
\end{restateproposition}
\begin{proof}
We condition on $(\Theta^{(s)}_{r-1},W_{1:l})$ throughout.
By definition,
\[
M^{(l)}_{s,r}=\mathbb E[G_{s,r}G_{s,r}^\top]
\quad\text{with}\quad
G_{s,r}=\lambda_{s,\mathrm{old}}G^{\mathrm{old}}_{s,r}+\lambda_{s,\mathrm{new}}G^{\mathrm{new}}_{s,r}.
\]
Since the old and new minibatches are sampled independently given the state,
$G^{\mathrm{old}}_{s,r}$ and $G^{\mathrm{new}}_{s,r}$ are conditionally independent, hence
\[
\mathrm{Cov}(G_{s,r})
=
\lambda_{s,\mathrm{old}}^2\mathrm{Cov}(G^{\mathrm{old}}_{s,r})
+\lambda_{s,\mathrm{new}}^2\mathrm{Cov}(G^{\mathrm{new}}_{s,r})
=
V^{(l)}_{s,r}.
\]
Also $\mathbb E[G_{s,r}]=\lambda_{s,\mathrm{old}}\mu^{\mathrm{old},(l)}_{s,r}+\lambda_{s,\mathrm{new}}\mu^{\mathrm{new},(l)}_{s,r}=\mu^{(l)}_{s,r}$.
Using the identity $\mathbb E[XX^\top]=\mathrm{Cov}(X)+\mathbb E[X]\mathbb E[X]^\top$, we can obtain
\[
I+\alpha_{s,r}M^{(l)}_{s,r}
=
I+\alpha_{s,r}V^{(l)}_{s,r}+\alpha_{s,r}\mu^{(l)}_{s,r}\mu^{(l)\top}_{s,r}
=
A^{(l)}_{s,r} + uu^\top,
\quad u:=\sqrt{\alpha_{s,r}}\,\mu^{(l)}_{s,r}.
\]
Since $A^{(l)}_{s,r}\succ 0$, the matrix determinant lemma $ \det(I + A + uu^\top) = \det(I + A) \left(1 + u^\top (I + A)^{-1} u \right)$ gives
\[
\det\!\big(A^{(l)}_{s,r}+uu^\top\big)
=
\det\!\big(A^{(l)}_{s,r}\big)\Big(1+u^\top\big(A^{(l)}_{s,r}\big)^{-1}u\Big)
=
\det\!\big(A^{(l)}_{s,r}\big)\Big(1+\alpha_{s,r}\mu^{(l)\top}_{s,r}\big(A^{(l)}_{s,r}\big)^{-1}\mu^{(l)}_{s,r}\Big).
\]
Taking $\frac12\log$ of both sides completes the proof.
\end{proof}
\subsection{Proof of Corollary~\ref{cor:mean_energy_alignment}}
\begin{restatecorollary}{\ref{cor:mean_energy_alignment}}[Restate]
Let $\mu^{\mathrm{old},(l)}_{s,r}:=\mathbb E[G^{\mathrm{old}}_{s,r}\mid \Theta^{(s)}_{r-1},W_{1:l}],
\mu^{\mathrm{new},(l)}_{s,r}:=\mathbb E[G^{\mathrm{new}}_{s,r}\mid \Theta^{(s)}_{r-1},W_{1:l}]$, so that $\mu^{(l)}_{s,r}=\lambda_{s,\mathrm{old}}\mu^{\mathrm{old},(l)}_{s,r}+\lambda_{s,\mathrm{new}}\mu^{\mathrm{new},(l)}_{s,r}.$
Then under the sensitivity metric $H^{(l)}_{s,r}$ in Definition~\ref{def:metric_alignment},
\begin{align*}
\mu^{(l)\top}_{s,r}H^{(l)}_{s,r}\mu^{(l)}_{s,r}
=
\lambda_{s,\mathrm{old}}^2\big\|\mu^{\mathrm{old},(l)}_{s,r}\big\|_{H}^2
+\lambda_{s,\mathrm{new}}^2\big\|\mu^{\mathrm{new},(l)}_{s,r}\big\|_{H}^2+
2\lambda_{s,\mathrm{old}}\lambda_{s,\mathrm{new}}
\big\|\mu^{\mathrm{old},(l)}_{s,r}\big\|_{H}
\big\|\mu^{\mathrm{new},(l)}_{s,r}\big\|_{H}
\cos_{H}\!\Big(\mu^{\mathrm{old},(l)}_{s,r},\mu^{\mathrm{new},(l)}_{s,r}\Big),
\end{align*}
where the information-geometric alignment
\[
\mathrm{Align}=\cos_{H}\!\Big(\mu^{\mathrm{old},(l)}_{s,r},\mu^{\mathrm{new},(l)}_{s,r}\Big).
\]

\end{restatecorollary}
\begin{proof}
Since $\mu_{s,r}^{(l)}=\lambda_{s,\mathrm{old}}\mu_{s,r}^{\mathrm{old},(l)}+\lambda_{s,\mathrm{new}}\mu_{s,r}^{\mathrm{new},(l)}$, then
\begin{align*}
\mu_{s,r}^{(l)\top}H_{s,r}^{(l)}\mu_{s,r}^{(l)}
&=
\big\|\lambda_{s,\mathrm{old}}\mu_{s,r}^{\mathrm{old},(l)}+\lambda_{s,\mathrm{new}}\mu_{s,r}^{\mathrm{new},(l)}\big\|^2_H \\
&=
\lambda_{s,\mathrm{old}}^2\|\mu_{s,r}^{\mathrm{old},(l)}\|^2_H+\lambda_{s,\mathrm{new}}^2\|\mu_{s,r}^{\mathrm{new},(l)}\|^2_H
+2\lambda_{s,\mathrm{old}}\lambda_{s,\mathrm{new}}\langle \mu_{s,r}^{\mathrm{old},(l)},\mu_{s,r}^{\mathrm{new},(l)}\rangle_H \\
&=
\lambda_{s,\mathrm{old}}^2\|\mu_{s,r}^{\mathrm{old},(l)}\|^2_H+\lambda_{s,\mathrm{new}}^2\|\mu_{s,r}^{\mathrm{new},(l)}\|^2_H
+2\lambda_{s,\mathrm{old}}\lambda_{s,\mathrm{new}}\|\mu_{s,r}^{\mathrm{old},(l)}\|_H\,\|\mu_{s,r}^{\mathrm{new},(l)}\|_H\cos_H\!\big(\mu_{s,r}^{\mathrm{old},(l)},\mu_{s,r}^{\mathrm{new},(l)}\big).
\end{align*}
This completes the proof.
\end{proof}
\subsection{Proof of Corollary~\ref{Scalar diagnostic decomposition}}
\begin{restatecorollary}{\ref{Scalar diagnostic decomposition}}[Restate]
Conditioning on $(\Theta^{(s)}_{r-1},W_{1:l})$, the per-step trajectory budget admits the scalar upper bound
\begin{align*}
\sum_{s=1}^t\sum_{r=1}^{R_s}\mathbb  E\!\left[\frac12\log\det\!\Big(I+\alpha_{s,r}M^{(l)}_{s,r}\Big)\right]
\le
\frac12\sum_{s=1}^t\sum_{r=1}^{R_s}\alpha_{s,r}\,\mathbb E\!\Big[\mathrm{tr}\!\big(V^{(l)}_{s,r}\big)+\|\mu^{(l)}_{s,r}\|^2\Big].
\end{align*}
Moreover, the two scalar components admit the explicit decompositions $\mathrm{tr}\!\big(V^{(l)}_{s,r}\big)=
\lambda_{s,\mathrm{old}}^2\mathrm{tr}\!\big(\Sigma^{\mathrm{old},(l)}_{s,r}\big)
+\lambda_{s,\mathrm{new}}^2\mathrm{tr}\!\big(\Sigma^{\mathrm{new},(l)}_{s,r}\big),$ and $\|\mu^{(l)}_{s,r}\|^2 = \lambda_{s,\mathrm{old}}^2\|\mu^{\mathrm{old},(l)}_{s,r}\|^2 +\lambda_{s,\mathrm{new}}^2\|\mu^{\mathrm{new},(l)}_{s,r}\|^2 +2\lambda_{s,\mathrm{old}}\lambda_{s,\mathrm{new}}
\langle\mu^{\mathrm{old},(l)}_{s,r},\mu^{\mathrm{new},(l)}_{s,r}\rangle$.
\end{restatecorollary}
\begin{proof}
Let $B\succeq 0$ have eigenvalues $\{\lambda_i\}_{i=1}^d\subset\mathbb R_+$.
Then $\log\det(I+B)=\sum_{i=1}^d\log(1+\lambda_i)$ and $\mathrm{tr}(B)=\sum_{i=1}^d\lambda_i$.

Since $\log(1+x)\le x$ for all $x\ge 0$, summing over $i$ yields $\log\det(I+B)\le \mathrm{tr}(B)$.

Apply to $B=\alpha_{s,r}M^{(l)}_{s,r}\succeq 0$:
\[
\frac12\log\det(I+\alpha_{s,r}M^{(l)}_{s,r})
\le
\frac12\,\alpha_{s,r}\,\mathrm{tr}(M^{(l)}_{s,r}).
\]
Taking expectation and summing over $(s,r)$ gives $\frac12\sum_{s=1}^t\sum_{r=1}^{R_s}\alpha_{s,r}\,\mathbb E\!\big[\mathrm{tr}(M^{(l)}_{s,r})\big]$.

By the identity $\mathbb E[XX^\top]=\mathrm{Cov}(X)+\mathbb E[X]\mathbb E[X]^\top$,
we have $M^{(l)}_{s,r}=V^{(l)}_{s,r}+\mu^{(l)}_{s,r}\mu^{(l)\top}_{s,r}$, hence
$\mathrm{tr}(M^{(l)}_{s,r})=\mathrm{tr}(V^{(l)}_{s,r})+\mathrm{tr}(\mu^{(l)}_{s,r}\mu^{(l)\top}_{s,r})
=\mathrm{tr}(V^{(l)}_{s,r})+\|\mu^{(l)}_{s,r}\|^2$,
which yields $\frac12\sum_{s=1}^t\sum_{r=1}^{R_s}\alpha_{s,r}\,\mathbb E\!\Big[\mathrm{tr}\!\big(V^{(l)}_{s,r}\big)+\|\mu^{(l)}_{s,r}\|^2\Big]$.

Finally, applying linearity of the trace to the definition of $V^{(l)}_{s,r}$, we obtain
 \[
 \mathrm{tr}\!\big(V^{(l)}_{s,r}\big)
 =
\lambda_{s,\mathrm{old}}^2\mathrm{tr}\!\big(\Sigma^{\mathrm{old},(l)}_{s,r}\big)
+
\lambda_{s,\mathrm{new}}^2\mathrm{tr}\!\big(\Sigma^{\mathrm{new},(l)}_{s,r}\big).
\]
The second identity follows by expanding the squared Euclidean norm
\[
\|\mu^{(l)}_{s,r}\|^2 = \lambda_{s,\mathrm{old}}^2\|\mu^{\mathrm{old},(l)}_{s,r}\|^2 +\lambda_{s,\mathrm{new}}^2\|\mu^{\mathrm{new},(l)}_{s,r}\|^2 +2\lambda_{s,\mathrm{old}}\lambda_{s,\mathrm{new}}
\langle\mu^{\mathrm{old},(l)}_{s,r},\mu^{\mathrm{new},(l)}_{s,r}\rangle.
\]
\end{proof}
\subsection{Proof of Corollary~\ref{cor:Optimal mixing}}
\begin{restatecorollary}{\ref{cor:Optimal mixing}}[Restate]
Fix any sensitivity metric $H^{(l)}_{s,r}$. For $\lambda\in[0,1]$, consider the mixed vector $\mu_{s,r}^{(l)}(\lambda):=\lambda \mu^{\mathrm{old},(l)}_{s,r}+(1-\lambda)\mu^{\mathrm{new},(l)}_{s,r}$ and its metric energy
$\|\mu_{s,r}^{(l)}(\lambda)\|_H^2$.
Then $\|\mu^{(l)}_{s,r}(\lambda)\|_H^2$ is minimized over $\lambda\in[0,1]$ by
\begin{equation*}
\lambda^\star
=
\mathrm{clip}_{[0,1]}\!\left(
\frac{{\mu^{\mathrm{new},(l)}_{s,r}}^\top H(\mu^{\mathrm{new},(l)}_{s,r}-\mu^{\mathrm{old},(l)}_{s,r})}{(\mu^{\mathrm{old},(l)}_{s,r}- \mu^{\mathrm{new},(l)}_{s,r})^\top H(\mu^{\mathrm{old},(l)}_{s,r}- \mu^{\mathrm{new},(l)}_{s,r})}
\right),
\end{equation*}
with the convention $\lambda^\star=0$ if $\mu^{\mathrm{old},(l)}_{s,r}=\mu^{\mathrm{new},(l)}_{s,r}$ .
\end{restatecorollary}
\begin{proof}
Let $d:=\mu^{\mathrm{old},(l)}_{s,r}-\mu^{\mathrm{new},(l)}_{s,r}$. Then $\mu(\lambda)=\mu^{\mathrm{new},(l)}_{s,r}+\lambda d$ and
\[
\|\mu(\lambda)\|_H^2=(\mu^{\mathrm{new},(l)}_{s,r}+\lambda d)^\top H(\mu^{\mathrm{new},(l)}_{s,r}+\lambda d)={\mu^{\mathrm{new},(l)}_{s,r}}^\top H\mu^{\mathrm{new},(l)}_{s,r}+2\lambda\,{\mu^{\mathrm{new},(l)}_{s,r}}^\top Hd+\lambda^2\,d^\top Hd.
\]
If $\mu^{\mathrm{old},(l)}_{s,r}=\mu^{\mathrm{new},(l)}_{s,r}$, then $d=0$ and $\|\mu(\lambda)\|_H^2\equiv {\mu^{\mathrm{new},(l)}_{s,r}}^\top H\mu^{\mathrm{new},(l)}_{s,r}$ is constant; we define $\lambda^\star=0$.
If $\mu^{\mathrm{old},(l)}_{s,r}\neq \mu^{\mathrm{new},(l)}_{s,r}$, then $d\neq 0$ and since $H\succ 0$, we have $d^\top Hd>0$.
Therefore the scalar function $\varphi(\lambda):=\|\mu(\lambda)\|_H^2$ is a strictly convex quadratic.
Differentiate:
\[
\varphi'(\lambda)=2\,{\mu^{\mathrm{new},(l)}_{s,r}}^\top Hd+2\lambda\,d^\top Hd.
\]
Setting $\varphi'(\lambda)=0$ yields the unique unconstrained minimizer
\[
\lambda^\star_{\mathrm{unc}}
=
-\frac{{\mu^{\mathrm{new},(l)}_{s,r}}^\top Hd}{d^\top Hd}
=
\frac{{\mu^{\mathrm{new},(l)}_{s,r}}^\top H({\mu^{\mathrm{new},(l)}_{s,r}}-\mu^{\mathrm{old},(l)}_{s,r})}{(\mu^{\mathrm{old},(l)}_{s,r}-\mu^{\mathrm{new},(l)}_{s,r})^\top H(\mu^{\mathrm{old},(l)}_{s,r}-\mu^{\mathrm{new},(l)}_{s,r})}.
\]
The constrained minimizer over $\lambda\in[0,1]$ is the Euclidean projection of $\lambda^\star_{\mathrm{unc}}$ onto $[0,1]$,
i.e.,\ $\lambda^\star=\mathrm{clip}_{[0,1]}(\lambda^\star_{\mathrm{unc}})$. The proof is complete.
\end{proof}


\section{Details of the Synthetic Gaussian Experiments}
\label{sec:app:Experimental Details}

We construct a sequence of $T$ binary classification tasks $\{\mathcal{T}_t\}_{t=0}^{T-1}$ in $\mathbb{R}^d$.
Each task is a balanced two-class Gaussian mixture:
\begin{align*}
x \mid (y=0,\mathcal{T}_t) &\sim \mathcal{N}(\mu_t - s w_t, I_d), \\
x \mid (y=1,\mathcal{T}_t) &\sim \mathcal{N}(\mu_t + s w_t, I_d),
\end{align*}
where $w_t$ is a unit direction controlling the class-discriminative axis, $\mu_t$ is a task-specific mean shift, and $s>0$ is the class separation.
To control task relatedness, we restrict $\{w_t\}$ to a random 2D subspace and vary its evolution across tasks.
We consider two modes:
(i) positive synergy, where $w_t$ rotates smoothly by a fixed angle increment $\Delta$ so adjacent tasks remain aligned;
(ii) negative synergy, where $w_t$ alternates sign ($w_{t+1}=-w_t$), which is equivalent to a label-flip across tasks and induces strong interference.
We control cross-task drift via $\mu_t = C\cdot\xi_t$ with $\xi_t\sim\mathcal{N}(0,I_d)$.
For each task $\mathcal{T}_t$, we sample an i.i.d.\ training set $D_t$ of size $n$ and an i.i.d.\ test set $D_t^{\mathrm{test}}$ of size $n_{\mathrm{test}}$.

After completing task $t$, we store exactly $m$ examples from its training set into a per-task buffer $\mathcal{B}_t$. The total replay memory thus scales as $O(Tm)$.
When training on task $t\ge 1$, each update step samples a mini-batch $B^{\mathrm{new}}$ from the current task and a replay mini-batch $B^{\mathrm{old}}$ from replay buffer $\cup_{i<t}\mathcal{B}_i$.
We minimize the mixture loss
\begin{equation*}
\ell_t(\theta)
= \lambda_{\mathrm{new}}(t)\,\widehat{\mathcal{L}}(B^{\mathrm{new}};\theta)
+ \lambda_{\mathrm{old}}(t)\,\widehat{\mathcal{L}}(B^{\mathrm{old}};\theta),
\qquad
\lambda_{\mathrm{new}}(t)=\frac{1}{t+1},\;\;
\lambda_{\mathrm{old}}(t)=\frac{t}{t+1}.
\end{equation*}
We use SGD with fixed learning rate and batch size. 
After finishing all $T$ tasks, we report the test risk
\begin{equation*}
L(\theta)=\frac{1}{T}\sum_{t=0}^{T-1}\widehat{\mathcal{L}}(D_t^{\mathrm{test}};\theta),
\end{equation*}
and the empirical risk (old tasks via replay buffers, the last task via its training set)
\begin{equation*}
\widehat{L}(\theta)
=\frac{T-1}{T}\widehat{\mathcal{L}}\Big(\bigcup_{t=0}^{T-2}\mathcal{B}_t;\theta\Big)
+\frac{1}{T}\widehat{\mathcal{L}}(D_{T-1};\theta).
\end{equation*}
The generalization gap is $\mathrm{gap}=L(\theta)-\widehat{L}(\theta)$. We report this task-averaged gap; it equals the total-risk gap of Theorem~\ref{thm:synergy_drift} divided by the task count $T$. Because $T$ is fixed in this experiment, the constant $1/T$ rescales the gap but leaves the predicted $m^{-1/2}$ scaling and the location of the finite-memory plateau unchanged.
Unless stated otherwise, all curves are averaged over $K$ independent runs; we report $95\%$ confidence intervals for the mean (Student-$t$).

\end{document}